\PassOptionsToPackage{table}{xcolor} %
\documentclass{article} %
\usepackage{iclr2027_conference,times}

\usepackage[utf8]{inputenc} %
\usepackage[T1]{fontenc}    %
\usepackage{hyperref}       %
\usepackage{url}            %
\usepackage{booktabs}       %
\usepackage{amsfonts}       %
\usepackage{nicefrac}       %
\usepackage{microtype}      %
\usepackage{xcolor}         %

\title{Bridging Stochastic Flow Maps and Boltzmann Generators with Normalizing Flows}

\usepackage{enumitem}
\usepackage{amsmath}
\usepackage{amssymb}
\usepackage{graphicx}
\usepackage{comment}
\usepackage[algo2e,ruled,vlined,linesnumbered]{algorithm2e}
\usepackage{multirow}
\usepackage{array}
\definecolor{tableblue}{HTML}{DDEBF7}
\usepackage{titletoc}
\usepackage{wrapfig}
\usepackage{tikz}
\usetikzlibrary{positioning, arrows.meta, fit, backgrounds, calc}
\definecolor{parA}{HTML}{1F77B4} %
\definecolor{parB}{HTML}{E8820C} %
\definecolor{parC}{HTML}{2CA02C} %

\usepackage{amsmath,amsfonts,bm,amssymb,nicefrac,algorithm}
\usepackage{amsthm}
\usepackage{cancel}
\usepackage[normalem]{ulem}

\usepackage{booktabs}          %
\usepackage{graphicx}          %
\usepackage{subcaption}        %
\usepackage{multirow}        %
\usepackage{stmaryrd}        %
\usepackage{yhmath}        %
\usepackage{hyperref}
\usepackage{aliascnt}
\usepackage[capitalize,noabbrev]{cleveref}

\crefname{algorithm}{Alg.}{Algs.}
\Crefname{algorithm}{Alg.}{Algs.}

\crefname{equation}{eq.}{eqs.}
\Crefname{equation}{Eq.}{Eqs.}

\newcommand{\R}{\mathbb{R}}

\newtheoremstyle{mythmstyle} 
{\topsep}    %
{\topsep}    %
{\itshape}   %
{0pt}        %
{\bfseries}  %
{}           %
{ }          %
{}           %
\theoremstyle{mythmstyle}
\newtheorem{theorem}{Theorem}

\newaliascnt{lemma}{theorem}
\newtheorem{lemma}[lemma]{Lemma}
\aliascntresetthe{lemma}
\crefname{lemma}{lemma}{lemmas}
\Crefname{lemma}{Lemma}{Lemmas}

\newaliascnt{corollary}{theorem}
\newtheorem{corollary}[corollary]{Corollary}
\aliascntresetthe{corollary}
\crefname{corollary}{corollary}{corollaries}
\Crefname{corollary}{Corollary}{Corollaries}

\newaliascnt{proposition}{theorem}
\newtheorem{proposition}[proposition]{Proposition}
\aliascntresetthe{proposition}
\crefname{proposition}{proposition}{propositions}
\Crefname{proposition}{Proposition}{Propositions}

\newaliascnt{remark}{theorem}
\newtheorem{remark}[remark]{Remark}
\aliascntresetthe{remark}
\crefname{remark}{remark}{remarks}
\Crefname{remark}{Remark}{Remarks}

\crefname{assumption}{assumption}{assumptions}
\Crefname{assumption}{Assumption}{Assumptions}

\newtheorem*{proposition*}{Proposition}
\newtheorem*{theorem*}{Theorem}
\newtheorem*{lemma*}{Lemma}

\newcommand{\TV}{\mathop{\mathrm{TV}}\nolimits}

\usepackage{physics}

\newcommand{\rmd}{\mathrm{d}}
\newcommand{\Idd}{\mathrm{I}_d}

\DeclareMathOperator{\diam}{diam}
\newcommand{\err}[1]{\color{gray}{\scriptscriptstyle \pm #1}}
\newcommand{\best}[1]{\textbf{#1}}
\newcommand{\second}[1]{\underline{#1}}
\newcolumntype{H}{>{\columncolor{gray!12}}c}

\newcommand{\ourmethodfull}{Normalizing Flow Flow Maps{}}
\newcommand{\ourmethod}{NF$^2$M{}}
\newcommand{\energyw}{\ensuremath{\mathcal{E}\text{-}\mathcal{W}_2}}
\newcommand{\torusw}{\ensuremath{\mathbb{T}\text{-}\mathcal{W}_2}}

\crefname{figure}{Fig.}{Figs.}
\Crefname{figure}{Fig.}{Figs.}
\crefname{table}{Tab.}{Tabs.}
\Crefname{table}{Tab.}{Tabs.}
\crefname{section}{Sec.}{Secs.}
\Crefname{section}{Sec.}{Secs.}
\crefname{equation}{eq.}{eqs.}
\Crefname{equation}{Eq.}{Eqs.}
\crefname{proposition}{Prop.}{Props.}
\Crefname{proposition}{Prop.}{Props.}
\crefname{theorem}{Thm.}{Thms.}
\Crefname{theorem}{Thm.}{Thms.}
\crefname{corollary}{Cor.}{Cors.}
\Crefname{corollary}{Cor.}{Cors.}
\crefname{appendix}{App.}{Apps.}
\Crefname{appendix}{App.}{Apps.}
\AddToHook{cmd/appendix/before}{%
    \crefalias{section}{appendix}%
    \crefalias{subsection}{appendix}%
}

\definecolor{nfmred}{HTML}{DA3D33}
\definecolor{nfmblue}{HTML}{2E78B8}
\definecolor{nfmpurple}{HTML}{7E3F9F}

\author{%
  \textbf{Louis Grenioux}$^{1}$\thanks{Correspondence to \texttt{lgrenioux@flatironinstitute.org}},~ \textbf{RuiKang OuYang}$^{2}$~\&~\textbf{Luhuan Wu}$^{1,3}$\\
  $^1$ Center for Computational Mathematics, Flatiron Institute, New York, NY, USA \\
  $^2$ Department of Engineering, University of Cambridge, Cambridge, UK  \\
  $^3$ Department of Applied Mathematics and Statistics, Johns Hopkins University, Baltimore, MD, USA
}

\iclrfinalcopy
\begin{document}

\maketitle

\begin{abstract}

Generating independent, equilibrium samples of molecular systems at scale remains a central obstacle in computational statistical mechanics. Boltzmann Generators address this by pairing a generative model with importance sampling to obtain consistent samples from the target distribution. We introduce \ourmethodfull{} (\ourmethod{}), which combines the strengths of recent stochastic flow maps with the tractability of classic normalizing flows to build a Boltzmann Generator. Unlike most methods, which correct the generative model only at the end, \ourmethod{} reweighs each denoising transition as generation proceeds, avoiding wasted compute on trajectories that are ultimately discarded. At each denoising step, a conditional normalizing flow proposes clean configurations given the current noisy state (a simpler task than sampling directly from the target) and its exact likelihood enables correcting each proposal toward the true denoising transition of the target Boltzmann distribution. This is in contrast to most existing methods, whose likelihoods are approximate or expensive to evaluate, undermining the statistical reliability of the correction. We establish consistency of the corrected transitions and bound how approximation errors propagate through the sampling chain. We evaluate \ourmethod{} on peptide systems, demonstrating improved sampling efficiency and sample quality.

\end{abstract}

\section{Introduction}

Sampling the Boltzmann distribution of molecular systems is central to statistical physics, underlying protein folding, ligand binding, conformational change, and molecular design \citep{liu2001monte,frenkel2001understanding,noe2009constructing,lindorff2011fast,buch2011high}. This remains hard because molecular energy landscapes are high-dimensional and rugged, with metastable states separated by large barriers. Classical approaches like MD and MCMC are asymptotically exact \citep{leimkuhler2015molecular,frenkel2001understanding,wirnsberger2020targeted}, but their local-update nature makes them mix slowly on such landscapes, producing correlated samples and requiring prohibitively long runs to capture transitions between metastable states.

Boltzmann Generators (BGs) take a different strategy \citep{noe2019boltzmann,muller2019neural}: a learned generative model produces plausible configurations independently and in parallel, which are then statistically corrected toward the true equilibrium distribution, e.g. via importance sampling (IS). This hinges on two requirements: proposal-target overlap and a tractable correction procedure—that existing generative backbones trade off differently. Normalizing flows (NFs) \citep{rezende2015variational,dinh2017density,noe2019boltzmann,papamakarios2021normalizing} dominate this role because their tractable likelihoods enable %
IS correction, but their invertibility constraints limit expressiveness and overlap. Bridging the NF to target gap via temperature-annealed interpolation \citep{tan2025scalable} instead suffers from mass teleportation on multimodal targets \citep{woodard2009sufficient,mate2023learning} and does not reliably outperform direct importance sampling \citep{tan2025scalable}. Continuous normalizing flows built from ODE trajectories \citep{chen2018neural,lipman2023flow} offer more expressive architectures, but their likelihoods require integrating the vector-field divergence, making correction costly and discretization-sensitive. Distilling these into few-step flow maps \citep{boffi2025buildconsistencymodellearning} sacrifices tractable likelihoods altogether. Subsequent work restores likelihoods for such few-step maps in two ways : FALCON \citep{rehman2025falconfewstepaccuratelikelihoods} computes them from the full Jacobian of the map, at the cost of memory overhead and sensitivity to the map's numerical invertibility, while F2D2 \citep{ai2026jointdistillationfastlikelihood} and SCALLOP \citep{ouyang2026fewstep} instead distill a likelihood-prediction head, gaining speed but risking error propagation from the learned likelihood into the correction step.

Diffusion models offer a different stochastic transport path, gradually perturbing clean samples via a tractable noising process and generating through learned, typically Gaussian, denoising transitions \citep{ho2020denoising,Song2021score}. These paths handle multimodal targets better than temperature annealing and have shown promise for BGs \citep{grenioux2026diffusionbasedannealedboltzmanngenerators}: although the endpoint density is intractable, the path density can be evaluated, enabling correction in an extended space \citep{Zhang2024efficient,zhang2025efficient}. The Gaussian approximation is accurate only for fine discretizations, however, creating a tradeoff between statistical performance and inference cost \citep{grenioux2026diffusionbasedannealedboltzmanngenerators}. Distributional diffusion and stochastic flow maps address this by modeling the full conditional distribution of clean samples given a noisy state, enabling expressive transitions over larger stepsizes \citep{debortoli2025distributionaldiffusionmodelsscoring,potaptchik2026metaflowmapsenable,holderrieth2026diamondmapsefficientreward} but the resulting models are typically implicit, lacking the likelihoods BG correction needs.

We close this gap with \emph{Normalizing Flow Flow Maps} (\ourmethod{}), a BG that generates samples along a diffusion transport path in a few steps using denoising kernels estimated via importance sampling. Like stochastic flow maps, \ourmethod{} models the conditional distribution $q^\theta(x_0 | x_t)$ of clean data given a noisy state, but it uses a conditional NF rather than an implicit model, yielding a tractable density.

This tractability lets us build an importance-sampling estimator of the true denoising kernel, exact as the number of particles grows: at each step, we propose candidates from $q^\theta(x_0 | x_t)$ and reweight them against the tractable conditional target $p(x_0 | x_t)$ (which is the product of the Boltzmann density and the diffusion noising kernel) yielding a Rao-Blackwellized mixture estimator that we sample to denoise the current state. Iterating this gradually transports noise toward the Boltzmann target, unlike NF- and CNF-based BGs that bridge the gap in one step. By stochastic localization \citep{grenioux2024stochastic}, $p(x_0 | x_t)$ 
loses its multimodality and concentrates as $t \to 0$ which makes it easier to sample from and simplifies the conditional NF's approximation target.

Our main contributions are:
\begin{itemize}[leftmargin=*]
    \item We propose \emph{Normalizing Flow Flow Maps} (\ourmethod{}), a %
    Boltzmann Generator that parameterizes a stochastic flow map with a normalizing flow, combining tractable densities with few-step generation.
    \item We design novel conditional normalizing flow architectures for molecule generation, adapting state-of-the-art image and video normalizing flows to this setting \citep{zhai2025normalizingflowscapablegenerative,gu2026starflow}.
    \item We propose two inference strategies, \ourmethod{}-IS and \ourmethod{}-MIS, that yield asymptotically exact samples from the Boltzmann distribution, and support them with theoretical guarantees.
    \item We evaluate \ourmethod{} against NF and flow-map baselines across molecular systems of varying sizes, demonstrating competitive or superior performance.
\end{itemize}

\section{Background}\label{sec:background}

\paragraph{Problem setup.} We represent an $N$-atom molecule configuration with 3D coordinate $x \in \mathbb{R}^{N \times 3}$ and 
we seek equilibrium samples from a Boltzmann distribution %
\begin{equation}\label{eq:boltzmann-target}
    p(x) = \frac{1}{\mathcal{Z}}\exp\{-\mathcal{E}(x)\},
    \qquad
    \mathcal{Z} = \int_{\mathbb{R}^d}\exp\{-\mathcal{E}(x)\}\,\rmd x,
\end{equation}
where $\mathcal{E}:\mathbb{R}^{N \times 3} \to\mathbb{R}$ is an energy function and $\mathcal{Z}$ is generally intractable. We assume that $\mathcal{E}(x)$ can be evaluated for any $x$. Following prior work \citep{noe2019boltzmann}, we also assume access to a small, biased dataset $\mathcal{D}=\{x^i\}_{i=1}^D$, such as configurations obtained from a short, unmixed MD trajectory. We denote by $q$ the biased data distribution underlying $\mathcal{D}$.

\paragraph{Boltzmann Generators.}
Boltzmann Generators (BGs) \citep{noe2019boltzmann,muller2019neural} use a learned model $q^\theta$ to propose configurations and then correct its mismatch with the Boltzmann target by \emph{self-normalized importance sampling (SNIS)}. Given $M$ independent proposals $x^m\sim q^\theta$, the unnormalized weights and estimator of an observable $\mathcal{O}$ are
\begin{equation}
    \widetilde w^m = \frac{\exp\{-\mathcal{E}(x^m)\}}{q^\theta(x^m)},
    \qquad
    \widehat{\mathbb{E}}_p[\mathcal{O}]
    = \frac{\sum_{m=1}^M\widetilde w^m\mathcal{O}(x^m)}
    {\sum_{m=1}^M\widetilde w^m}.
\end{equation}
The estimator is generally biased at finite $M$ but is consistent as $M\to\infty$, provided the proposal's support covers the target's. Its variance depends on the overlap between $q^\theta$ and $p$, and computing the importance weights requires evaluating the proposal density $q^\theta(x)$.

\paragraph{Normalizing flows.}
Normalizing flows (NFs) \citep{rezende2015variational,dinh2017density} are a natural choice for parameterizing $q^\theta$ because they support both sampling and exact density evaluation. A NF learns a diffeomorphism $\mathrm{T}_\theta:\mathbb{R}^d\to\mathbb{R}^d$ from data to a base density $\rho$, typically $\mathcal{N}(0,I)$, so that $\mathrm{T}_\theta^{-1}$ maps base samples to data space. Its density follows from the change-of-variables formula,
\begin{equation}
    \log q^\theta(x)
    = \log\rho(\mathrm{T}_\theta(x))
    + \log\abs{\det\frac{\partial\mathrm{T}_\theta(x)}{\partial x}}.
\end{equation}
In practice, $\mathrm{T}_\theta$ is a composition of invertible layers whose inverses and Jacobian determinants are tractable. These structural requirements enable exact likelihoods but constrain the neural-network architectures that can be used, thereby limiting the expressivity of the learned distribution.

\paragraph{Diffusion paths and denoising posteriors.}
Diffusion models \citep{ho2020denoising,Song2021score} transform a tractable reference $p_T$ into a target $p_0$ through learned denoising transitions. To train them, consider a prescribed continuous-time forward process that progressively noises $X_0\sim p_0$, yielding states $X_t\sim p_t$ for $t\in[0,T]$. 
We use a linear Gaussian noising kernel, yielding the marginal 
\noindent
\begin{minipage}{0.48\linewidth}
\begin{equation}\label{eq:noising-conditional}
    p_{t|0}(x_t\mid x_0)
    = \mathcal{N}(x_t;\alpha_t x_0,\sigma_t^2 I),
\end{equation}
\end{minipage}\hfill
\begin{minipage}{0.50\linewidth}
\begin{equation}\label{eq:noised-marginal}
    p_t(x_t)
    = \int p_0(x_0)p_{t|0}(x_t\mid x_0)\,\rmd x_0.
\end{equation}
\end{minipage}
We use the variance-exploding (VE) schedule of \citet{karras2022elucidating}, for which $\alpha_t=1$ and $\sigma_t=t$. At sufficiently large $T$, the signal of $p_0$ in $p_T$ is small relative to the added noise, so the terminal marginal is approximated by the tractable reference distribution $p_T \approx \mathcal{N}(0,\sigma_T^2I)$.

Generation %
starts from $X_T\sim p_T$ and successively samples less noisy states. For $0\leq s<t\leq T$, this requires the denoising transition $p_{s|t}(x_s | x_t)$.
Given $x_t$, one  first infers %
from the \emph{denoising posterior}
\begin{equation}\label{eq:denoising-posterior-background}
    p_{0|t}(x_0\mid x_t)
    = %
    p_0(x_0)p_{t|0}(x_t\mid x_0)/ p_t(x_t). 
\end{equation}
Marginalizing over these possible clean states gives the exact decomposition
\begin{equation}\label{eq:ddpm_kernel}
    p_{s|t}(x_s\mid x_t)
    = \int p_{s|t,0}(x_s\mid x_t,x_0)\,
    p_{0|t}(x_0\mid x_t)\,\rmd x_0,
\end{equation}
where $p_{s|t,0}$ is an endpoint conditioned Gaussian  given the clean state $x_0$ and the noisy state $x_t$,
\begin{equation}\label{eq:gaussian-bridge}
    p_{s|t,0}(x_s\mid x_t,x_0)
    = \mathcal{N}\!\left(
        x_s;a_{s,t}x_0+b_{s,t}x_t,\Sigma_{s,t}
    \right),
\end{equation}
with the parameters determined by the forward process: 
\begin{equation*}
      (\alpha_{t|s},\sigma_{t|s}^2)
      = \left({\alpha_t}/{\alpha_s},\,
      \sigma_t^2-\alpha_{t|s}^2\sigma_s^2\right), \,      (a_{s,t},b_{s,t},\Sigma_{s,t})
      = \left({\alpha_s\sigma_{t|s}^2}/{\sigma_t^2},\,
      {\alpha_{t|s}\sigma_s^2}/{\sigma_t^2},\,
      {\sigma_{t|s}^2\sigma_s^2}/{\sigma_t^2}I\right).
  \end{equation*}
Diffusion models approximate it  by a point mass at a learned conditional mean $\mu^\theta_{0|t}(x_t)\approx\mathbb{E}[X_0 | X_t=x_t]$, trained on %
paired data $(x_0,x_t)$. %
The reverse transition is approximated as Gaussian 
\begin{equation}\label{eq:ddpm_kernel_gaussian}
    p_{s|t}(x_s\mid x_t)
    \approx
    \mathcal{N}\!\left(
        x_s;a_{s,t}\mu^\theta_{0|t}(x_t)+b_{s,t}x_t,\Sigma_{s,t}
    \right).
\end{equation}
With an accurate conditional mean prediction, this approximation recovers the reverse diffusion dynamics in the continuous-time limit as the step size tends to zero \citep{Song2021score}. With finite steps, however, this approximation can be poor for coarse step sizes.

\section{Normalizing Flow Flow Maps (\ourmethod{})}

\begin{figure}
    \centering
    \includegraphics[width=\linewidth]{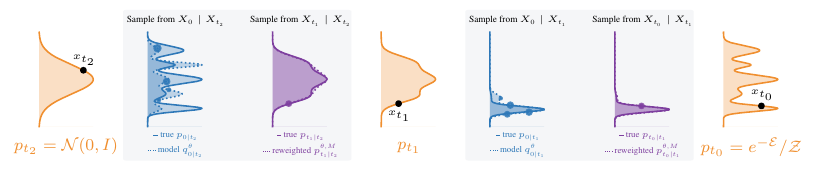}
    \caption{Conceptual illustration of the \ourmethod{} mechanism. Given a sample $x_t$ drawn from the marginal distribution at noise level $t$, we first draw $M$ samples {\color{nfmblue}$\{x_{0|t}^m\}_{m=1}^M$} from the learned posterior {\color{nfmblue}$q_{0|t}^\theta(\cdot|x_t)$}. We then reweight them via \Cref{eq:denoising_snis} to construct the posterior {\color{nfmpurple}$p_{s|t}^{\theta,M}$}, which proposes a sample {\color{nfmpurple}$x_s$} at the next noise level $s$. We repeat this procedure until reaching the target distribution.}
    \label{fig:fig1}
\end{figure}

While existing BG designs compose approximate transitions and correct accumulated error only at the end by reweighting the full composition, \ourmethod{} instead reweighs each denoising transition directly along the diffusion path of \Cref{sec:background}, with target $p_0=p$. Since each transition is individually asymptotically unbiased, so is their composition, with no need for a final correction. Concretely, \ourmethod{} learns a conditional normalizing flow $q^\theta_{0|t}(x_0|x_t)$ approximating the denoising posterior at every noise level: the flow proposes clean configurations from the current noisy state, importance sampling corrects them toward the true posterior, and the result determines the transition via \Cref{eq:ddpm_kernel}.

\Cref{subsec:method:inference} describes the basic sampler, \ourmethod{}-IS; \Cref{subsec:method:training} covers its training objective and flow architecture; \Cref{subsec:method:mis} improves inference through a trajectory-sharing mechanism and a custom MCMC scheme; and \Cref{subsec:method:theory} provides theoretical intuition and guarantees.

\subsection{Boltzmann sampling along the diffusion path}\label{subsec:method:inference}

An ideal sampler along the diffusion path would follow the exact denoising transitions $p_{s|t}$ in \Cref{eq:ddpm_kernel}, starting from $p_T$ and ending at the Boltzmann distribution $p_0=p$. The  intractable component in \Cref{eq:ddpm_kernel} is the Boltzmann denoising posterior $p_{0|t}$, which satisfies by \Cref{eq:denoising-posterior-background}
\begin{align}\label{eq:target_denoising_posterior}
    p_{0|t}(x_0\mid x_t)
    &\propto \exp\{-\mathcal{E}(x_0)\}\,p_{t|0}(x_t\mid x_0).
\end{align}
While we cannot sample from it, the right hand size of \Cref{eq:target_denoising_posterior} can be evaluated pointwise. 
For now, suppose we have a posterior sampler $q^\theta_{0|t}$ trained on the biased dataset $q$, we would like to treat it as a proposal and employ SNIS to sample from the target posterior distribution $p_{0|t}$.  Although $q^\theta_{0|t}$ generally deviates from the target posterior $p_{0|t}$, we instead use it as a proposal in a SNIS scheme, which corrects for this mismatch before constructing the next denoising transition.

\paragraph{SNIS-corrected transition kernel.}
Fix $0\leq s<t\leq T$ and a noisy state $x_t$. We draw $M$ independent clean states $x^m_{0|t}\sim q^\theta_{0|t}(\cdot| x_t)$ and assign them self-normalized importance weights
\begin{equation}\label{eq:denoising_snis_weights}
    w_t^m = \frac{\widetilde w_t^m}{\sum_{j=1}^M\widetilde w_t^j},
    \qquad
    \widetilde w_t^m
    = \frac{\exp\{-\mathcal{E}(x^m_{0|t})\}\,p_{t|0}(x_t\mid x^m_{0|t})}
    {q^\theta_{0|t}(x^m_{0|t}\mid x_t)}.
\end{equation}
As $M\to\infty$, the weighted empirical measure $\sum_{m=1}^M w_t^m\delta_{x^m_{0|t}}$ approximates the denoising posterior $p_{0|t}(\cdot | x_t)$. Substituting it for $p_{0|t}$ inside the integral of \Cref{eq:ddpm_kernel} and marginalizing analytically over the Gaussian bridge $p_{s|t,0}$ turns the true transition into a finite mixture, the SNIS-corrected transition 
\begin{equation}\label{eq:denoising_snis}
    p_{s|t}(x_s\mid x_t)
    \approx
    \widehat{p}^{M,\theta}_{s|t}(x_s\mid x_t)
    := \sum_{m=1}^M w_t^m\,p_{s|t,0}(x_s\mid x_t,x^m_{0|t}),
\end{equation}

in place of the true kernel's infinite mixture over $p_{0|t}$. This is a Rao-Blackwellized estimator: we retain the full conditional law $p_{s|t,0}$ and mix over it analytically, rather than drawing and reweighting a single noisy sample per proposal. The exact and approximate transitions share the same endpoint-conditioned Gaussian $p_{s|t,0}$ and differ only through $p_{0|t}$ versus its finite-mixture approximation. Since $x_0$ enters the Gaussian mean with coefficient $a_{s,t}$, errors in the denoising posterior are attenuated by $|a_{s,t}|$ when propagated to the next state, as formalized in \Cref{subsec:method:theory}.

\begin{wrapfigure}[14]{r}{0.56\linewidth}
\centering
\begin{minipage}{\linewidth}
\vspace{-1.25em}
\begin{algorithm}[H]
\SetAlgoNlRelativeSize{-1}
\LinesNotNumbered
\DontPrintSemicolon
\small
\SetAlgoLined

\KwIn{Grid $\{t_k\}_{k=0}^K$, particle count $M$, flow $q^\theta_{0|t}$, energy $\mathcal{E}$.}
\KwOut{$x_0$ approximately distributed as $p(\cdot)\propto\exp\{-\mathcal{E}(\cdot)\}$.}

Sample $x_K\sim\mathcal{N}(0,\sigma_T^2 I)$\;

\For{$k=K,\ldots,1$}{
    Draw $x^m_{0|k}\overset{\mathrm{iid}}{\sim} q^\theta_{0|t_k}(\cdot| x_k)$, $m=1,\ldots,M$\;

    Compute weights $\widetilde{w}_k^m = \dfrac{\exp\{-\mathcal{E}(x^m_{0|k})\}\,p_{t_k|0}(x_k| x^m_{0|k})}{q^\theta_{0|t_k}(x^m_{0|k}| x_k)}$\;

    Sample $x_{k-1}\sim \sum_{m=1}^M \dfrac{\widetilde{w}_k^m}{\sum_{j=1}^M\widetilde{w}_k^j}\, p_{t_{k-1}|t_k,0}(\cdot| x_k,x^m_{0|k})$\;
}
\Return{$x_0$}\;

\caption{\ourmethod{}-IS for a single trajectory}
\label{alg:chained_denoising}
\end{algorithm}
\end{minipage}
\end{wrapfigure}

Chaining the SNIS-corrected transition in \Cref{eq:denoising_snis} along a discretized diffusion path yields the full sampler, \ourmethod{}-IS: given a time grid $0=t_0<t_1<\cdots<t_K=T$, we initialize $x_K\sim\mathcal{N}(0,\sigma_T^2I)$ and apply $\widehat{p}^{M,\theta}_{s|t}$ successively down to $x_0$, as summarized in \Cref{alg:chained_denoising}. The approximation \Cref{eq:denoising_snis} converges to the corresponding true denoising kernel as $M\to\infty$. The practical error also includes the terminal approximation $p_T\approx\mathcal{N}(0,\sigma_T^2I)$ and the finite-$M$ errors accumulated along the chain. We study their propagation in \Cref{subsec:method:theory} and \Cref{app:sec:proofs}.

\subsection{Training}\label{subsec:method:training}

We now discuss how to train the conditional flow $q^\theta_{0|t}$. Although equilibrium samples from the Boltzmann target $p$ are unavailable, we have access to a biased dataset $\mathcal{D}$ drawn from some distribution $q:=q_0$. Applying the same forward noising process to $q_0$ induces a noisy marginal $q_t$ and denoising posterior $q_{0|t}$, and we train $q^\theta_{0|t}$ to approximate $q_{0|t}$ from noisy–clean pairs generated from $\mathcal{D}$. The bias inherited from $\mathcal{D}$ is corrected only later, at inference time.

\paragraph{Training objective.}
For a positive time-weighting function $\lambda(t)$, we minimize the forward KL from $q_{0|t}$ to $q^\theta_{0|t}$, averaged over the noisy marginal $q_t$,
\begin{align}
    \label{eq:nf2m_kl}
    \mathcal{L}_{\mathrm{KL}}(\theta)
    &= \int_0^T\lambda(t)\,
    \mathbb{E}_{X_t\sim q_t}
    \left[
        \mathbb{D}_{\mathrm{KL}}\left(
            q_{0|t}(\cdot\mid X_t)
            \,\|\,
            q^\theta_{0|t}(\cdot\mid X_t)
        \right)
    \right]\rmd t.
\end{align}
We note that the reverse KL $\mathbb{D}_{\mathrm{KL}}(q^\theta_{0|t}\|p_{0|t})$ would also be tractable up to a constant via \Cref{eq:target_denoising_posterior}. We leave this alternative to future work and focus here on the forward-KL formulation.

Since $q_t(x_t)q_{0|t}(x_0| x_t) = q_0(x_0)p_{t|0}(x_t| x_0)$ by Bayes rule, expanding the KL and dropping terms independent of $\theta$ reduces \Cref{eq:nf2m_kl} to the tractable conditional maximum-likelihood objective
\begin{align}\label{eq:nf2m_loss}
    \mathcal{L}_{\text{\ourmethod{}}}(\theta)
    = -\int_0^T\lambda(t)\,
    \mathbb{E}_{X_0\sim q_0}
    \mathbb{E}_{X_t\sim p_{t|0}(\cdot\mid X_0)}
    \left[\log q^\theta_{0|t}(X_0\mid X_t)\right]\rmd t.
\end{align}
Training thus requires only samples from biased data $\mathcal{D}$, noise levels $t$, and the forward noising kernel. In principle, \ourmethod{} applies to any conditional family $q_{0|t}^\theta$ that supports sampling and exact likelihood evaluation. For example, an isotropic Gaussian with learned mean and fixed variance $v_tI$ gives
\begin{align}
    q^\theta_{0|t}(x_0\mid x_t)
    &= \mathcal{N}(x_0;\mu^\theta_{0|t}(x_t),v_tI),
\end{align}
then \Cref{eq:nf2m_loss} becomes, up to terms independent of $\theta$,
\begin{align}
    \int_0^T\frac{\lambda(t)}{2v_t}\,
    \mathbb{E}_{X_0\sim q_0}
    \mathbb{E}_{X_t\sim p_{t|0}(\cdot\mid X_0)}
    \left[\norm{X_0-\mu^\theta_{0|t}(X_t)}^2\right]\rmd t.
\end{align}
This is the usual $x_0$-prediction denoising objective in diffusion models \citep{ho2020denoising, karras2022elucidating}, where $\mu_{0| t}^\theta$ is optimized to predict the conditional mean $\mathbb{E}[X_0 | X_t = x_t]$ induced by 
$q_0.$

In this work, we replace the Gaussian conditional with a flexible conditional NF architecture. This choice enables representing multimodal denoising posteriors while retaining the likelihood required for training in \Cref{eq:nf2m_loss} and for inference in \Cref{eq:denoising_snis_weights}.

\paragraph{Conditional Normalizing Flow architecture.} We develop conditional molecular variants of TarFlow \citep{zhai2025normalizingflowscapablegenerative} and StarFlow \citep{gu2026starflow}. Both are autoregressive flows \citep{huang2018neuralautoregressiveflows} that use causal Transformers for autoregressive affine coupling, giving a triangular Jacobian that enables parallel density evaluation with sequential sampling. TarFlow stacks several medium conditioned blocks while StarFlow uses one large conditional block followed by small unconditional ones.
Following \citet{tan2025scalable}, we represent an $N$-atom configuration as $N$ tokens in $\mathbb{R}^3$, conditioning both backbones on noise level $t$ (via DiT-style adaptive layer norm \citep{peebles2023scalable}) and on $x_t$ itself, either through cross-attention (features of $x_0$ as queries, $x_t$ as keys/values) or a prefix mechanism (prepending $x_t$ features to the autoregressive sequence). Details in \Cref{app:molecular-tarflow}.

\paragraph{Molecular symmetries.} Following \citet{tan2025scalable}, we enforce translation and rotation invariance by centering and randomly rotating configurations rather than using equivariance. We lift them to $\mathbb{R}^{N\times3}$ with a Gaussian centroid $c_0\sim\mathcal{N}(0,N^{-1}I_3)$, tractable under a standard Gaussian base. To reduce centroid-induced importance-weight variance, a closed-form adjustment to $\log q^\theta_{0|t}(x_0|x_t)$ removes the rotationally irrelevant angular component of the centroid's offset from its conditional mean given $x_t$. This recovers \citet{tan2025scalable}'s unconditional correction at high noise (\Cref{app:com-adjustment-conditional}).

\subsection{Improving the inference}\label{subsec:method:mis}

\paragraph{Multiple Importance Sampling.} \ourmethod{}-IS shares candidates only at initialization. Can this sharing be extended further along the chain? When $p$ is a Gaussian $\mathcal{N}(\mu, \gamma^2 \Idd)$, one can show that 
\begin{align}\label{eq:posterior_kl_gaussian}
    \mathbb{D}_{\mathrm{KL}}(p_{0|t}(\cdot \mid x) ~\|~ p_{0|t}(\cdot \mid y)) = \frac{\alpha_t^2\gamma^2 \norm{x-y}^2}{2\sigma_t^2(\alpha_t^2\gamma^2+\sigma_t^2)} ,\forall x, y \in \mathbb{R}^d,t \in [0, T].
\end{align}
Posteriors conditioned on nearby states are themselves close, which suggests a simple idea: when $x$ and $y$ are close, candidates drawn from $q^\theta_{0|t}(\cdot | x)$ should also be informative proposals for $p_{0|t}(\cdot | y)$.

This is exactly the idea behind \emph{\ourmethod{}-MIS}: each trajectory reuses candidates generated for \emph{other} trajectories, rather than drawing its own (\Cref{fig:mis_cartoon}). Concretely, instead of keeping the $M$ candidates from each of $N$ proposals separate, we pool them into $N$ groups of $C$ shared proposal components, so that each trajectory draws on $CM$ candidates rather than $M$, for the same total of $NM$ flow samples as \ourmethod{}-IS. This comes at the cost of evaluating additional cross-likelihoods for the importance weights, an overhead that remains comparatively inexpensive for our autoregressive flow architectures.

Such sharing is especially valuable at high noise levels, where noisy states are less informative about $x_0$ and the posterior across trajectories approach the common Boltzmann target. 
follows literature on multiple importance sampling \citep{veach1995optimally, owen2000safe, he2014optimalmixtureweightsmultiple, bugallo2017adaptive, elvira2019generalized}. See \Cref{app:mis} for algorithmic details. 

\paragraph{Posterior-adapted MCMC rejuvenation.} We can refine the IS or MIS sample with a few steps of MCMC targeting the posterior $p_{0|t}$ directly. Its score decomposes as $\nabla_{x_0} \log p_{0|t}(x_0|x_t) = \nabla \log p(x_0) - A x_0 + b$, with $A = \alpha_t^2/\sigma_t^2$ and $b = (\alpha_t/\sigma_t^2) x_t$, giving the Langevin diffusion
\begin{equation}
  \rmd Y_s = \left(-A Y_s + b + \nabla \log p(Y_s)\right) \rmd s + \sqrt{2}\, \rmd W_s.
\end{equation}
{
Holding $\nabla \log p$ fixed and integrating the affine drift exactly yields the Gaussian kernel
}
\begin{equation}
  Y_{s+h} \mid Y_s \;\sim\; \mathcal{N}\left(e^{-Ah} Y_s + \tfrac{1-e^{-Ah}}{A}\left(b + \nabla \log p(Y_s)\right),\; \tfrac{1-e^{-2Ah}}{A} I\right),
\end{equation}
or, using the noise-to-signal ratio $\tau = (1-e^{-2Ah})/(2A) \in (0, \tau_{\max}]$, with $\tau_{\max}=1/(2A)$,
\begin{equation}
  Y_{s+h} \mid Y_s \;\sim\; \mathcal{N}\left(c(\tau) Y_s + \varphi(\tau)\left(b + \nabla \log p(X_s)\right),\; 2\tau I\right).
\end{equation}
A Metropolis--Hastings accept/reject step makes the proposal kernel invariant under $p_{0|t}$, so rejuvenated particles retain their original importance weights. See \Cref{app:maexl} for the full construction.
\subsection{Theoretical guarantees}\label{subsec:method:theory}

Why approximate the denoising posterior at intermediate noise levels instead of applying SNIS 
directly to the final Boltzmann target? An intermediate denoising step need not inherit the full 
error of the clean-state approximation. The following proposition quantifies how the 
endpoint-conditioned Gaussian transition controls this error; its proof is given in \Cref{app:sec:proofs}.

\begin{proposition}\label{prop:wasserstein_bridge_main}
Let $\hat{p}_{0|t}$ be an approximation of $p_{0|t}$ and define the corresponding denoising kernel
\begin{align*}
    \hat{p}_{s|t}(x_s \mid x_t) = \int p_{s|t,0}(x_s \mid x_t, x_0)\,\hat{p}_{0|t}(x_0 \mid x_t)\,\rmd x_0.
\end{align*}
Let $\ell \geq 1$, $x_t \in \R^d$, and assume that both $p_{0|t}(\cdot| x_t)$ and $\hat{p}_{0|t}(\cdot| x_t)$ have finite $\ell$-th moment. Then
\begin{align*}
    W_\ell\left(p_{s|t}(\cdot\mid x_t),\,\hat{p}_{s|t}(\cdot\mid x_t)\right) \leq a_{s,t}\,W_\ell\left(p_{0|t}(\cdot\mid x_t),\,\hat{p}_{0|t}(\cdot\mid x_t)\right).
\end{align*}
\end{proposition}

\paragraph{Error attenuation and posterior concentration.} \Cref{prop:wasserstein_bridge_main} shows that errors in approximating $p_{0|t}$ are damped by $a_{s,t}$ when propagated to $p_{s|t}$. Under the VE schedule, $a_{s,t}=1-(s/t)^2$ grows with the relative step size $(t-s)/t$, so damping is strongest where the schedule takes relatively small steps (\Cref{fig:a_s_t}), which is why a poorly approximated posterior in \Cref{fig:fig1} can still yield a good SNIS approximation $q^{\theta,M}_{s|t}$. Near $t\approx T$, where $p_{0|t}(\cdot| x_t)\approx p_0$ is broad and hard to learn, the \citet{karras2022elucidating} discretization takes the smallest relative steps and so makes $a_{s,t}$ smallest, attenuating their propagation. Conversely, at low noise $p_{0|t}(\cdot | x_t)$ concentrates around $x_t$ and loses its multimodality, which should make the flow easier to approximate but a sharply peaked target remains hard to represent with affine couplings. MAExL rejuvenation covers this regime, being most accurate at low noise (\Cref{app:maexl}), so the larger $a_{s,t}$ acts on an error that is already small.

\paragraph{Per-step error.} \Cref{prop:wasserstein_bridge_main} bounds a single step for an abstract $\hat p_{0|t}$. To obtain a guarantee for the full sampler, we substitute the SNIS estimator for $\hat p_{0|t}$ and track how per-step errors accumulate along the chain (proofs in \Cref{app:sec:proofs}).

\begin{proposition}[Per-step SNIS error, informal version of \Cref{thm:per_step_error}] \label{prop:per_step_error}
Let $0<s<t$ and let $q^{\theta,M}_{s|t}$ use proposal $q_{0|t}$. Assume $q_{0|t}$ and its square-weighted tilt have sub-Gaussian tails of rate $\kappa_t(x_t)$. For any bounded $\psi$, the bias and MSE of $q^{\theta,M}_{s|t}(\psi | x_t)$ relative to $p_{s|t}(\psi\mid x_t)$ satisfy,
\begin{align*}
    \mathrm{bias} = \mathcal{O}\left(\frac{\rho_t(x_t)\gamma_{t,s}\kappa_t(x_t)\sqrt d}{M}\right), \qquad
    \mathrm{MSE} =  \mathcal{O}\left(\frac{\rho_t(x_t)\gamma_{t,s}^2\kappa_t(x_t)^2d}{M}\right),
\end{align*}
up to $\mathcal{O}(M^{-2})$ terms and logarithmic factors in $M$, where $\rho_t(x_t)=1+\chi^2(p_{0|t}(\cdot | x_t)\|q_{0|t}(\cdot | x_t))$ measures proposal quality and $\gamma_{t,s}$ is the schedule-dependent bridge sensitivity \Cref{eq:bridge_sensitivity}.
\end{proposition}
As in \Cref{prop:wasserstein_bridge_main}, the bounds reflect proposal fit ($\rho_t$) and sensitivity to residual $x_0$-error ($\gamma_{t,s}\kappa_t$). SNIS adds the number of importance-sampling particles $M$, which controls the decay of bias and MSE.
\paragraph{Error propagation.} \Cref{prop:per_step_error} gives a per-step control statement about a single transition. The sampler, however, chains $K$ such transitions to go from $t_K$ down to  $\tau$, so what ultimately matters is how these per-step errors combine along the full path. We answer this next.

\begin{proposition}[Path error, informal version of \Cref{cor:chained_no_mcmc}]\label{prop:path_error_informal}
Let $T=t_K>\cdots>t_0= \tau > 0$ and let  $q^M_{t_0}$ denote the estimator obtained by chaining the $M$-particle SNIS kernels of \Cref{prop:per_step_error} along the schedule, started from the exact law at $t_K$. 
For any bounded test function $\psi$, let
$\Delta_M := q_{t_0}^M(\psi)-p_{t_0}(\psi)$. Then
\begin{align*}
    \abs{\mathbb E[\Delta_M]}
    \leq \sum_{k=1}^K \mathbb E_{p_{t_k}}[B_k(X_{t_k})],
    \qquad
    \sqrt{\mathbb E[\Delta_M^2]}
    \leq \sum_{k=1}^K
    \sqrt{\mathbb E_{p_{t_k}}[V_k(X_{t_k})]},
\end{align*}
where $B_k,V_k=O(1/M)$ are exactly the per-step bias and MSE rates of \Cref{prop:per_step_error} at step $k$.
\end{proposition}

The bias thus accumulates only \emph{additively}, and the root-MSE only through a \emph{triangle inequality}, 
across the $K$ denoising steps.

\section{Related Works}

\paragraph{BGs with tractable correction weights.}

Sequential Boltzmann Generators (SBG) use sequential Monte Carlo (SMC) to sample along a geometric annealing path between a trained TarFlow proposal and the Boltzmann target. However, their ablation results show that SMC  does not consistently improve over  SNIS on the same trained proposal across systems and metrics
\citep{tan2025scalable}. \ourmethod{} also deploys a TarFlow sequentially, but both %
inference %
and %
learning %
differ fundamentally: rather than annealing a fixed proposal toward the target, \ourmethod{} performs progressive denoising transitions, and the flow itself is trained as a conditional model rather than an unconditional one.

The Autoregressive Boltzmann Generator (ArBG) models coordinate-bin probabilities with uniform dequantization \citep{rehman2026autoregressiveboltzmanngenerators}, yielding an exactly evaluable but bin-wise coarse density. Its coordinate-wise factorization enables twisted SMC using temporarily capped peptide energies to resample partial structures.
Stochastic normalizing flows interleave invertible maps with MCMC or Langevin transitions, weighting entire sampling paths \citep{wu2020stochasticnormalizingflows}.
\ourmethod{} instead corrects denoising proposals over complete molecular configurations at each noise level, avoiding both the binning artifacts of ArBG and the need to evaluate energies of artificially capped partial structures.

\paragraph{BGs with approximate correction.}

Free-form Flows jointly train a network and its pseudo-inverse by maximum likelihood \citep{draxler2024freeformflowsmakearchitecture}. Continuous normalizing flows keep maximum-likelihood training but replace the invertible network with a neural ODE \citep{kohler2020equivariant}, whose likelihood requires integrating the costly probability-flow ODE through the network's divergence, with many steps needed for both accuracy and invertibility. Later work cuts training cost via flow matching but leaves this inference cost untouched \citep{klein2023equivariant, klein2024transferable}. FALCON instead distills into a few-step flow map, replacing the full ODE integration with a handful of Jacobian evaluations, at the price of only approximate invertibility \citep{rehman2025falconfewstepaccuratelikelihoods}. F2D2 drops exact Jacobians entirely, jointly distilling the flow map and its density evolution via Hutchinson trace estimates \citep{ai2026jointdistillationfastlikelihood, Hutchinson01011990}, and SCALLOP replaces this with a Hutchinson-free conditional divergence-matching loss \citep{ouyang2026fewstep}. Across these distillation-based methods, the learned density's approximation error enters the importance weights directly and persists regardless of proposal count. \ourmethod{} instead evaluates its conditional proposal densities exactly.

\paragraph{Stochastic flow maps and intermediate correction.}
Distributional diffusion models learn the conditional distribution of clean states given noisy states, enabling larger denoising steps \citep{xiao2022tacklinggenerativelearningtrilemma, debortoli2025distributionaldiffusionmodelsscoring}. Meta Flow Maps and Diamond Maps use clean-state candidates for intermediate correction toward reward-tilted distributions \citep{potaptchik2026metaflowmapsenable,holderrieth2026diamondmapsefficientreward}. Their implicit conditional generators do not directly provide the proposal densities needed for Boltzmann importance weights. \ourmethod{} supplies them through a conditional normalizing flow.

\section{Experiments}
\label{sec:experiments}

\paragraph{Benchmarks.}
We evaluate on the five all-atom peptide systems of \citet{tan2025scalable} reusing their MD training trajectories, splits, force fields, and evaluation protocol throughout
so our numbers are directly comparable to theirs. We report the energy Wasserstein distance (\energyw, sensitive to fine local structure) and the dihedral torus Wasserstein distance (\torusw, sensitive to metastable-state occupancy), each against $10^4$ held-out reference configurations over three seeds (\Cref{app:datasets}).

\paragraph{Baselines.}
We compare \ourmethod{} (variants \ourmethod{}-IS and \ourmethod{}-MIS) against SNIS, the exact-Jacobian flow maps FALCON and FALCON-A \citep{rehman2025falconfewstepaccuratelikelihoods}, and the distilled-density flow maps F2D2 \citep{ai2026jointdistillationfastlikelihood} and SCALLOP \citep{ouyang2026fewstep}. The flow-maps are distilled from the same per-system velocity teacher, which uses a DiT-style backbone \citep{peebles2023scalable}. Our SNIS baseline applies direct self-normalized importance sampling to an unconditional molecular flow proposal trained following \citet{tan2025scalable}, without their annealed SMC procedure. %
We additionally compare against the NF-based SMC sampler from \citet{tan2025scalable} in \Cref{tab:results_sbg}: it does not consistently outperform the cheaper direct-SNIS baseline, which we retain in the main comparison. %
For each system, their total wall-clock training budget, including the teacher and flow-map pretraining, matches the training time of the selected \ourmethod{} checkpoint (\Cref{sec:hyperparameters}).

\paragraph{Budgets.}
Proposal budgets otherwise follow each method's practical operating point: \ourmethod{} draws its final $10^4$ particles directly, while SNIS and the four flow-map baselines first draw $10^6$ proposals and then resample $10^4$ samples for evaluation (except FALCON/FALCON-A on chignolin, where exact per-step Jacobian costs cap the pool at $10^5$). For \ourmethod{}, we run 512 MAExL steps followed by 128 post-hoc MALA steps targeting the Boltzmann distribution. For all other methods, we apply enough post-hoc MALA steps to the final $10^4$ samples to match \ourmethod{}'s total target-energy evaluations; all metrics use $10^4$ final samples. \Cref{tab:hparam_langevin_cost} compares the cost of this refinement to the total cost of each method, and \Cref{tab:posthoc_mala_steps} reports the number of post-hoc MALA steps used for each method.

\subsection{Main results}
\Cref{tab:wasserstein} reports the energy and dihedral Wasserstein distances, while \Cref{tab:cost} reports training and inference times. The systems are ordered by increasing size in both tables, from ALA-2 to chignolin.

\begin{table*}[t]
\centering
\caption{Energy and torus Wasserstein distances on ALA-$N$ and chignolin systems, from 3 random seeds, computed on the test split vs.\ $10^4$ reference samples. Best \textbf{bold}, second \underline{underlined}.}
\label{tab:wasserstein}
\scriptsize
\renewcommand{\arraystretch}{1.25}
\setlength{\tabcolsep}{3pt}
\resizebox{\linewidth}{!}{%
\begin{tabular}{c l c H H c c c c}
\toprule
 & & SNIS & \ourmethod{}-IS & \ourmethod{}-MIS & FALCON & FALCON-A & F2D2 & SCALLOP \\
\midrule
\multirow{2}{*}{\textbf{ALA-2}} & $\mathcal{E}$-$\mathcal{W}_2$ & $0.768\err{0.031}$ & $\best{0.132}\err{0.013}$ & $\second{0.242}\err{0.080}$ & $0.759\err{0.047}$ & $0.595\err{0.329}$ & $0.890\err{0.030}$ & $0.874\err{0.033}$ \\
 & $\mathbb{T}$-$\mathcal{W}_2$ & $\best{0.242}\err{0.012}$ & $1.050\err{0.011}$ & $1.239\err{0.006}$ & $0.519\err{0.253}$ & $0.862\err{1.089}$ & $0.752\err{0.071}$ & $\second{0.374}\err{0.043}$ \\
\midrule
\multirow{2}{*}{\textbf{ALA-3}} & $\mathcal{E}$-$\mathcal{W}_2$ & $0.268\err{0.067}$ & $\second{0.164}\err{0.011}$ & $\best{0.157}\err{0.009}$ & $0.181\err{0.049}$ & $0.241\err{0.144}$ & $0.553\err{0.133}$ & $0.703\err{0.186}$ \\
 & $\mathbb{T}$-$\mathcal{W}_2$ & $\best{0.362}\err{0.019}$ & $0.490\err{0.016}$ & $0.465\err{0.007}$ & $\second{0.463}\err{0.170}$ & $1.100\err{0.706}$ & $1.049\err{0.476}$ & $0.847\err{0.196}$ \\
\midrule
\multirow{2}{*}{\textbf{ALA-4}} & $\mathcal{E}$-$\mathcal{W}_2$ & $1.515\err{0.351}$ & $\best{0.988}\err{0.027}$ & $\second{1.107}\err{0.013}$ & $1.763\err{0.192}$ & $1.985\err{0.737}$ & $2.748\err{2.096}$ & $2.981\err{0.970}$ \\
 & $\mathbb{T}$-$\mathcal{W}_2$ & $\best{0.777}\err{0.131}$ & $\second{0.824}\err{0.041}$ & $0.882\err{0.051}$ & $0.893\err{0.140}$ & $1.462\err{1.092}$ & $2.430\err{0.115}$ & $2.368\err{0.189}$ \\
\midrule
\multirow{2}{*}{\textbf{ALA-6}} & $\mathcal{E}$-$\mathcal{W}_2$ & $0.923\err{1.017}$ & $1.018\err{0.063}$ & $1.301\err{0.092}$ & $\best{0.610}\err{0.410}$ & $\second{0.680}\err{0.493}$ & $2.161\err{1.723}$ & $3.767\err{1.724}$ \\
 & $\mathbb{T}$-$\mathcal{W}_2$ & $1.767\err{0.675}$ & $\best{0.990}\err{0.026}$ & $\second{1.027}\err{0.021}$ & $1.716\err{0.332}$ & $1.557\err{0.025}$ & $2.927\err{0.288}$ & $3.126\err{0.462}$ \\
\midrule
\multirow{2}{*}{\textbf{Chignolin}} & $\mathcal{E}$-$\mathcal{W}_2$ & $7.748\err{5.290}$ & $\best{4.331}\err{0.262}$ & $6.112\err{0.129}$ & $4.807\err{1.066}$ & $5.028\err{3.624}$ & $4.793\err{2.335}$ & $\second{4.770}\err{1.076}$ \\
 & $\mathbb{T}$-$\mathcal{W}_2$ & $3.859\err{0.167}$ & $\second{2.597}\err{0.054}$ & $\best{2.581}\err{0.018}$ & $2.989\err{0.085}$ & $3.448\err{0.619}$ & $3.930\err{0.450}$ & $3.568\err{0.200}$ \\
\bottomrule
\end{tabular}%
}
\end{table*}

\paragraph{Distributional accuracy across syste, scales.} 
\Cref{tab:wasserstein} shows that \ourmethod{}-IS/MIS lead every system except ALA-6 on $\mathcal{E}$-$\mathcal{W}_2$. %
On the smallest systems ALA-2/ALA-3, SNIS 
achieves the lowest \torusw, reflecting the closest match to the reference dihedral-angle distribution. However, SNIS's \torusw\space rises on the larger systems. In contrast,
\ourmethod{}'s progressive correction 
maintains stronger performance, achieving the lowest $\mathbb{T}$-$\mathcal{W}_2$ on the largest systems, ALA-6 and chignolin, %
with FALCON and FALCON-A in between.
FALCON achieves the lowest $\mathcal{E}$-$\mathcal{W}_2$ on ALA-6, where \ourmethod{} does not lead on that metric. F2D2 and SCALLOP have substantially higher $\mathcal{E}$-$\mathcal{W}_2$ on ALA-4 and ALA-6.

We also ablate \ourmethod{}'s in-loop MAExL: it improves energy accuracy and stability with little change in torsion distances, while increasing from 512 to 1024 steps brings little further benefit for IS (\Cref{tab:maexl_ablation}). Separately, removing post-hoc MALA from all methods changes \ourmethod{}'s results relatively little but substantially worsens energy accuracy for several baselines (\Cref{tab:post_hoc_mala}).

\begin{wraptable}[17]{r}{0.5\linewidth}
\vspace{-1.2em}
\centering
\caption{Inference time at the matched target-energy budget (sampling + post-hoc MALA ladder) and total training wall-clock across all required stages (mean $\pm$ std, 3 seeds). Chignolin FALCON/FALCON-A ran at a reduced $10^5$-sample budget (vs. $10^6$ elsewhere; $10^4$ for \ourmethod{}).}
\label{tab:cost}
\tiny
\renewcommand{\arraystretch}{1.3}
\setlength{\tabcolsep}{1pt}
\resizebox{\linewidth}{!}{%
\begin{tabular}{c l c H H c c c c}
\toprule
& & SNIS & \ourmethod{}-IS & \ourmethod{}-MIS & FALCON & FALCON-A & F2D2 & SCALLOP \\
\midrule
\multirow{2}{*}{\textbf{ALA-2}} & Inf. (m) & $2.0\err{0.0}$ & $2.3\err{0.2}$ & $2.1\err{0.2}$ & $24.2\err{0.0}$ & $24.1\err{0.1}$ & $0.8\err{0.0}$ & $0.8\err{0.0}$ \\
& Train (h) & $3.4$ & $4.3$ & $4.3$ & $4.8$ & $5.0$ & $6.3$ & $5.6$ \\
\midrule
\multirow{2}{*}{\textbf{ALA-3}} & Inf. (m) & $2.1\err{0.0}$ & $2.7\err{0.0}$ & $2.5\err{0.2}$ & $\gtrsim 129$ & $\gtrsim 129$ & $2.0\err{0.0}$ & $2.0\err{0.0}$ \\
& Train (h) & $4.7$ & $6.0$ & $6.0$ & $7.1$ & $7.2$ & $6.5$ & $5.8$ \\
\midrule
\multirow{2}{*}{\textbf{ALA-4}} & Inf. (m) & $2.8\err{0.1}$ & $3.1\err{0.2}$ & $2.7\err{0.2}$ & $\gtrsim 202$ & $\gtrsim 201$ & $2.4\err{0.0}$ & $2.4\err{0.0}$ \\
& Train (h) & $5.8$ & $5.3$ & $5.3$ & $7.1$ & $7.1$ & $5.6$ & $5.1$ \\
\midrule
\multirow{2}{*}{\textbf{ALA-6}} & Inf. (m) & $6.2\err{0.0}$ & $7.5\err{0.2}$ & $6.8\err{0.2}$ & $\gtrsim 991$ & $\gtrsim 992$ & $6.7\err{0.0}$ & $6.7\err{0.1}$ \\
& Train (h) & $11.7$ & $12.8$ & $10.2$ & $15.1$ & $15.1$ & $10.3$ & $9.8$ \\
\midrule
\multirow{2}{*}{\textbf{Chig.}} & Inf. (m) & $40.7\err{0.0}$ & $47.2\err{0.3}$ & $37.7\err{0.0}$ & $\gtrsim 2820$ & $\gtrsim 2818$ & $73.0\err{0.1}$ & $73.1\err{0.0}$ \\
& Train (h) & $13.3$ & $21.7$ & $21.7$ & $24.4$ & $24.0$ & $22.2$ & $23.4$ \\
\bottomrule
\end{tabular}%
}
\end{wraptable}

\paragraph{The accuracy-efficiency trade-off of flow-map methods.} 
\Cref{tab:cost} shows that \ourmethod{}-IS/MIS have comparable training and inference time to SNIS, although they require more training time on chignolin.  FALCON, F2D2, and SCALLOP use different approximations of proposal densities,  resulting in distinct accuracy-efficiency tradeoffs.  
F2D2 and SCALLOP predict density changes with a learned divergence head, with inference times comparable to \ourmethod{} and SNIS on the ALA systems while roughly 1.6--2x slower on chignolin. %
FALCONs instead compute full flow-map Jacobians; without guaranteed invertibility, their determinants need not yield exact proposal densities.%
Their inference is roughly $11\times$ slower than SNIS on ALA-2 and $48$--$129\times$ slower on ALA-3 through ALA-6. On chignolin, a full $10^6$-proposal evaluation was projected at $\sim 500$ GPU-hours per seed (\Cref{subsec:hparam_inference}), forcing a reduced pool of $10^5$—\, still $\sim 60\times$ slower than SNIS or \ourmethod{}-IS. 
We additionally show that under SNIS, distilled flow maps do not consistently outperform their velocity teacher (\Cref{tab:flow_map_vs_velocity}).%
By contrast, \ourmethod{} computes its conditional proposal density exactly from the flow's tractable Jacobian.

\paragraph{\ourmethod{}: IS versus MIS.} \ourmethod{}-IS and \ourmethod{}-MIS perform similarly,
with MIS using a smaller candidate budget. 
MIS reduces inference time most clearly on chignolin, with smaller differences on ALA systems.
At matched candidate budgets, IS and MIS achieve comparable inner ESS, although both exhibit substantial weight concentration at low noise, especially on chignolin (\Cref{fig:inner_ess_seqis_seqmis}).%

Overall, \ourmethod{} delivers strong performance at comparable training cost. Its inference time is broadly comparable to SNIS, F2D2, and SCALLOP and much shorter than FALCON and FALCON-A. Its performance advantage is most pronounced on the larger systems.

\section{Conclusion}

We presented \ourmethod{}, a Boltzmann Generator that parameterizes a stochastic flow map with a conditional normalizing flow, correcting each denoising transition via tractable SNIS. This yields consistent samples with theoretical error control and strong empirical performance across peptide systems, with no approximate likelihood. Currently, \ourmethod{} is trained per system; a natural next step is transferability across peptides as in \citet{tan2025scalable}, learning a single conditional flow that generalizes across sequences and lengths. 
One other promising directions would be to do training with the reverse KL in \Cref{subsec:method:training} to further refining the model.

\subsection*{Acknowledgement}

RKOY acknowledges the UK Engineering and Physical Sciences Research Council (EPSRC) grant EP/L016516/1 for the University of Cambridge Centre for Doctoral Training, the Cambridge Centre for Analysis. The computations reported in this paper were performed using resources made available by the Flatiron Institute, a division of the Simons Foundation.

\newpage

\subsection*{AI use statement}

We used generative AI tools to improve the writing quality of this manuscript
(grammar, phrasing, and clarity), but not to generate full sections of the
paper from scratch. All technical content, claims, and structure originate
from the authors. We also used generative AI tools to assist in writing code
for our experiments. All AI-assisted text and code were reviewed by the
authors, and we take full responsibility for the final content of this work.

\subsection*{Reproducibility statement}

We provide substantial implementation and experimental details in the
appendix, including the conditional flow architectures (\Cref{app:molecular-tarflow}), the
molecular symmetry correction (\Cref{app:com-adjustment-conditional}), the multiple importance sampling
construction (\Cref{app:mis}, the MCMC rejuvenation scheme (\Cref{app:maexl}), the
proofs of our theoretical guarantees (\Cref{app:sec:proofs}), and the full experimental
setup and evaluation protocol (\Cref{app:datasets,sec:hyperparameters}).

\bibliographystyle{iclr2027_conference}
\bibliography{refs}

@inproceedings{lipman2023flow,
	title        = {Flow Matching for Generative Modeling},
	author       = {Y. Lipman and R. T. Q. Chen and H. Ben-Hamu and M. Nickel and M. Le},
	year         = 2023,
	booktitle    = {International Conference on Learning Representations}
}

@article{albergo2025interpolant,
	title        = {Stochastic Interpolants: A Unifying Framework for Flows and Diffusions},
	author       = {M. Albergo and N. M. Boffi and E. Vanden-Eijnden},
	year         = 2025,
	journal      = {Journal of Machine Learning Research},
	volume       = 26,
	number       = 209,
	pages        = {1--80}
}

@article{Hutchinson01011990,
	title        = {A stochastic estimator of the trace of the influence matrix for laplacian smoothing splines},
	author       = {M.F. Hutchinson},
	year         = 1990,
	journal      = {Communications in Statistics - Simulation and Computation},
	publisher    = {Taylor \& Francis},
	volume       = 19,
	number       = 2,
	pages        = {433--450},
	doi          = {10.1080/03610919008812866},
	url          = {https://doi.org/10.1080/03610919008812866},
	eprint       = {https://doi.org/10.1080/03610919008812866}
}

@inproceedings{tan2025scalable,
	title        = {Scalable Equilibrium Sampling with Sequential Boltzmann Generators},
	author       = {Tan, Charlie B. and Bose, Joey and Lin, Chen and Klein, Leon and Bronstein, Michael M. and Tong, Alexander},
	year         = 2025,
	month        = {13--19 Jul},
	booktitle    = {Proceedings of the 42nd International Conference on Machine Learning},
	publisher    = {PMLR},
	series       = {Proceedings of Machine Learning Research},
	volume       = 267,
	pages        = {58467--58498},
	url          = {https://proceedings.mlr.press/v267/tan25a.html},
	editor       = {Singh, Aarti and Fazel, Maryam and Hsu, Daniel and Lacoste-Julien, Simon and Berkenkamp, Felix and Maharaj, Tegan and Wagstaff, Kiri and Zhu, Jerry}
}

@inproceedings{zhai2025normalizingflowscapablegenerative,
	title        = {Normalizing Flows are Capable Generative Models},
	author       = {Zhai, Shuangfei and Zhang, Ruixiang and Nakkiran, Preetum and Berthelot, David and Gu, Jiatao and Zheng, Huangjie and Chen, Tianrong and Bautista, Miguel \'{A}ngel and Jaitly, Navdeep and Susskind, Joshua M.},
	year         = 2025,
	month        = {13--19 Jul},
	booktitle    = {Proceedings of the 42nd International Conference on Machine Learning},
	publisher    = {PMLR},
	series       = {Proceedings of Machine Learning Research},
	volume       = 267,
	pages        = {74348--74369},
	url          = {https://proceedings.mlr.press/v267/zhai25d.html},
	editor       = {Singh, Aarti and Fazel, Maryam and Hsu, Daniel and Lacoste-Julien, Simon and Berkenkamp, Felix and Maharaj, Tegan and Wagstaff, Kiri and Zhu, Jerry}
}

@inproceedings{draxler2024freeformflowsmakearchitecture,
	title        = {Free-form Flows: Make Any Architecture a Normalizing Flow},
	author       = {Draxler, Felix and Sorrenson, Peter and Zimmermann, Lea and Rousselot, Armand and K\"{o}the, Ullrich},
	year         = 2024,
	month        = {02--04 May},
	booktitle    = {Proceedings of The 27th International Conference on Artificial Intelligence and Statistics},
	publisher    = {PMLR},
	series       = {Proceedings of Machine Learning Research},
	volume       = 238,
	pages        = {2197--2205},
	url          = {https://proceedings.mlr.press/v238/draxler24a.html},
	editor       = {Dasgupta, Sanjoy and Mandt, Stephan and Li, Yingzhen}
}

@inproceedings{rezende2015variational,
	title        = {Variational inference with normalizing flows},
	author       = {Rezende, Danilo and Mohamed, Shakir},
	year         = 2015,
	booktitle    = {International conference on machine learning},
	pages        = {1530--1538},
	organization = {PMLR}
}

@inproceedings{potaptchik2026metaflowmapsenable,
	title        = {Meta Flow Maps enable scalable reward alignment},
	author       = {Peter Potaptchik and Adhi Saravanan and Abbas Mammadov and Alvaro Prat and Michael Samuel Albergo and Yee Whye Teh},
	year         = 2026,
	booktitle    = {Forty-third International Conference on Machine Learning},
	url          = {https://openreview.net/forum?id=K5gV8Yptne}
}

@inproceedings{holderrieth2026diamondmapsefficientreward,
	title        = {Diamond Maps: Efficient Reward Alignment via Stochastic Flow Maps},
	author       = {Peter Holderrieth and Douglas Chen and Luca Eyring and Ishin Shah and Giri Anantharaman and Yutong He and Zeynep Akata and Tommi Jaakkola and Nicholas Matthew Boffi and Max Simchowitz},
	year         = 2026,
	booktitle    = {Forty-third International Conference on Machine Learning},
	url          = {https://openreview.net/forum?id=tAtpSwjsCB}
}

@inproceedings{debortoli2025distributionaldiffusionmodelsscoring,
	title        = {Distributional Diffusion Models with Scoring Rules},
	author       = {De Bortoli, Valentin and Galashov, Alexandre and Guntupalli, J Swaroop and Zhou, Guangyao and Murphy, Kevin Patrick and Gretton, Arthur and Doucet, Arnaud},
	year         = 2025,
	month        = {13--19 Jul},
	booktitle    = {Proceedings of the 42nd International Conference on Machine Learning},
	publisher    = {PMLR},
	series       = {Proceedings of Machine Learning Research},
	volume       = 267,
	pages        = {12632--12676},
	url          = {https://proceedings.mlr.press/v267/de-bortoli25b.html},
	editor       = {Singh, Aarti and Fazel, Maryam and Hsu, Daniel and Lacoste-Julien, Simon and Berkenkamp, Felix and Maharaj, Tegan and Wagstaff, Kiri and Zhu, Jerry}
}

@inproceedings{xiao2022tacklinggenerativelearningtrilemma,
	title        = {Tackling the Generative Learning Trilemma with Denoising Diffusion {GAN}s},
	author       = {Zhisheng Xiao and Karsten Kreis and Arash Vahdat},
	year         = 2022,
	booktitle    = {International Conference on Learning Representations},
	url          = {https://openreview.net/forum?id=JprM0p-q0Co}
}

@article{noe2019boltzmann,
	title        = {Boltzmann generators: Sampling equilibrium states of many-body systems with deep learning},
	author       = {No{\'e}, Frank and Olsson, Simon and K{\"o}hler, Jonas and Wu, Hao},
	year         = 2019,
	journal      = {Science},
	publisher    = {American Association for the Advancement of Science},
	volume       = 365,
	number       = 6457,
	pages        = {eaaw1147}
}

@inproceedings{dinh2017density,
	title        = {Density estimation using Real {NVP}},
	author       = {Laurent Dinh and Jascha Sohl-Dickstein and Samy Bengio},
	year         = 2017,
	booktitle    = {International Conference on Learning Representations},
	url          = {https://openreview.net/forum?id=HkpbnH9lx}
}

@misc{debortoli2024targetscorematching,
	title        = {Target Score Matching},
	author       = {Valentin De Bortoli and Michael Hutchinson and Peter Wirnsberger and Arnaud Doucet},
	year         = 2024,
	url          = {https://arxiv.org/abs/2402.08667},
	eprint       = {2402.08667},
	archiveprefix = {arXiv},
	primaryclass = {cs.LG}
}

@misc{kahouli2025controlvariatescorematching,
	title        = {Control Variate Score Matching for Diffusion Models},
	author       = {Khaled Kahouli and Romuald Elie and Klaus-Robert Müller and Quentin Berthet and Oliver T. Unke and Arnaud Doucet},
	year         = 2025,
	url          = {https://arxiv.org/abs/2512.20003},
	eprint       = {2512.20003},
	archiveprefix = {arXiv},
	primaryclass = {cs.LG}
}

@misc{young2026diffusionpathsamplerssequential,
	title        = {Diffusion Path Samplers via Sequential Monte Carlo},
	author       = {James Matthew Young and Paula Cordero-Encinar and Sebastian Reich and Andrew Duncan and O. Deniz Akyildiz},
	year         = 2026,
	url          = {https://arxiv.org/abs/2601.21951},
	eprint       = {2601.21951},
	archiveprefix = {arXiv},
	primaryclass = {stat.ML}
}

@inproceedings{blessing2026bridgematchingsamplerscalable,
	title        = {Bridge Matching Sampler: Scalable Sampling via Generalized Fixed-Point Diffusion Matching},
	author       = {Denis Blessing and Lorenz Richter and Julius Berner and Egor Malitskiy and Gerhard Neumann},
	year         = 2026,
	booktitle    = {Forty-third International Conference on Machine Learning},
	url          = {https://openreview.net/forum?id=gUbNUNuJbO}
}

@article{roberts1998optimal,
	title        = {Optimal scaling of discrete approximations to Langevin diffusions},
	author       = {Roberts, Gareth O. and Rosenthal, Jeffrey S.},
	year         = 1998,
	journal      = {Journal of the Royal Statistical Society: Series B (Statistical Methodology)},
	volume       = 60,
	number       = 1,
	pages        = {255--268},
	doi          = {https://doi.org/10.1111/1467-9868.00123},
	url          = {https://rss.onlinelibrary.wiley.com/doi/abs/10.1111/1467-9868.00123},
	eprint       = {https://rss.onlinelibrary.wiley.com/doi/pdf/10.1111/1467-9868.00123}
}

@article{eastman2017openmm,
	title        = {OpenMM 7: Rapid development of high performance algorithms for molecular dynamics},
	author       = {Eastman, Peter and Swails, Jason and Chodera, John D and McGibbon, Robert T and Zhao, Yutong and Beauchamp, Kyle A and Wang, Lee-Ping and Simmonett, Andrew C and Harrigan, Matthew P and Stern, Chaya D and others},
	year         = 2017,
	journal      = {PLoS computational biology},
	volume       = 13,
	number       = 7,
	pages        = {e1005659}
}

@inproceedings{klein2023timewarp,
	title        = {Timewarp: Transferable Acceleration of Molecular Dynamics by Learning Time-Coarsened Dynamics},
	author       = {Leon Klein and Andrew Y. K. Foong and Tor Erlend Fjelde and Bruno Kacper Mlodozeniec and Marc Brockschmidt and Sebastian Nowozin and Frank Noe and Ryota Tomioka},
	year         = 2023,
	booktitle    = {Thirty-seventh Conference on Neural Information Processing Systems},
	url          = {https://openreview.net/forum?id=EjMLpTgvKH}
}

@article{dibak2021temperature,
	title        = {Temperature steerable flows and {Boltzmann} generators},
	author       = {Dibak, Manuel and Klein, Leon and Kr\"amer, Andreas and No\'e, Frank},
	year         = 2022,
	month        = {Oct},
	journal      = {Phys. Rev. Res.},
	volume       = 4,
	pages        = {L042005},
	doi          = {10.1103/PhysRevResearch.4.L042005},
	issue        = 4,
	numpages     = 6
}

@article{honda200410,
	title        = {10 residue folded peptide designed by segment statistics},
	author       = {Honda, Shinya and Yamasaki, Kazuhiko and Sawada, Yoshito and Morii, Hisayuki},
	year         = 2004,
	journal      = {Structure},
	publisher    = {Elsevier},
	volume       = 12,
	number       = 8,
	pages        = {1507--1518}
}

@article{flamary2021pot,
	title        = {POT: Python Optimal Transport},
	author       = {R{\'e}mi Flamary and Nicolas Courty and Alexandre Gramfort and Mokhtar Z. Alaya and Aur{\'e}lie Boisbunon and Stanislas Chambon and Laetitia Chapel and Adrien Corenflos and Kilian Fatras and Nemo Fournier and L{\'e}o Gautheron and Nathalie T.H. Gayraud and Hicham Janati and Alain Rakotomamonjy and Ievgen Redko and Antoine Rolet and Antony Schutz and Vivien Seguy and Danica J. Sutherland and Romain Tavenard and Alexander Tong and Titouan Vayer},
	year         = 2021,
	journal      = {Journal of Machine Learning Research},
	volume       = 22,
	number       = 78,
	pages        = {1--8}
}

@inproceedings{kohler2020equivariant,
	title        = {Equivariant Flows: Exact Likelihood Generative Learning for Symmetric Densities},
	author       = {K{\"o}hler, Jonas and Klein, Leon and Noe, Frank},
	year         = 2020,
	month        = {13--18 Jul},
	booktitle    = {Proceedings of the 37th International Conference on Machine Learning},
	publisher    = {PMLR},
	series       = {Proceedings of Machine Learning Research},
	volume       = 119,
	pages        = {5361--5370},
	url          = {https://proceedings.mlr.press/v119/kohler20a.html},
	editor       = {III, Hal Daumé and Singh, Aarti}
}

@inproceedings{klein2023equivariant,
	title        = {Equivariant flow matching},
	author       = {Klein, Leon and Kr\"{a}mer, Andreas and Noe, Frank},
	year         = 2023,
	booktitle    = {Advances in Neural Information Processing Systems},
	publisher    = {Curran Associates, Inc.},
	volume       = 36,
	pages        = {59886--59910},
	url          = {https://proceedings.neurips.cc/paper_files/paper/2023/file/bc827452450356f9f558f4e4568d553b-Paper-Conference.pdf},
	editor       = {A. Oh and T. Naumann and A. Globerson and K. Saenko and M. Hardt and S. Levine}
}

@book{liu2001monte,
	title        = {Monte Carlo Strategies in Scientific Computing},
	author       = {Liu, Jun S.},
	year         = 2001,
	publisher    = {Springer},
	address      = {New York, NY}
}

@book{frenkel2001understanding,
	title        = {Understanding Molecular Simulation: From Algorithms to Applications},
	author       = {Frenkel, Daan and Smit, Berend},
	year         = 2001,
	publisher    = {Academic Press},
	address      = {San Diego, CA},
	edition      = 2
}

@article{noe2009constructing,
	title        = {Constructing the equilibrium ensemble of folding pathways from short off-equilibrium simulations},
	author       = {No{\'e}, Frank and Sch{\"u}tte, Christof and Vanden-Eijnden, Eric and Reich, Lothar and Weikl, Thomas R.},
	year         = 2009,
	journal      = {Proceedings of the National Academy of Sciences},
	volume       = 106,
	number       = 45,
	pages        = {19011--19016}
}

@article{lindorff2011fast,
	title        = {How fast-folding proteins fold},
	author       = {Lindorff-Larsen, Kresten and Piana, Stefano and Dror, Ron O. and Shaw, David E.},
	year         = 2011,
	journal      = {Science},
	volume       = 334,
	number       = 6055,
	pages        = {517--520}
}

@article{buch2011high,
	title        = {High-Throughput All-Atom Molecular Dynamics Simulations Using Distributed Computing},
	author       = {Buch, Ivet and Giorgino, Toni and De Fabritiis, Gianni},
	year         = 2011,
	journal      = {Journal of Chemical Information and Modeling},
	volume       = 51,
	number       = 8,
	pages        = {1824--1830}
}

@book{leimkuhler2015molecular,
	title        = {Molecular Dynamics: With Deterministic and Stochastic Numerical Methods},
	author       = {Leimkuhler, Benedict and Matthews, Charles},
	year         = 2015,
	publisher    = {Springer},
	address      = {Cham}
}

@article{wirnsberger2020targeted,
	title        = {Targeted free energy perturbation},
	author       = {Wirnsberger, Peter and Adler, Jonas and K{\"o}hler, Jonas and No{\'e}, Frank},
	year         = 2020,
	journal      = {Proceedings of the National Academy of Sciences},
	volume       = 117,
	number       = 18,
	pages        = {9902--9908}
}

@article{muller2019neural,
	title        = {Neural Importance Sampling},
	author       = {M{\"u}ller, Thomas and McWilliams, Brian and Rousselle, Fabrice and Gross, Markus and Nov{\'a}k, Jan},
	year         = 2019,
	journal      = {ACM Transactions on Graphics},
	volume       = 38,
	number       = 5,
	pages        = {145:1--145:19}
}

@article{papamakarios2021normalizing,
	title        = {Normalizing Flows for Probabilistic Modeling and Inference},
	author       = {Papamakarios, George and Nalisnick, Eric and Rezende, Danilo Jimenez and Mohamed, Shakir and Lakshminarayanan, Balaji},
	year         = 2021,
	journal      = {Journal of Machine Learning Research},
	volume       = 22,
	number       = 57,
	pages        = {1--64}
}

@inproceedings{rehman2025falconfewstepaccuratelikelihoods,
	title        = {{FALCON}: Few-step Accurate Likelihoods for Continuous Flows},
	author       = {Danyal Rehman and Tara Akhound-Sadegh and Artem Gazizov and Yoshua Bengio and Alexander Tong},
	year         = 2026,
	booktitle    = {The Fourteenth International Conference on Learning Representations},
	url          = {https://openreview.net/forum?id=FbssShlI4N}
}

@inproceedings{rehman2026autoregressiveboltzmanngenerators,
	title        = {Autoregressive Boltzmann Generators},
	author       = {Danyal Rehman and Charlie B. Tan and Yoshua Bengio and Joey Bose and Alexander Tong},
	year         = 2026,
	booktitle    = {Forty-third International Conference on Machine Learning},
	url          = {https://openreview.net/forum?id=75AYDsndHP}
}

@article{grenioux2026diffusionbasedannealedboltzmanngenerators,
	title        = {Diffusion-based Annealed Boltzmann Generators : benefits, pitfalls and hopes},
	author       = {Louis Grenioux and Maxence Noble},
	year         = 2026,
	journal      = {Transactions on Machine Learning Research},
	issn         = {2835-8856},
	url          = {https://openreview.net/forum?id=la4FDaeIbw},
	note         = {}
}

@inproceedings{chen2018neural,
	title        = {Neural Ordinary Differential Equations},
	author       = {Chen, Ricky T. Q. and Rubanova, Yulia and Bettencourt, Jesse and Duvenaud, David},
	year         = 2018,
	booktitle    = {Advances in Neural Information Processing Systems},
	volume       = 31
}

@article{Ho2020denoising,
	title        = {Denoising diffusion probabilistic models},
	author       = {Ho, Jonathan and Jain, Ajay and Abbeel, Pieter},
	year         = 2020,
	journal      = {Advances in neural information processing systems},
	volume       = 33,
	pages        = {6840--6851}
}

@inproceedings{Song2021score,
	title        = {Score-Based Generative Modeling through Stochastic Differential Equations},
	author       = {Song, Yang and Sohl-Dickstein, Jascha and Kingma, Diederik P and Kumar, Abhishek and Ermon, Stefano and Poole, Ben},
	year         = 2021,
	booktitle    = {The Ninth International Conference on Learning Representations},
	url          = {https://openreview.net/forum?id=PxTIG12RRHS}
}

@inproceedings{karras2022elucidating,
	title        = {Elucidating the Design Space of Diffusion-Based Generative Models},
	author       = {Tero Karras and Miika Aittala and Timo Aila and Samuli Laine},
	year         = 2022,
	booktitle    = {Advances in Neural Information Processing Systems},
	url          = {https://openreview.net/forum?id=k7FuTOWMOc7},
	editor       = {Alice H. Oh and Alekh Agarwal and Danielle Belgrave and Kyunghyun Cho}
}

@article{Zhang2024efficient,
	title        = {Efficient and unbiased sampling of boltzmann distributions via consistency models},
	author       = {Zhang, Fengzhe and He, Jiajun and Midgley, Laurence I and Antor{\'a}n, Javier and Hern{\'a}ndez-Lobato, Jos{\'e} Miguel},
	year         = 2024,
	journal      = {arXiv preprint arXiv:2409.07323}
}

@article{zhang2025efficient,
	title        = {Efficient and Unbiased Sampling from Boltzmann Distributions via Variance-Tuned Diffusion Models},
	author       = {Fengzhe Zhang and Laurence Illing Midgley and Jos{\'e} Miguel Hern{\'a}ndez-Lobato},
	year         = 2025,
	journal      = {Transactions on Machine Learning Research},
	issn         = {2835-8856},
	url          = {https://openreview.net/forum?id=Jq2dcMCS5R},
	note         = {J2C Certification}
}

@article{Agapiou2017Importance,
	title        = {Importance Sampling: Intrinsic Dimension and Computational Cost},
	author       = {S. Agapiou and O. Papaspiliopoulos and D. Sanz-Alonso and A. M. Stuart},
	year         = 2017,
	journal      = {Statistical Science},
	publisher    = {Institute of Mathematical Statistics},
	volume       = 32,
	number       = 3,
	pages        = {405--431},
	issn         = {08834237, 21688745},
	url          = {http://www.jstor.org/stable/26408299},
	urldate      = {2025-06-19}
}

@article{Durmus2015quantitative,
	title        = {Quantitative bounds of convergence for geometrically ergodic Markov chain in the Wasserstein distance with application to the Metropolis Adjusted Langevin Algorithm},
	author       = {Durmus, Alain and Moulines, {\'E}ric},
	year         = 2015,
	month        = {Jan},
	day          = {01},
	journal      = {Statistics and Computing},
	volume       = 25,
	number       = 1,
	pages        = {5--19},
	doi          = {10.1007/s11222-014-9511-z},
	issn         = {1573-1375},
	url          = {https://doi.org/10.1007/s11222-014-9511-z}
}

@inproceedings{ouyang2026fewstep,
	title        = {Few-Step Boltzmann Generators via Scalable Likelihood Flow Maps},
	author       = {RuiKang OuYang and Hanlin Yu and Xinyue Ai and Yutong He and Nicholas Matthew Boffi and Pradeep Ravikumar and Jos{\'e} Miguel Hern{\'a}ndez-Lobato and Max Simchowitz and Benjamin Kurt Miller and Omar Chehab},
	year         = 2026,
	booktitle    = {ICML 2026 Workshop on Structured Probabilistic Inference {\&} Generative Modeling},
	url          = {https://openreview.net/forum?id=mfdTHHRIid}
}

@inproceedings{boffi2025buildconsistencymodellearning,
	title        = {How to build a consistency model: Learning flow maps via self-distillation},
	author       = {Nicholas Matthew Boffi and Michael Samuel Albergo and Eric Vanden-Eijnden},
	year         = 2025,
	booktitle    = {The Thirty-ninth Annual Conference on Neural Information Processing Systems},
	url          = {https://openreview.net/forum?id=Di5apl8HSH}
}

@inproceedings{ai2026jointdistillationfastlikelihood,
	title        = {Joint Distillation for Fast Likelihood Evaluation and Sampling in Flow-based Models},
	author       = {Xinyue Ai and Yutong He and Albert Gu and Ruslan Salakhutdinov and J Zico Kolter and Nicholas Matthew Boffi and Max Simchowitz},
	year         = 2026,
	booktitle    = {The Fourteenth International Conference on Learning Representations},
	url          = {https://openreview.net/forum?id=8uZ5UdIul2}
}

@article{woodard2009sufficient,
	title        = {Sufficient Conditions for Torpid Mixing of Parallel and Simulated Tempering},
	author       = {Woodard, Dawn B. and Schmidler, Scott C. and Huber, Mark},
	year         = 2009,
	journal      = {Electronic Journal of Probability},
	volume       = 14,
	pages        = {780--804}
}

@article{mate2023learning,
	title        = {Learning Interpolations between Boltzmann Densities},
	author       = {M{\'a}t{\'e}, B{\'a}lint and Fleuret, Fran{\c{c}}ois},
	year         = 2023,
	journal      = {Transactions on Machine Learning Research},
	issn         = {2835-8856},
	url          = {https://openreview.net/forum?id=TH6YrEcbth}
}

@article{gu2026starflow,
	title        = {STARFlow: Scaling Latent Normalizing Flows for High-resolution Image Synthesis},
	author       = {Gu, Jiatao and Chen, Tianrong and Berthelot, David and Zheng, Huangjie and Wang, Yuyang and Zhang, Ruixiang and Dinh, Laurent and Bautista, Miguel Angel and Susskind, Joshua and Zhai, Shuangfei},
	year         = 2026,
	journal      = {Advances in Neural Information Processing Systems},
	volume       = 38,
	pages        = {120986--121022}
}

@article{klein2024transferable,
	title        = {Transferable boltzmann generators},
	author       = {Klein, Leon and No{\'e}, Frank},
	year         = 2024,
	journal      = {Advances in Neural Information Processing Systems},
	volume       = 37,
	pages        = {45281--45314}
}

@inproceedings{grenioux2024stochastic,
	title        = {Stochastic Localization via Iterative Posterior Sampling},
	author       = {Grenioux, Louis and Noble, Maxence and Gabri\'{e}, Marylou and Oliviero Durmus, Alain},
	year         = 2024,
	booktitle    = {Proceedings of the 41st International Conference on Machine Learning},
	publisher    = {PMLR},
	series       = {Proceedings of Machine Learning Research},
	volume       = 235,
	pages        = {16337--16376},
	url          = {https://proceedings.mlr.press/v235/grenioux24a.html},
	editor       = {Salakhutdinov, Ruslan and Kolter, Zico and Heller, Katherine and Weller, Adrian and Oliver, Nuria and Scarlett, Jonathan and Berkenkamp, Felix}
}

@inproceedings{huang2018neuralautoregressiveflows,
	title        = {Neural Autoregressive Flows},
	author       = {Huang, Chin-Wei and Krueger, David and Lacoste, Alexandre and Courville, Aaron},
	year         = 2018,
	month        = {10--15 Jul},
	booktitle    = {Proceedings of the 35th International Conference on Machine Learning},
	publisher    = {PMLR},
	series       = {Proceedings of Machine Learning Research},
	volume       = 80,
	pages        = {2078--2087},
	url          = {https://proceedings.mlr.press/v80/huang18d.html},
	editor       = {Dy, Jennifer and Krause, Andreas}
}

@inproceedings{veach1995optimally,
	title        = {Optimally combining sampling techniques for Monte Carlo rendering},
	author       = {Veach, Eric and Guibas, Leonidas J.},
	year         = 1995,
	booktitle    = {Proceedings of the 22nd Annual Conference on Computer Graphics and Interactive Techniques},
	publisher    = {Association for Computing Machinery},
	address      = {New York, NY, USA},
	series       = {SIGGRAPH '95},
	pages        = {419–428},
	doi          = {10.1145/218380.218498},
	isbn         = {0897917014},
	url          = {https://doi.org/10.1145/218380.218498},
	numpages     = 10
}

@article{owen2000safe,
	title        = {Safe and Effective Importance Sampling},
	author       = {Art Owen and Yi Zhou},
	year         = 2000,
	journal      = {Journal of the American Statistical Association},
	publisher    = {[American Statistical Association, Taylor & Francis, Ltd.]},
	volume       = 95,
	number       = 449,
	pages        = {135--143},
	issn         = {01621459, 1537274X},
	url          = {http://www.jstor.org/stable/2669533},
	urldate      = {2026-09-09}
}

@misc{he2014optimalmixtureweightsmultiple,
	title        = {Optimal mixture weights in multiple importance sampling},
	author       = {Hera Y. He and Art B. Owen},
	year         = 2014,
	url          = {https://arxiv.org/abs/1411.3954},
	eprint       = {1411.3954},
	archiveprefix = {arXiv},
	primaryclass = {stat.CO}
}

@article{bugallo2017adaptive,
	title        = {Adaptive Importance Sampling: The past, the present, and the future},
	author       = {Bugallo, Monica F. and Elvira, Victor and Martino, Luca and Luengo, David and Miguez, Joaquin and Djuric, Petar M.},
	year         = 2017,
	journal      = {IEEE Signal Processing Magazine},
	volume       = 34,
	number       = 4,
	pages        = {60--79},
	doi          = {10.1109/MSP.2017.2699226}
}

@article{elvira2019generalized,
	title        = {{Generalized Multiple Importance Sampling}},
	author       = {V{\'i}ctor Elvira and Luca Martino and David Luengo and M{\'o}nica F. Bugallo},
	year         = 2019,
	journal      = {Statistical Science},
	publisher    = {Institute of Mathematical Statistics},
	volume       = 34,
	number       = 1,
	pages        = {129 -- 155},
	doi          = {10.1214/18-STS668},
	url          = {https://doi.org/10.1214/18-STS668}
}

@inproceedings{peebles2023scalable,
	title        = {Scalable Diffusion Models with Transformers},
	author       = {Peebles, William and Xie, Saining},
	year         = 2023,
	month        = {October},
	booktitle    = {Proceedings of the IEEE/CVF International Conference on Computer Vision (ICCV)},
	pages        = {4195--4205}
}

@inproceedings{ying2021graphormer,
	title        = {Do Transformers Really Perform Badly for Graph Representation?},
	author       = {Ying, Chengxuan and Cai, Tianle and Luo, Shengjie and Zheng, Shuxin and Ke, Guolin and He, Di and Shen, Yanming and Liu, Tie-Yan},
	year         = 2021,
	booktitle    = {Advances in Neural Information Processing Systems},
	publisher    = {Curran Associates, Inc.},
	volume       = 34,
	pages        = {28877--28888},
	url          = {https://proceedings.neurips.cc/paper_files/paper/2021/file/f1c1592588411002af340cbaedd6fc33-Paper.pdf},
	editor       = {M. Ranzato and A. Beygelzimer and Y. Dauphin and P.S. Liang and J. Wortman Vaughan}
}

@inproceedings{geng2025mean,
	title        = {Mean Flows for One-step Generative Modeling},
	author       = {Zhengyang Geng and Mingyang Deng and Xingjian Bai and J Zico Kolter and Kaiming He},
	year         = 2025,
	booktitle    = {The Thirty-ninth Annual Conference on Neural Information Processing Systems},
	url          = {https://openreview.net/forum?id=uWj4s7rMnR}
}

@inproceedings{wu2020stochasticnormalizingflows,
	title        = {Stochastic Normalizing Flows},
	author       = {Wu, Hao and K\"{o}hler, Jonas and Noe, Frank},
	year         = 2020,
	booktitle    = {Advances in Neural Information Processing Systems},
	publisher    = {Curran Associates, Inc.},
	volume       = 33,
	pages        = {5933--5944},
	url          = {https://proceedings.neurips.cc/paper_files/paper/2020/file/41d80bfc327ef980528426fc810a6d7a-Paper.pdf},
	editor       = {H. Larochelle and M. Ranzato and R. Hadsell and M.F. Balcan and H. Lin}
}

\clearpage
\startcontents[appendix]
\vbox{%
    \hsize\textwidth
    \linewidth\hsize
    \vskip 0.1in
    \hrule height 4pt
  \vskip 0.25in
  \vskip -\parskip%
    \centering
    {\LARGE\bf {Bridging Stochastic Flow Maps and Boltzmann Generators with Normalizing Flows\\ Appendix} \par}
     \vskip 0.29in
  \vskip -\parskip
  \hrule height 1pt
  \vskip 0.09in%
  }
\printcontents[appendix]{}{1}{\setcounter{tocdepth}{3}}

\clearpage %
\appendix

\crefformat{section}{Appendix~#2#1#3}
\Crefformat{section}{Appendix~#2#1#3}

\section{Details on the method}\label{app:sec:method}
\subsection{Theoretical details and proofs}\label{app:sec:proofs}

\paragraph{Notation.}
For a probability measure $\pi$ and a $\pi$-measurable function $f$, we write $\pi(f) = \mathbb{E}_\pi[f(X)] = \int f(x)\,\pi(\rmd x)$. For any bounded function $f$ and any $B \subset \R^d$, we set
\begin{align*}
    [f]_B = \sup_{x,x' \in B} \abs{f(x') - f(x)}.
\end{align*}
We say that a probability distribution $\pi$ on $\R^d$ has \emph{sub-Gaussian tails with center $m$ and rate $\kappa$} if for all $u > 0$,
\begin{align}\label{eq:subgaussian}
    \pi\left(\norm{Z - m} > \kappa(\sqrt{d} + u)\right) \leq \exp(-u^2).
\end{align}

\paragraph{Bridges and the operator $G_{s|t,x_t}$.}

\begin{figure}[t]
    \centering
    \includegraphics[width=\linewidth]{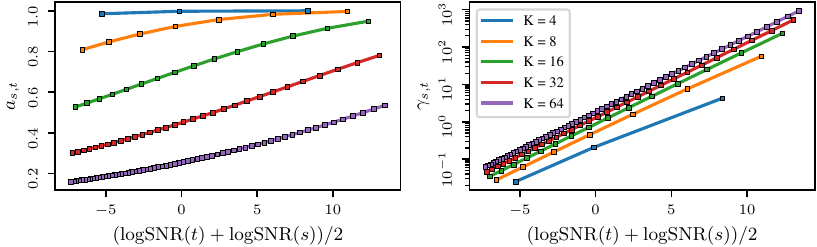}
    \caption{{Values of $a_{s,t}$ and $\gamma_{s,t}$ along the time grid of \citet{karras2022elucidating}.}}
    \label{fig:a_s_t}
\end{figure}

Throughout, $p_{s|t,0}(x_s | x_t,x_0)$ denotes the Gaussian bridge density. We write $\alpha_{s,t}=a_{s,t}$ for the clean-state coefficient used in the main text and use the standard parametrization
\begin{align*}
    \mathcal{L}\left(X_s \mid X_t = x_t, X_0 = x_0\right) = \mathcal{N}(\alpha_{s,t} x_0 + b_{s,t} x_t,\,\Sigma_{s,t}), \qquad \Sigma_{s,t} = \sigma_{s,t}^2\,\Idd,
\end{align*}
and define the \emph{bridge sensitivity} (see \Cref{fig:a_s_t})
\begin{align}\label{eq:bridge_sensitivity}
    \gamma_{t,s} = \frac{\alpha_s \sigma_{t|s}}{\sigma_t \sigma_s} = \frac{\alpha_{s,t}}{\sigma_{s,t}}.
\end{align}
Writing $A_t = \alpha_t^2/\sigma_t^2$ for the signal-to-noise ratio at time $t$ and using $\sigma_{t|s}^2 = \sigma_t^2 - (\alpha_t/\alpha_s)^2\sigma_s^2$, the bridge sensitivity is the square root of an SNR increment,
\begin{align}\label{eq:gamma_snr}
    \gamma_{t,s}^2 = \frac{\alpha_s^2}{\sigma_s^2} - \frac{\alpha_t^2}{\sigma_t^2} = A_s - A_t.
\end{align}
Since $t \mapsto A_t$ is nonincreasing, $\gamma_{t,s} \leq \sqrt{A_s} = \alpha_s/\sigma_s$, which is finite for every $s > 0$ and diverges as $s \to 0$, the bridge degenerating to a point mass. All schedules below therefore stop at a strictly positive time, see \Cref{rmk:endpoints}.
For any $p_s$-measurable $\phi$, we let
\begin{align*}
    (G_{s|t,x_t}\phi)(x_0) = \int \phi(x_s)\,p_{s|t,0}(x_s \mid x_t,x_0)\,\rmd x_s,
\end{align*}
so that
\begin{align*}
    p_{s|t}(\phi \mid x_t) = \int \phi(x_s)\,p_{s|t,0}(x_s \mid x_t,x_0)\,p_{0|t}(x_0 \mid x_t)\,\rmd x_s\,\rmd x_0 = p_{0|t}(G_{s|t,x_t}\phi \mid x_t).
\end{align*}

\paragraph{SNIS approximations.}
We assume access to a proposal $q_{0|t}$ that approximates $p_{0|t}$. Drawing $x^m_{0|t}\overset{\mathrm{iid}}{\sim} q_{0|t}(\cdot| x_t)$ for $m=1,\ldots,M$ and forming unnormalized importance weights
\begin{align*}
    \tilde w^m_{0|t} = \frac{p_{t|0}(x_t \mid x^m_{0|t})\,p(x^m_{0|t})}{q_{0|t}(x^m_{0|t} \mid x_t)},
     \qquad w^m_{0|t} = \frac{\tilde w^m_{0|t}}{\sum_{i=1}^M \tilde w^i_{0|t}},
\end{align*}
the SNIS approximation of $p_{0|t}$ is
\begin{align*}
    p^M_{0|t}(x_0 \mid x_t) = \sum_{m=1}^M  w^m_{0|t}\,\delta_{x^m_{0|t}}(x_0),
\end{align*}
and the induced SNIS approximation of $p_{s|t}$ is
\begin{align*}
    p^M_{s|t}(x_s \mid x_t) = \sum_{m=1}^M  w^m_{0|t}\,p_{s|t,0}(x_s \mid x_t,x^m_{0|t}).
\end{align*}
By the identity $p_{s|t}(\phi | x_t) = p_{0|t}(G_{s|t,x_t}\phi | x_t)$, we have $p^M_{s|t}(\phi | x_t) = p^M_{0|t}(G_{s|t,x_t}\phi | x_t)$.

\paragraph{Standing assumption (per time $t$).}
For a fixed $0\leq s < t \leq T$ and $x_t\in\R^d$, we say that the \emph{sub-Gaussian tail assumption} holds at $(t,x_t)$ if both $q_{0|t}(\cdot| x_t)$ and the normalized square-weight measure proportional to $\tilde w^2_{0|t}\,q_{0|t}(\cdot | x_t)$ have sub-Gaussian tails (in the sense of \eqref{eq:subgaussian}) with common center $m_t(x_t)$ and rate $\kappa_t(x_t)$. We further denote the chi-squared divergence
\begin{align*}
    \rho_t(x_t) = 1 + \chi^2\left(p_{0|t}(\cdot \mid x_t)\,\big\|\,q_{0|t}(\cdot \mid x_t)\right).
\end{align*}

\subsubsection{Controlling the per-step error}

We begin with a standard Wasserstein contraction property of the bridge.

\begin{proposition}\label{prop:wasserstein_bridge}
    Let $p \geq 1$, $x_t \in \R^d$, and assume that both $p_{0|t}(\cdot | x_t)$ and $q_{0|t}(\cdot | x_t)$ have finite $p$-th moment. Then
    \begin{align*}
        W_p\left(p_{s|t}(\cdot \mid x_t),\,q_{s|t}(\cdot \mid x_t)\right) \leq \alpha_{s,t}\,W_p\left(p_{0|t}(\cdot \mid x_t),\,q_{0|t}(\cdot \mid x_t)\right).
    \end{align*}
\end{proposition}

\begin{proof}
    Let  $\Pi$ be an optimal $W_p$-coupling of $p_{0|t}(\cdot | x_t)$ and $q_{0|t}(\cdot | x_t)$, and draw $(X_0, X_0') \sim  \Pi$. Independently, draw  $\xi \sim \mathcal{N}(0, \Sigma_{s,t})$, and set
    \begin{align*}
        X_s = \alpha_{s,t} X_0 + b_{s,t} x_t +  \xi, \qquad X_s' = \alpha_{s,t} X_0' + b_{s,t} x_t +  \xi.
    \end{align*}
    Then $(X_s, X_s')$ is a coupling of $p_{s|t}(\cdot | x_t)$ and $q_{s|t}(\cdot | x_t)$, and
    \begin{align*}
        \norm{X_s - X_s'} = \alpha_{s,t}\norm{X_0 - X_0'}.
    \end{align*}
    Taking $L^p$-norms gives the result.
\end{proof}

We next sharpen the SNIS bias and MSE bound of \citep[Theorem~2.1]{Agapiou2017Importance} by localizing the test function to a set $B$ where the proposal places most of its mass.

\begin{lemma}\label{lemma:agapiou_improved}
     Let $\mathsf{X}$ be a Polish space and let $\mu$ and $\nu$ be probability distributions  on $\mathsf{X}$ with $\mu \ll \nu$, and let $w = \rmd\mu/\rmd\nu$. Set $\rho = 1 + \chi^2(\mu \,\|\, \nu) = \nu(w^2)$, and for any  measurable $B \subset \mathsf{X}$ define
    \begin{align*}
        \eta_{B^\mathsf{c}} = \int_{B^\mathsf{c}} w(x)^2\,\nu(\rmd x).
    \end{align*}
    For $N \geq 1$, let $\mu^N$ be the SNIS approximation of $\mu$ with proposal $\nu$,
    \begin{align*}
        \mu^N = \sum_{i=1}^N \frac{w(X^i)}{\sum_{j=1}^N w(X^j)}\,\delta_{X^i}, \qquad X^i \overset{\mathrm{iid}}{\sim} \nu.
    \end{align*}
    Then, for any bounded $\mu$-measurable function $f$,
    \begin{align}
        \abs{\mathbb{E}\left[\mu^N(f) - \mu(f)\right]} &\leq \frac{6\rho}{N}\,[f]_B + 2\norm{f}_\infty \sqrt{\eta_{B^\mathsf{c}}\,\nu(B^\mathsf{c})} + 4\norm{f}_\infty\,N\,\nu(B^\mathsf{c}), \label{eq:agapiou_bias} \\
        \mathbb{E}\left[\left(\mu^N(f) - \mu(f)\right)^2\right] &\leq \frac{2\rho}{N}\,[f]_B^2 + 8\norm{f}_\infty^2\,\eta_{B^\mathsf{c}}\,\nu(B^\mathsf{c}) + 32\norm{f}_\infty^2\,N\,\nu(B^\mathsf{c}). \label{eq:agapiou_mse}
    \end{align}
\end{lemma}

\begin{proof}
    Set $c_B = (\sup_B f - \inf_B f)/2$ and define the truncation
    \begin{align*}
        f_B(x) = \begin{cases}
            f(x) & \text{if } x \in B, \\
            c_B & \text{otherwise,}
        \end{cases}
    \end{align*}
    so that $\norm{f_B - c_B}_\infty = [f]_B/2$. Applying \citep[Theorem~2.1]{Agapiou2017Importance} to $(f_B - c_B)/([f]_B/2)$ and using linearity yields
    \begin{align*}
        \abs{\mathbb{E}\left[\mu^N(f_B) - \mu(f_B)\right]} \leq \frac{6\rho}{N}\,[f]_B, \qquad \mathbb{E}\left[\left(\mu^N(f_B) - \mu(f_B)\right)^2\right] \leq \frac{2\rho}{N}\,[f]_B^2.
    \end{align*}

    Now let $h_B = f - f_B$; then $h_B = 0$ on $B$ and $\norm{h_B}_\infty \leq 2\norm{f}_\infty$. Denote the event $E_B = \bigcap_{i=1}^N\{X^i \in B\}$. On $E_B$, all atoms of $\mu^N$ lie in $B$, so $\mu^N(h_B) = 0$ and only the deterministic term $\mu(h_B)$ contributes. We bound the latter using Cauchy--Schwarz:
    \begin{align*}
        \abs{\mu(h_B)} &\leq 2\norm{f}_\infty \int_{B^\mathsf{c}}\mu(\rmd x) \\
        &\leq 2\norm{f}_\infty\sqrt{\int_{B^\mathsf{c}}w(x)^{2}\,\nu(\rmd x)}\sqrt{\int_{B^\mathsf{c}}\nu(\rmd x)} \\
        &= 2\norm{f}_\infty\sqrt{\eta_{B^\mathsf{c}}\,\nu(B^\mathsf{c})}.
    \end{align*}
    On $E_B^\mathsf{c}$, we use the crude bound $\abs{\mu^N(h_B) - \mu(h_B)} \leq 4\norm{f}_\infty$. By the union bound,
    \begin{align*}
        \mathbb{P}(E_B^\mathsf{c}) \leq N\,\nu(B^\mathsf{c}).
    \end{align*}
    Combining,
    \begin{align*}
        \mathbb{E}\left[\abs{\mu^N(h_B) - \mu(h_B)}\right] \leq 2\norm{f}_\infty\sqrt{\eta_{B^\mathsf{c}}\nu(B^\mathsf{c})} + 4\norm{f}_\infty\,N\,\nu(B^\mathsf{c}),
    \end{align*}
    and similarly $\mathbb{E}\left[\left(\mu^N(h_B) - \mu(h_B)\right)^2\right] \leq 4\norm{f}_\infty^2\,\eta_{B^\mathsf{c}}\nu(B^\mathsf{c}) + 16\norm{f}_\infty^2\,N\,\nu(B^\mathsf{c})$. The decomposition $\mu^N(f) - \mu(f) = \mu^N(f_B) - \mu(f_B) + \mu^N(h_B) - \mu(h_B)$ together with $(a+b)^2 \leq 2(a^2 + b^2)$ then yields \eqref{eq:agapiou_bias}--\eqref{eq:agapiou_mse}.
\end{proof}

We now control the oscillation of $G_{s|t,x_t}\psi$ on a ball.

\begin{lemma}\label{lemma:kernel_sensitivity}
     Let $x_t \in \R^d$ and $0 \leq s < t \leq T$. For any bounded test function $\psi$, $\norm{G_{s|t,x_t}\psi}_\infty \leq \norm{\psi}_\infty$, and for all $x,x' \in \R^d$,
    \begin{align}
        \abs{G_{s|t,x_t}\psi(x) - G_{s|t,x_t}\psi(x')} \leq 2\norm{\psi}_\infty\left[2\Phi\left(\frac{\gamma_{t,s}\norm{x-x'}}{2}\right) - 1\right], \label{eq:kernel_pointwise}
    \end{align}
    where $\Phi$ is the standard normal cumulative distribution function and $\gamma_{t,s}$ is defined in \eqref{eq:bridge_sensitivity}. Consequently, for $B_u = \mathcal{B}(m,\kappa(\sqrt{d}+u))$ with $m \in \R^d$, $\kappa > 0$ and $u > 0$,
    \begin{align}
        [G_{s|t,x_t}\psi]_{B_u} \leq 2\norm{\psi}_\infty \min\left\{1,\ \sqrt{\tfrac{2}{\pi}}\,\gamma_{t,s}\,\kappa\left(\sqrt{d}+u\right)\right\}. \label{eq:kernel_truncated}
    \end{align}
\end{lemma}

\begin{proof}
    The sup-norm bound is immediate since $G_{s|t,x_t}$ is an expectation operator. For  \eqref{eq:kernel_pointwise}, let $x,x' \in  \R^d$. Then
    \begin{align*}
        \abs{G_{s|t,x_t}\psi(x) - G_{s|t,x_t}\psi(x')} \leq \norm{\psi}_\infty \int \abs{p_{s|t,0}(x_s \mid x_t,x) - p_{s|t,0}(x_s \mid x_t,x')}\,\rmd x_s.
    \end{align*}
    The right-hand side equals $2\norm{\psi}_\infty\,\TV(p_{s|t,0}(\cdot | x_t,x),\,p_{s|t,0}(\cdot | x_t,x'))$. The bridge densities are Gaussians with the same covariance $\Sigma_{s,t} = \sigma_{s,t}^2\Idd$ and means differing by $\alpha_{s,t}(x - x')$, so by the standard identity
    \begin{align*}
        \TV\left(\mathcal{N}(m_1,\sigma^2 \Idd),\mathcal{N}(m_2,\sigma^2 \Idd)\right) = 2\Phi\left(\frac{\norm{m_1 - m_2}}{2\sigma}\right) - 1,
    \end{align*}
     together with $\alpha_{s,t}/\sigma_{s,t} = \gamma_{t,s}$, we obtain \eqref{eq:kernel_pointwise}.

    For \eqref{eq:kernel_truncated}, let $x,x' \in B_u$, so that $\norm{x - x'} \leq \diam(B_u) = 2\kappa(\sqrt{d}+u)$. Since $\Phi$ is nondecreasing, \eqref{eq:kernel_pointwise} gives
    \begin{align*}
        [G_{s|t,x_t}\psi]_{B_u} \leq 2\norm{\psi}_\infty\left[2\Phi\left(\gamma_{t,s}\,\kappa(\sqrt{d}+u)\right) - 1\right].
    \end{align*}
    Finally $2\Phi(r) - 1 \leq 1$ for every $r \geq 0$, and, with $\varphi$ the standard normal density,
    \begin{align*}
        2\Phi(r) - 1 = \int_{-r}^{r}\varphi(z)\,\rmd z \leq 2r\,\varphi(0) = \sqrt{\tfrac{2}{\pi}}\,r,
    \end{align*}
    so that $2\Phi(r) - 1 \leq \min\{1,\sqrt{2/\pi}\,r\}$, which yields \eqref{eq:kernel_truncated}.
\end{proof}

Combining the previous two lemmas yields the main per-step bound.

\begin{theorem}\label{thm:per_step_error}
    Fix  $0 < s < t \leq T$ and $x_t \in \R^d$, and assume the sub-Gaussian tail assumption holds at $(t,x_t)$ with center $m_t(x_t)$ and rate $\kappa_t(x_t)$.  Write $u_M = \sqrt{3\log M}$ and
    \begin{align}\label{eq:theta_def}
        \Theta_{t,s}(u) = \min\left\{1,\ \sqrt{\tfrac{2}{\pi}}\,\gamma_{t,s}\,\kappa_t(x_t)\left(\sqrt{d}+u\right)\right\}.
    \end{align}
    Then, for any bounded test function $\psi$,
    \begin{align*}
        \norm{\psi}_\infty^{-1}\abs{\mathbb{E}\left[p^M_{s|t}(\psi \mid x_t) - p_{s|t}(\psi \mid x_t)\right]} &\leq  \frac{12\,\rho_t(x_t)}{M}\,\Theta_{t,s}(u_M) + \frac{2\sqrt{\rho_t(x_t)}}{M^3} + \frac{4}{M^2}, \\
        \norm{\psi}_\infty^{-2}\,\mathbb{E}\left[\left(p^M_{s|t}(\psi \mid x_t) - p_{s|t}(\psi \mid x_t)\right)^2\right] &\leq  \frac{8\,\rho_t(x_t)}{M}\,\Theta_{t,s}(u_M)^2 + \frac{8\,\rho_t(x_t)}{M^6} + \frac{32}{M^2}.
    \end{align*}
     In particular, bounding $\Theta_{t,s}$ by its linear branch,
    \begin{align*}
        \norm{\psi}_\infty^{-1}\abs{\mathbb{E}\left[p^M_{s|t}(\psi \mid x_t) - p_{s|t}(\psi \mid x_t)\right]} &\leq \frac{24\,\rho_t(x_t)}{M\sqrt{2\pi}}\,\gamma_{t,s}\,\kappa_t(x_t)\left(\sqrt{d}+u_M\right) + \frac{2\sqrt{\rho_t(x_t)}}{M^3} + \frac{4}{M^2}, \\
        \norm{\psi}_\infty^{-2}\,\mathbb{E}\left[\left(p^M_{s|t}(\psi \mid x_t) - p_{s|t}(\psi \mid x_t)\right)^2\right] &\leq \frac{16\,\rho_t(x_t)}{\pi M}\,\gamma_{t,s}^2\,\kappa_t(x_t)^2\left(\sqrt{d}+u_M\right)^2 + \frac{8\,\rho_t(x_t)}{M^6} + \frac{32}{M^2},
    \end{align*}
    while bounding it by its constant branch gives an $O(\rho_t(x_t)/M)$ bound valid uniformly in $\gamma_{t,s}$.
\end{theorem}

\begin{proof}
    Setting $f = G_{s|t,x_t}\psi$, the identity $p^M_{s|t}(\psi | x_t) = p^M_{0|t}(f | x_t)$ allows us to apply \Cref{lemma:agapiou_improved} with $\mu = p_{0|t}(\cdot | x_t)$ and $\nu = q_{0|t}(\cdot | x_t)$. For $u > 0$, take $B_u = \mathcal{B}(m_t(x_t),\,\kappa_t(x_t)(\sqrt{d}+u))$. By the sub-Gaussian tail assumption,
    \begin{align*}
        q_{0|t}(B_u^\mathsf{c} \mid x_t) \leq e^{-u^2}, \qquad \eta_{B_u^\mathsf{c}} \leq \rho_t(x_t)\,e^{-u^2}, \qquad \eta_{B_u^\mathsf{c}}\,q_{0|t}(B_u^\mathsf{c} \mid x_t) \leq \rho_t(x_t)\,e^{-2u^2}.
    \end{align*}
    Using  \eqref{eq:kernel_truncated} to control $[f]_{B_u}$ and $\norm{f}_\infty \leq \norm{\psi}_\infty$ in \Cref{lemma:agapiou_improved} gives
    \begin{align*}
        \abs{\mathbb{E}\left[p^M_{s|t}(\psi \mid x_t) - p_{s|t}(\psi \mid x_t)\right]} &\leq  \frac{12\,\rho_t(x_t)}{M}\,\norm{\psi}_\infty\,\Theta_{t,s}(u) \\
        &\quad + 2\norm{\psi}_\infty\sqrt{\rho_t(x_t)}\,e^{-u^2} + 4\norm{\psi}_\infty\,M\,e^{-u^2},
    \end{align*}
    and similarly for the MSE. Choosing $u = u_M = \sqrt{3\log M}$ (so $e^{-u^2} = M^{-3}$) yields the stated bounds.
\end{proof}

\subsubsection{Adding an MCMC rejuvenation step}

We now study the same SNIS estimator equipped with an additional rejuvenation move. Let $R_{t,x_t}$ be a Markov kernel leaving $p_{0|t}(\cdot | x_t)$ invariant. We define the \emph{rejuvenated} estimator
\begin{align*}
    p^{M,L}_{s|t}(x_s \mid x_t) &= \sum_{m=1}^M w^m_{0|t}\,p_{s|t,0}(x_s \mid x_t, y^m_{0|t}), \\
    y^m_{0|t} &\sim R^L_{t,x_t}(x^m_{0|t},\cdot),\quad x^m_{0|t} \overset{\mathrm{iid}}{\sim} q_{0|t}(\cdot \mid x_t),
\end{align*}
where $L \geq 0$ is the number of MCMC steps.
 The bridge is integrated out in $p^{M,L}_{s|t}$, but the MCMC randomness is not: the $y^m_{0|t}$ enter through their realized values. Alongside it we introduce the \emph{MCMC-integrated} estimator, defined for any bounded $\psi$ by
\begin{align}\label{eq:mcmc_integrated_estimator}
    \bar p^{M,L}_{s|t}(\psi \mid x_t) = p^M_{0|t}\left(R^L_{t,x_t} G_{s|t,x_t}\psi \mid x_t\right) = \sum_{m=1}^M w^m_{0|t}\,\left(R^L_{t,x_t}G_{s|t,x_t}\psi\right)(x^m_{0|t}),
\end{align}
in which the MCMC move is averaged rather than sampled. The two are related as follows: since $w^m_{0|t}$ is $\sigma(x^{1:M}_{0|t})$-measurable and the $y^m_{0|t}$ are conditionally independent given $x^{1:M}_{0|t}$,
\begin{align}\label{eq:conditional_identity}
    \mathbb{E}\left[p^{M,L}_{s|t}(\psi \mid x_t) \mid x^{1:M}_{0|t}\right] = \bar p^{M,L}_{s|t}(\psi \mid x_t),
\end{align}
the identity holding in conditional expectation and \emph{not} pathwise. Together with the invariance identity $p_{s|t}(\psi | x_t) = p_{0|t}(R^L_{t,x_t}G_{s|t,x_t}\psi | x_t)$, \eqref{eq:conditional_identity} shows that the two estimators have the same bias, so every bias bound below applies verbatim to the sampled estimator $p^{M,L}_{s|t}$. Their mean squared errors differ by the conditional variance
\begin{align}\label{eq:conditional_variance}
    \operatorname{Var}\left[p^{M,L}_{s|t}(\psi \mid x_t) \mid x^{1:M}_{0|t}\right] = \sum_{m=1}^M \left(w^m_{0|t}\right)^2 \operatorname{Var}_{R^L_{t,x_t}(x^m_{0|t},\cdot)}\left[G_{s|t,x_t}\psi\right],
\end{align}
which does not vanish as $L \to \infty$. We therefore state the MSE bounds of this subsection for $\bar p^{M,L}_{s|t}$ and treat the sampled estimator separately in \Cref{rmk:sampled_variance}.

\begin{remark}[Scope of the invariance assumption]\label{rmk:adaptation_scope}
    Our results assume $R_{t,x_t}$ leaves $p_{0|t}(\cdot | x_t)$ invariant, which corresponds to MAExL at a \emph{fixed} step size. The production sampler of \Cref{app:maexl} adapts the step size online from the acceptance history, and such a chain need not be exactly invariant at finite $L$. The analysis is however insensitive to how the step size is chosen as long as it is not a function of the current state: if $(h_\ell)_{\ell \leq L}$ is generated independently of the chain (frozen, or adapted on a pilot run), then conditionally on $(h_\ell)_{\ell \leq L}$ the composition $R^{(h_1)}_{t,x_t}\cdots R^{(h_L)}_{t,x_t}$ is again invariant, \Cref{prop:snis_mcmc_rejuvenation} applies with $\Delta_{t,L}$ replaced by its conditional counterpart, and averaging over $(h_\ell)_{\ell \leq L}$ preserves both bounds. Quantifying the invariance defect of the fully state-dependent scheme at finite $L$ is beyond our scope.
\end{remark}

\begin{proposition}[SNIS with $\mu$-invariant MCMC rejuvenation]\label{prop:snis_mcmc_rejuvenation}
    In the setting of \Cref{lemma:agapiou_improved}, let $R$ be a Markov kernel on $\R^d$ leaving $\mu$ invariant, and for $L \geq 0$ define the $L$-step rejuvenated SNIS estimator
    \begin{align*}
        \mu^{N,L}(f) = \sum_{i=1}^N \frac{w(X^i)}{\sum_{j=1}^N w(X^j)}\,(R^L f)(X^i), \qquad X^i \overset{\mathrm{iid}}{\sim} \nu.
    \end{align*}
    Then, for any bounded $\mu$-measurable function $f$ and any $B \subset \R^d$,
    \begin{align}
        \abs{\mathbb{E}\left[\mu^{N,L}(f) - \mu(f)\right]} &\leq \frac{6\rho}{N}\,[R^L f]_B + 2\norm{f}_\infty\sqrt{\eta_{B^\mathsf{c}}\nu(B^\mathsf{c})} + 4\norm{f}_\infty\,N\,\nu(B^\mathsf{c}), \label{eq:snis_mcmc_bias_general} \\
        \mathbb{E}\left[\left(\mu^{N,L}(f) - \mu(f)\right)^2\right] &\leq \frac{2\rho}{N}\,[R^L f]_B^2 + 8\norm{f}_\infty^2\,\eta_{B^\mathsf{c}}\nu(B^\mathsf{c}) + 32\norm{f}_\infty^2\,N\,\nu(B^\mathsf{c}). \label{eq:snis_mcmc_mse_general}
    \end{align}
    Moreover, defining the local Wasserstein diameter
    \begin{align*}
        \Delta_L(B) = \sup_{x,x'\in B} W_1\left(R^L(x,\cdot),\,R^L(x',\cdot)\right),
    \end{align*}
    if $f$ is Lipschitz then $[R^L f]_B \leq \operatorname{Lip}(f)\,\Delta_L(B)$, and \eqref{eq:snis_mcmc_bias_general}--\eqref{eq:snis_mcmc_mse_general} give
    \begin{align}
        \abs{\mathbb{E}\left[\mu^{N,L}(f) - \mu(f)\right]} &\leq \frac{6\rho}{N}\,\operatorname{Lip}(f)\Delta_L(B) + 2\norm{f}_\infty\sqrt{\eta_{B^\mathsf{c}}\nu(B^\mathsf{c})} + 4\norm{f}_\infty\,N\,\nu(B^\mathsf{c}), \\
        \mathbb{E}\left[\left(\mu^{N,L}(f) - \mu(f)\right)^2\right] &\leq \frac{2\rho}{N}\,\operatorname{Lip}(f)^2\Delta_L(B)^2 + 8\norm{f}_\infty^2\,\eta_{B^\mathsf{c}}\nu(B^\mathsf{c}) + 32\norm{f}_\infty^2\,N\,\nu(B^\mathsf{c}).
    \end{align}
\end{proposition}

\begin{proof}
    Set $f_L = R^L f$. Since $R$ leaves $\mu$ invariant, $\mu(f_L) = \mu(R^L f) = \mu(f)$, hence
    \begin{align*}
        \mu^{N,L}(f) - \mu(f) = \mu^N(f_L) - \mu(f_L),
    \end{align*}
    where $\mu^N$ is the standard SNIS approximation of $\mu$ with proposal $\nu$ applied to $f_L$. Applying \Cref{lemma:agapiou_improved} to $f_L$ (and using $\norm{R^L f}_\infty \leq \norm{f}_\infty$) gives \eqref{eq:snis_mcmc_bias_general}--\eqref{eq:snis_mcmc_mse_general}.

    For the Lipschitz refinement, fix $x,x' \in B$ and let $\Pi$ be any coupling of $R^L(x,\cdot)$ and $R^L(x',\cdot)$. Then
    \begin{align*}
        \abs{R^L f(x) - R^L f(x')} \leq \int \abs{f(y) - f(y')}\,\Pi(\rmd y,\rmd y') \leq \operatorname{Lip}(f) \int \norm{y - y'}\,\Pi(\rmd y,\rmd y').
    \end{align*}
    Optimizing over $\Pi$ and taking the supremum over $x,x' \in B$ gives $[R^L f]_B \leq \operatorname{Lip}(f)\,\Delta_L(B)$.
\end{proof}

The next proposition refines \Cref{lemma:kernel_sensitivity} by showing that the MCMC step contracts the effective $x_0$-diameter seen by the bridge.

\begin{proposition}[MCMC-smoothed kernel sensitivity]\label{prop:mcmc_smoothed_kernel_sensitivity}
    Let $x_t \in \R^d$,  $0 < s < t \leq T$, and let $R_t$ be a Markov kernel on $\R^d$. For $L \geq 0$ and $B \subset \R^d$, set
    \begin{align*}
        \Delta_{t,L}(B) = \sup_{x,x'\in B} W_1\left(R_t^L(x,\cdot),\,R_t^L(x',\cdot)\right).
    \end{align*}
    Then, for any bounded test function $\psi$,
    \begin{align*}
        \norm{R_t^L G_{s|t,x_t}\psi}_\infty \leq \norm{\psi}_\infty,
    \end{align*}
    and
    \begin{align}
        [R_t^L G_{s|t,x_t}\psi]_B &\leq 2\norm{\psi}_\infty\left[2\Phi\left(\frac{\gamma_{t,s}\,\Delta_{t,L}(B)}{2}\right) - 1\right] \label{eq:KG_exact_sensitivity} \\
        &\leq  \norm{\psi}_\infty \min\left\{2,\ \sqrt{\tfrac{2}{\pi}}\,\gamma_{t,s}\,\Delta_{t,L}(B)\right\}. \label{eq:KG_linear_sensitivity}
    \end{align}
    In particular, if $\Delta_{t,L}(B) \leq a_{t,L}(B)\,\diam(B)$ with $a_{t,L}(B) \in [0,1]$, then
    \begin{align*}
        [R_t^L G_{s|t,x_t}\psi]_B \leq \sqrt{\tfrac{2}{\pi}}\,\norm{\psi}_\infty\,\gamma_{t,s}\,a_{t,L}(B)\,\diam(B);
    \end{align*}
    and if there exist $C_t < \infty$ and $r_t \in (0,1)$ such that $W_1(R_t^L(x,\cdot),R_t^L(x',\cdot)) \leq C_t\,r_t^L\,\norm{x-x'}$ for all $x,x' \in B$, then
    \begin{align*}
        [R_t^L G_{s|t,x_t}\psi]_B \leq \sqrt{\tfrac{2}{\pi}}\,\norm{\psi}_\infty\,\gamma_{t,s}\,C_t\,r_t^L\,\diam(B).
    \end{align*}
\end{proposition}

\begin{proof}
    The sup-norm bound follows because both $R_t^L$ and $G_{s|t,x_t}$ are expectation operators. For the oscillation, let $x,x' \in B$ and let $\Pi$ be any coupling of $R_t^L(x,\cdot)$ and $R_t^L(x',\cdot)$. Then
    \begin{align*}
        \abs{R_t^L G_{s|t,x_t}\psi(x) - R_t^L G_{s|t,x_t}\psi(x')} \leq \int \abs{G_{s|t,x_t}\psi(y) - G_{s|t,x_t}\psi(y')}\,\Pi(\rmd y,\rmd y'),
    \end{align*}
    and by  \eqref{eq:kernel_pointwise} the integrand is at most $2\norm{\psi}_\infty[2\Phi(\gamma_{t,s}\norm{y-y'}/2) - 1]$. The map $r \mapsto 2\Phi(\gamma_{t,s} r/2) - 1$ is nondecreasing and concave on $\R_+$, so Jensen's inequality gives
    \begin{align*}
        \int \left[2\Phi\left(\frac{\gamma_{t,s}\norm{y - y'}}{2}\right) - 1\right]\Pi(\rmd y,\rmd y') \leq 2\Phi\left(\frac{\gamma_{t,s}}{2}\int \norm{y - y'}\Pi(\rmd y,\rmd y')\right) - 1.
    \end{align*}
    Optimizing over $\Pi$ and taking the supremum over $x,x' \in B$ yields \eqref{eq:KG_exact_sensitivity};  the inequality $2\Phi(r) - 1 \leq \min\{1,\sqrt{2/\pi}\,r\}$ for $r \geq 0$ then gives \eqref{eq:KG_linear_sensitivity}.
\end{proof}

\begin{theorem}[Per-step error with MCMC rejuvenation]\label{thm:mcmc_improved_per_step_error}
    Fix  $0 < s < t \leq T$ and $x_t \in \R^d$, and assume the sub-Gaussian tail assumption holds at $(t,x_t)$ with center $m_t(x_t)$ and rate $\kappa_t(x_t)$. For $u > 0$, set
    \begin{align*}
        B_u = \mathcal{B}\left(m_t(x_t),\,\kappa_t(x_t)(\sqrt{d}+u)\right), \qquad \Delta_{t,L}(u) = \Delta_{t,L}(B_u).
    \end{align*}
    Then, for any bounded test function $\psi$,  the bias of the sampled estimator $p^{M,L}_{s|t}$ and the mean squared error of the MCMC-integrated estimator $\bar p^{M,L}_{s|t}$ of \eqref{eq:mcmc_integrated_estimator} satisfy
    \begin{align}
        \norm{\psi}_\infty^{-1}\abs{\mathbb{E}\left[p^{M,L}_{s|t}(\psi \mid x_t) - p_{s|t}(\psi \mid x_t)\right]} &\leq \frac{12\,\rho_t(x_t)}{M}\left[2\Phi\left(\frac{\gamma_{t,s}\,\Delta_{t,L}(u)}{2}\right) - 1\right] \nonumber \\
        &\quad + 2\sqrt{\rho_t(x_t)}\,e^{-u^2} + 4M\,e^{-u^2}, \label{eq:mcmc_improved_bias_exact} \\
        \norm{\psi}_\infty^{-2}\,\mathbb{E}\left[\left( \bar p^{M,L}_{s|t}(\psi \mid x_t) - p_{s|t}(\psi \mid x_t)\right)^2\right] &\leq \frac{8\,\rho_t(x_t)}{M}\left[2\Phi\left(\frac{\gamma_{t,s}\,\Delta_{t,L}(u)}{2}\right) - 1\right]^2 \nonumber \\
        &\quad + 8\,\rho_t(x_t)\,e^{-2u^2} + 32M\,e^{-u^2}. \label{eq:mcmc_improved_mse_exact}
    \end{align}
     Bounding $2\Phi(r) - 1 \leq \min\{1,\sqrt{2/\pi}\,r\}$ and writing
    \begin{align}\label{eq:lambda_def}
        \Lambda_{t,L}(u) = \min\left\{1,\ \tfrac{1}{\sqrt{2\pi}}\,\gamma_{t,s}\,\Delta_{t,L}(u)\right\},
    \end{align}
    \begin{align}
        \norm{\psi}_\infty^{-1}\abs{\mathbb{E}\left[p^{M,L}_{s|t}(\psi \mid x_t) - p_{s|t}(\psi \mid x_t)\right]} &\leq  \frac{12\,\rho_t(x_t)}{M}\,\Lambda_{t,L}(u) \nonumber \\
        &\quad + 2\sqrt{\rho_t(x_t)}\,e^{-u^2} + 4M\,e^{-u^2}, \label{eq:mcmc_improved_bias_linear} \\
        \norm{\psi}_\infty^{-2}\,\mathbb{E}\left[\left( \bar p^{M,L}_{s|t}(\psi \mid x_t) - p_{s|t}(\psi \mid x_t)\right)^2\right] &\leq  \frac{8\,\rho_t(x_t)}{M}\,\Lambda_{t,L}(u)^2 \nonumber \\
        &\quad + 8\,\rho_t(x_t)\,e^{-2u^2} + 32M\,e^{-u^2}. \label{eq:mcmc_improved_mse_linear}
    \end{align}
\end{theorem}

\begin{proof}
    Set $f = G_{s|t,x_t}\psi$. Since $R_t$ is $p_{0|t}(\cdot | x_t)$-invariant, $p_{0|t}(R_t^L f | x_t) = p_{0|t}(f | x_t) = p_{s|t}(\psi | x_t)$. Apply \Cref{prop:snis_mcmc_rejuvenation} with $\mu = p_{0|t}(\cdot | x_t)$, $\nu = q_{0|t}(\cdot | x_t)$, $N = M$, and $f = G_{s|t,x_t}\psi$;  the resulting estimator $\mu^{M,L}(f)$ is exactly $\bar p^{M,L}_{s|t}(\psi | x_t)$, and by \eqref{eq:conditional_identity} its bias coincides with that of $p^{M,L}_{s|t}(\psi | x_t)$. The sub-Gaussian tail assumption gives, exactly as in the proof of \Cref{thm:per_step_error},
    \begin{align*}
        q_{0|t}(B_u^\mathsf{c} \mid x_t) \leq e^{-u^2}, \qquad \eta_{B_u^\mathsf{c}}\,q_{0|t}(B_u^\mathsf{c} \mid x_t) \leq \rho_t(x_t)\,e^{-2u^2}.
    \end{align*}
    The oscillation term is controlled by \Cref{prop:mcmc_smoothed_kernel_sensitivity}:
    \begin{align*}
        [R_t^L G_{s|t,x_t}\psi]_{B_u} \leq 2\norm{\psi}_\infty\left[2\Phi\left(\frac{\gamma_{t,s}\,\Delta_{t,L}(u)}{2}\right) - 1\right].
    \end{align*}
    Plugging in yields \eqref{eq:mcmc_improved_bias_exact}--\eqref{eq:mcmc_improved_mse_exact}; linearization gives \eqref{eq:mcmc_improved_bias_linear}--\eqref{eq:mcmc_improved_mse_linear}.
\end{proof}

\begin{remark}[MSE of the sampled estimator]\label{rmk:sampled_variance}
    The MSE bounds above are stated for $\bar p^{M,L}_{s|t}$; by \eqref{eq:conditional_variance} the sampled estimator $p^{M,L}_{s|t}$ carries an extra conditional variance that does not vanish as $L \to \infty$. Controlling it needs no new argument. On the extended space $\R^d \times \R^d$, set
    \begin{align*}
        \bar\nu(\rmd x_0,\rmd y) = q_{0|t}(\rmd x_0 \mid x_t)R^L_{t,x_t}(x_0,\rmd y), \qquad \bar\mu(\rmd x_0,\rmd y) = p_{0|t}(\rmd x_0 \mid x_t)R^L_{t,x_t}(x_0,\rmd y).
    \end{align*}
    Then $\rmd\bar\mu/\rmd\bar\nu(x_0,y) = \tilde w_{0|t}(x_0)/p_t(x_t)$, so $1 + \chi^2(\bar\mu\|\bar\nu) = \rho_t(x_t)$ and the SNIS weights built from $\bar\nu$ coincide with $w^{1:M}_{0|t}$. By invariance the $y$-marginal of $\bar\mu$ is $p_{0|t}(\cdot | x_t)$, so $\bar\mu(\bar f) = p_{s|t}(\psi | x_t)$ for $\bar f(x_0,y) = G_{s|t,x_t}\psi(y)$. Thus $p^{M,L}_{s|t}(\psi | x_t)$ is the SNIS approximation of $\bar\mu$ applied to $\bar f$, and \Cref{lemma:agapiou_improved} with $B = \R^d \times B^y_u$, combined with \eqref{eq:kernel_truncated}, gives the MSE bound of \Cref{thm:per_step_error} with $\kappa_t(x_t)$ replaced by the sub-Gaussian rate of the $y$-marginal of $\bar\nu$, i.e. of the $L$-step pushforwards of $q_{0|t}(\cdot | x_t)$ and of the normalized square-weight measure through $R^L_{t,x_t}$. Rejuvenation therefore improves the bias by the mixing factor $r_t^L$ while leaving the MSE of the same order as without it; averaging $J$ conditionally independent moves per particle divides the extra term \eqref{eq:conditional_variance} by $J$.
\end{remark}

\begin{corollary}[Convenient choice of radius]\label{cor:mcmc_improved_convenient_radius}
    Under the assumptions of \Cref{thm:mcmc_improved_per_step_error}, taking $u_M = \sqrt{3\log M}$ gives
    \begin{align*}
        \norm{\psi}_\infty^{-1}\abs{\mathbb{E}\left[p^{M,L}_{s|t}(\psi \mid x_t) - p_{s|t}(\psi \mid x_t)\right]} &\leq  \frac{12\,\rho_t(x_t)}{M}\,\Lambda_{t,L}(u_M) + \frac{2\sqrt{\rho_t(x_t)}}{M^3} + \frac{4}{M^2}, \\
        \norm{\psi}_\infty^{-2}\,\mathbb{E}\left[\left( \bar p^{M,L}_{s|t}(\psi \mid x_t) - p_{s|t}(\psi \mid x_t)\right)^2\right] &\leq  \frac{8\,\rho_t(x_t)}{M}\,\Lambda_{t,L}(u_M)^2 + \frac{8\,\rho_t(x_t)}{M^6} + \frac{32}{M^2}.
    \end{align*}
\end{corollary}

\paragraph{Specialization to MAExL.}
We now specialize to the case where $R_t$ is the MAExL Markov kernel of \cite{Durmus2015quantitative} targeting $\pi_t = p_{0|t}(\cdot | x_t)$. Recall  $A_t = \alpha_t^2/\sigma_t^2$ from \eqref{eq:gamma_snr} and the auxiliary quantities
\begin{align*}
    \bar m_t(x_t) = \frac{x_t}{\alpha_t}, \qquad Q_t = A_t\,\Idd.
\end{align*}

\begin{corollary}[Per-step error with MAExL-rejuvenated SNIS]\label{cor:maexl_rejuvenated_snis_error}
    Fix  $0 < s < t \leq T$ and $x_t \in \R^d$. Assume that, after centering at $\bar m_t(x_t)$, the MAExL kernel $R_t$ satisfies the EI-MALA assumptions M1--M2 of \cite{Durmus2015quantitative}. Then by \citep[Proposition~3, Theorem~1, Theorem~2]{Durmus2015quantitative}, there exist constants $C_t < \infty$, $r_t \in (0,1)$, $\eta_t > 0$, $\varepsilon_t > 0$, and a bounded distance $d_t = d_{\eta_t,\varepsilon_t}$ such that for all $L \geq 0$ and $x,y \in \R^d$,
    \begin{align*}
        W_{d_t}\left(R_t^L(x,\cdot),\,R_t^L(y,\cdot)\right) \leq C_t\,r_t^L\,\{V_t(x) + V_t(y)\}, \qquad V_t(x) = 1 \vee \norm{x - \bar m_t(x_t)}_{Q_t}.
    \end{align*}

    Assume the sub-Gaussian tail assumption holds at $(t,x_t)$ with center $m_t(x_t)$ and rate $\kappa_t(x_t)$. For $u > 0$, set $B_u = \mathcal{B}(m_t(x_t),\,\kappa_t(x_t)(\sqrt{d}+u))$ and $V_{t,u} = \sup_{x \in B_u} V_t(x)$. Define the normalized $d_t$-Lipschitz constant of $G_{s|t,x_t}\psi$,
    \begin{align*}
        \mathfrak{L}_{t,s} = \sup_{x \neq y} \frac{\abs{G_{s|t,x_t}\psi(x) - G_{s|t,x_t}\psi(y)}}{\norm{\psi}_\infty\,d_t(x,y)};
    \end{align*}
    using  \eqref{eq:kernel_pointwise} together with the lower comparison of $d_{\eta,\varepsilon}$ with the $Q_t$-norm \citep[Lemma~8]{Durmus2015quantitative}, one may take
    \begin{align}\label{eq:frakL_bound}
        \mathfrak{L}_{t,s} \leq 2 \vee \left(\sqrt{\tfrac{2}{\pi}}\,\gamma_{t,s}\,\frac{\varepsilon_t}{\sqrt{A_t}}\right).
    \end{align}

    Then, for any bounded test function $\psi$,  with the bias referring to $p^{M,L}_{s|t}$ and the MSE to $\bar p^{M,L}_{s|t}$,
    \begin{align}
        \norm{\psi}_\infty^{-1}\abs{\mathbb{E}\left[p^{M,L}_{s|t}(\psi \mid x_t) - p_{s|t}(\psi \mid x_t)\right]} &\leq \frac{12\,\rho_t(x_t)}{M}\,\mathfrak{L}_{t,s}\,C_t\,r_t^L\,V_{t,u} \nonumber \\
        &\quad + 2\sqrt{\rho_t(x_t)}\,e^{-u^2} + 4M\,e^{-u^2}, \label{eq:maexl_rejuvenated_bias} \\
        \norm{\psi}_\infty^{-2}\,\mathbb{E}\left[\left( \bar p^{M,L}_{s|t}(\psi \mid x_t) - p_{s|t}(\psi \mid x_t)\right)^2\right] &\leq \frac{8\,\rho_t(x_t)}{M}\,\mathfrak{L}_{t,s}^2\,C_t^2\,r_t^{2L}\,V_{t,u}^2 \nonumber \\
        &\quad + 8\,\rho_t(x_t)\,e^{-2u^2} + 32M\,e^{-u^2}. \label{eq:maexl_rejuvenated_mse}
    \end{align}
    Furthermore, $V_{t,u} \leq 1 \vee \sqrt{A_t}\left(\norm{m_t(x_t) - \bar m_t(x_t)} + \kappa_t(x_t)(\sqrt{d}+u)\right)$.
\end{corollary}

\begin{proof}
    Set $f = G_{s|t,x_t}\psi$. Since $R_t$ is $p_{0|t}(\cdot | x_t)$-invariant,  the MCMC-integrated estimator $\bar p^{M,L}_{s|t}$ coincides with the SNIS estimator applied to $R_t^L f$. Applying \Cref{prop:snis_mcmc_rejuvenation} with $\mu = p_{0|t}(\cdot | x_t)$, $\nu = q_{0|t}(\cdot | x_t)$, $N = M$, and $B = B_u$, we get
    \begin{align*}
        \abs{\mathbb{E}\left[ \bar p^{M,L}_{s|t}(\psi \mid x_t) - p_{s|t}(\psi \mid x_t)\right]} &\leq \frac{6\,\rho_t(x_t)}{M}\,[R_t^L f]_{B_u} + 2\norm{\psi}_\infty\sqrt{\eta_{B_u^\mathsf{c}}\,q_{0|t}(B_u^\mathsf{c} \mid x_t)} \\
        &\quad + 4\norm{\psi}_\infty\,M\,q_{0|t}(B_u^\mathsf{c} \mid x_t),
    \end{align*}
    and analogously for the MSE;  by \eqref{eq:conditional_identity} the left-hand side is unchanged if $\bar p^{M,L}_{s|t}$ is replaced by $p^{M,L}_{s|t}$.

    For the oscillation term, note that $f$ is $\norm{\psi}_\infty\,\mathfrak{L}_{t,s}$-Lipschitz with respect to $d_t$, so the quantitative Wasserstein control of MAExL gives
    \begin{align*}
        [R_t^L f]_{B_u} \leq \norm{\psi}_\infty\,\mathfrak{L}_{t,s}\sup_{x,y \in B_u} W_{d_t}\left(R_t^L(x,\cdot),\,R_t^L(y,\cdot)\right) \leq 2\norm{\psi}_\infty\,\mathfrak{L}_{t,s}\,C_t\,r_t^L\,V_{t,u}.
    \end{align*}
    The sub-Gaussian tail assumption gives $q_{0|t}(B_u^\mathsf{c} | x_t) \leq e^{-u^2}$ and $\eta_{B_u^\mathsf{c}}\,q_{0|t}(B_u^\mathsf{c} | x_t) \leq \rho_t(x_t)\,e^{-2u^2}$, yielding \eqref{eq:maexl_rejuvenated_bias}--\eqref{eq:maexl_rejuvenated_mse}.

    The bound on $V_{t,u}$ follows from the triangle inequality and the definition of the $Q_t$-norm.
\end{proof}

\paragraph{Interpretation.}
Compared with the non-rejuvenated SNIS bound of \Cref{thm:per_step_error}, the leading oscillation term is multiplied by the MAExL mixing factor $C_t\,r_t^L\,V_{t,u}$. Thus the bridge sensitivity $\gamma_{t,s}$ is unaltered by the MCMC step; rather, MAExL contracts the effective $x_0$-diameter seen by the bridge. The constants $C_t, r_t$ and the adapted metric $d_t = d_{\eta_t,\varepsilon_t}$ are those of the EI-MALA Wasserstein analysis of \cite{Durmus2015quantitative}.  The analysis never requires the rejuvenation output to be an exact draw from $p_{0|t}(\cdot | x_t)$: only the invariance identity $p_{0|t}(R^L_{t,x_t}G_{s|t,x_t}\psi | x_t) = p_{s|t}(\psi | x_t)$ is used, and the quality of the finite-$L$ chain, together with its dependence on the initialization supplied by the proposal, enters solely through $\Delta_{t,L}(B)$, which is finite and in general nonzero for every finite $L$.

\subsubsection{From per-step bounds to a path-error bound}

We now propagate the per-step bias and MSE bounds along the diffusion schedule.

\begin{proposition}[Telescoping]\label{prop:telescoping}
    Let $t_K > t_{K-1} > \cdots > t_0$, and let $P_k$ be the exact Markov kernel from time $t_k$ to $t_{k-1}$. Let $\widehat P_1,\ldots,\widehat P_K$ be independent random Markov kernels approximating $P_1,\ldots,P_K$.  Given a probability measure $\eta_K$ on $\R^d$, define the exact marginals
    \begin{align*}
        \eta_{k-1} = \eta_k\,P_k \quad (k = K,\ldots,1),
    \end{align*}
    so that for any bounded $\psi$, $\eta_0(\psi) = \eta_K\,P_K\cdots P_1\,\psi$. Define the chained approximation
    \begin{align*}
        \widehat\eta_0(\psi) = \eta_K\,\widehat P_K\,\widehat P_{K-1}\cdots \widehat P_1\,\psi.
    \end{align*}
    Suppose that for every bounded measurable $f$, every $x \in \R^d$, and every $k$, there exist deterministic functions $b_k$ and $v_k$ such that
    \begin{align*}
        \abs{\mathbb{E}\left[\widehat P_k f(x) - P_k f(x)\right]} \leq b_k(x,\norm{f}_\infty), \qquad \mathbb{E}\left[\left(\widehat P_k f(x) - P_k f(x)\right)^2\right] \leq v_k(x,\norm{f}_\infty)^2.
    \end{align*}
    Then, for every bounded measurable $\psi$,
    \begin{align}
        \abs{\mathbb{E}\left[\widehat\eta_0(\psi) - \eta_0(\psi)\right]} &\leq \sum_{k=1}^K \eta_k\left[b_k(\cdot,\norm{\psi}_\infty)\right], \label{eq:bias_propagation} \\
        \mathbb{E}\left[\left(\widehat\eta_0(\psi) - \eta_0(\psi)\right)^2\right] &\leq \left(\sum_{k=1}^K \sqrt{\eta_k\left[v_k(\cdot,\norm{\psi}_\infty)^2\right]}\right)^{2}. \label{eq:mse_propagation}
    \end{align}
\end{proposition}

\begin{proof}
    For $k=1,\ldots,K$, define the random backward test functions $\widehat\psi_0 = \psi$ and $\widehat\psi_{k-1} = \widehat P_{k-1}\cdots \widehat P_1\,\psi$. Each $\widehat P_j$ is a Markov kernel, so $\norm{\widehat\psi_{k-1}}_\infty \leq \norm{\psi}_\infty$. The telescoping identity
    \begin{align*}
        \widehat P_K\cdots \widehat P_1 - P_K\cdots P_1 = \sum_{k=1}^K P_K\cdots P_{k+1}\,(\widehat P_k - P_k)\,\widehat P_{k-1}\cdots \widehat P_1
    \end{align*}
    gives
    \begin{align*}
        \widehat\eta_0(\psi) - \eta_0(\psi) = \sum_{k=1}^K \eta_k\left[(\widehat P_k - P_k)\,\widehat\psi_{k-1}\right].
    \end{align*}

    \emph{Bias.} Conditioning on $\widehat P_1,\ldots,\widehat P_{k-1}$ fixes $\widehat\psi_{k-1}$, hence
    \begin{align*}
        \abs{\mathbb{E}\left[\eta_k\left[(\widehat P_k - P_k)\,\widehat\psi_{k-1}\right]\right]} \leq \eta_k\left[b_k(\cdot,\norm{\psi}_\infty)\right].
    \end{align*}
    Summing over $k$ gives \eqref{eq:bias_propagation}.

    \emph{MSE.} By Minkowski's inequality and Jensen's inequality, conditioning on $\widehat P_1,\ldots,\widehat P_{k-1}$,
    \begin{align*}
        \left\{\mathbb{E}\left[\left(\widehat\eta_0(\psi) - \eta_0(\psi)\right)^2\right]\right\}^{1/2} &\leq \sum_{k=1}^K \left\{\mathbb{E}\left[\left(\eta_k\left[(\widehat P_k - P_k)\widehat\psi_{k-1}\right]\right)^2\right]\right\}^{1/2} \\
        &\leq \sum_{k=1}^K \sqrt{\eta_k\left[v_k(\cdot,\norm{\psi}_\infty)^2\right]},
    \end{align*}
    which gives \eqref{eq:mse_propagation}.
\end{proof}

Specializing to the chained SNIS estimators yields the path-error bounds.

\begin{corollary}[Path error without MCMC]\label{cor:chained_no_mcmc}
    Let $T = t_K > t_{K-1} > \cdots > t_0  = \tau > 0$, and for $k = 1,\ldots,K$ set $s_k = t_{k-1}$ and $\gamma_k = \gamma_{t_k,s_k}$. Let $P_k(x,\rmd x') = p_{s_k|t_k}(\rmd x' | x)$ be the exact denoising kernel and let $Q_k^M$ be its SNIS approximation. Assume $Q_1^M,\ldots,Q_K^M$ are independent and the initial law is exact, $q_{t_K} = p_{t_K}$, so that $\eta_k = p_{t_k}$. Define the chained approximation
    \begin{align*}
         q_\tau^M(\psi) = p_{t_K}\,Q_K^M\cdots Q_1^M\,\psi,
    \end{align*}
    and write $\rho_k(x) = 1 + \chi^2(p_{0|t_k}(\cdot | x)\,\|\,q_{0|t_k}(\cdot | x))$ and $\kappa_k(x) = \kappa_{t_k}(x)$. Assume that for each $k$, the sub-Gaussian tail assumption holds at $(t_k,x)$ for every $x$ in the support of $p_{t_k}$, with center $m_k(x)$ and rate $\kappa_k(x)$.  With $u_M = \sqrt{3\log M}$ and
    \begin{align*}
        \Theta_k(x) = \min\left\{1,\ \sqrt{\tfrac{2}{\pi}}\,\gamma_k\,\kappa_k(x)\left(\sqrt{d}+u_M\right)\right\},
    \end{align*}
    set
    \begin{align*}
        B_k(x) &=  \frac{12\,\rho_k(x)}{M}\,\Theta_k(x) + \frac{2\sqrt{\rho_k(x)}}{M^3} + \frac{4}{M^2}, \\
        V_k(x) &=  \frac{8\,\rho_k(x)}{M}\,\Theta_k(x)^2 + \frac{8\,\rho_k(x)}{M^6} + \frac{32}{M^2}.
    \end{align*}
    Then, for every bounded measurable $\psi$,
    \begin{align}
        \abs{\mathbb{E}\left[ q_\tau^M(\psi) - p_\tau(\psi)\right]} &\leq \norm{\psi}_\infty\sum_{k=1}^K \eta_k(B_k), \label{eq:chained_bias} \\
        \mathbb{E}\left[\left( q_\tau^M(\psi) - p_\tau(\psi)\right)^2\right] &\leq \norm{\psi}_\infty^2\left(\sum_{k=1}^K \sqrt{\eta_k(V_k)}\right)^{2}. \label{eq:chained_mse}
    \end{align}
\end{corollary}

\begin{proof}
    Apply \Cref{prop:telescoping} to $\widehat P_k = Q_k^M$, with the per-step bias and MSE bounds of \Cref{thm:per_step_error}  at $u = u_M$; these are legitimate since $s_k \geq \tau > 0$ for every $k$.
\end{proof}

\begin{corollary}[Path error with MCMC rejuvenation]\label{cor:chained_mcmc}
    In the setting of \Cref{cor:chained_no_mcmc}, let $Q_k^{M,L}$ be the MAExL-rejuvenated SNIS approximation of $P_k$ with $L$ MCMC steps, and assume $Q_1^{M,L},\ldots,Q_K^{M,L}$ are independent, the independence covering their MCMC randomness as well as their proposal draws. Assume the MAExL constants $C_k(x) < \infty$, $r_k(x) \in (0,1)$, $\varepsilon_{t_k} > 0$ of \Cref{cor:maexl_rejuvenated_snis_error} are well-defined at every $(t_k,x)$, and set
    \begin{align*}
        \mathfrak{L}_k &= 2 \vee \left(\sqrt{\tfrac{2}{\pi}}\,\gamma_k\,\frac{\varepsilon_{t_k}}{\sqrt{A_{t_k}}}\right), \\
        \mathfrak{V}_k(x) &= 1 \vee \sqrt{A_{t_k}}\left(\norm{m_k(x) - \tfrac{x}{\alpha_{t_k}}} + \kappa_k(x)\left(\sqrt{d}+u_M\right)\right),
    \end{align*}
    as in \eqref{eq:frakL_bound} and the last display of \Cref{cor:maexl_rejuvenated_snis_error}. Assume moreover that the sub-Gaussian tail assumption holds at $(t_k,x)$ for the $L$-step pushforwards of $q_{0|t_k}(\cdot | x)$ and of the normalized square-weight measure through $R^L_{t_k,x}$, with common center $\tilde m_{k,L}(x)$ and rate $\tilde\kappa_{k,L}(x)$, and write
    \begin{align*}
        \tilde\Theta_{k,L}(x) = \min\left\{1,\ \sqrt{\tfrac{2}{\pi}}\,\gamma_k\,\tilde\kappa_{k,L}(x)\left(\sqrt{d}+u_M\right)\right\}.
    \end{align*}
    Define the chained approximation
    \begin{align*}
        q_\tau^{M,L}(\psi) = p_{t_K}\,Q_K^{M,L}\cdots Q_1^{M,L}\,\psi,
    \end{align*}
    and set
    \begin{align*}
        B_k(x) &= \frac{12\,\rho_k(x)}{M}\,\mathfrak{L}_k\,C_k(x)\,r_k(x)^L\,\mathfrak{V}_k(x) + \frac{2\sqrt{\rho_k(x)}}{M^3} + \frac{4}{M^2}, \\
        V_k(x) &= \frac{8\,\rho_k(x)}{M}\,\tilde\Theta_{k,L}(x)^2 + \frac{8\,\rho_k(x)}{M^6} + \frac{32}{M^2}.
    \end{align*}
    Then, for every bounded measurable $\psi$,
    \begin{align}
        \abs{\mathbb{E}\left[q_\tau^{M,L}(\psi) - p_\tau(\psi)\right]} &\leq \norm{\psi}_\infty\sum_{k=1}^K \eta_k(B_k), \\
        \mathbb{E}\left[\left(q_\tau^{M,L}(\psi) - p_\tau(\psi)\right)^2\right] &\leq \norm{\psi}_\infty^2\left(\sum_{k=1}^K \sqrt{\eta_k(V_k)}\right)^{2}.
    \end{align}
    The bias bound therefore contracts geometrically in $L$, while the MSE bound is of the same order as in \Cref{cor:chained_no_mcmc}; see \Cref{rmk:sampled_variance}.
\end{corollary}

\begin{proof}
    Each $Q_k^{M,L}$ is a random Markov kernel in $x_{t_k}$: the bridge remains integrated in
    \begin{align*}
        Q_k^{M,L}(x,\rmd x') = \sum_{m=1}^M w^m_{0|t_k}\,p_{s_k|t_k,0}(\rmd x' \mid x, y^m_{0|t_k}),
    \end{align*}
    the MCMC draws $y^m_{0|t_k}$ being part of its internal randomness. \Cref{prop:telescoping} therefore applies with $\widehat P_k = Q_k^{M,L}$, the assumed independence across $k$ covering both sources of randomness.

    It remains to supply $b_k$ and $v_k$. Fix $k$ and a bounded measurable $f$, and apply the per-step results at $(t_k,x)$ with $s = s_k$, $\psi = f$ and $u = u_M$, which is legitimate since $s_k \geq \tau > 0$. For the bias, \eqref{eq:maexl_rejuvenated_bias} together with \eqref{eq:frakL_bound} and $V_{t_k,u_M} \leq \mathfrak{V}_k(x)$ gives
    \begin{align*}
        \abs{\mathbb{E}\left[Q_k^{M,L}f(x) - P_kf(x)\right]} \leq \norm{f}_\infty B_k(x),
    \end{align*}
    this being a bound on the sampled estimator by \eqref{eq:conditional_identity}. For the MSE, \Cref{rmk:sampled_variance} gives
    \begin{align*}
        \mathbb{E}\left[\left(Q_k^{M,L}f(x) - P_kf(x)\right)^2\right] \leq \norm{f}_\infty^2 V_k(x),
    \end{align*}
    with $\tilde\kappa_{k,L}(x)$ the sub-Gaussian rate of the $L$-step pushforwards assumed above. Both bounds are deterministic and of the form required by \Cref{prop:telescoping}, and since $\norm{\widehat\psi_{k-1}}_\infty \leq \norm{\psi}_\infty$ for every $k$, \eqref{eq:bias_propagation}--\eqref{eq:mse_propagation} yield the claim.
\end{proof}

\begin{remark}[Early stopping and terminal initialization]\label{rmk:endpoints}
    The grid terminates at $t_0 = \tau > 0$ ($\tau = 10^{-4}$ in the experiments), so by \eqref{eq:gamma_snr} every step has $\gamma_k \leq \sqrt{A_{s_k}} \leq \sqrt{A_\tau} = \alpha_\tau/\sigma_\tau < \infty$ and the constants above are finite, the final-step constant scaling as $\sigma_\tau^{-1}$. The truncated form \eqref{eq:kernel_truncated} keeps the per-step bound $O(\rho_k(x)/M)$ uniformly in $\tau$. The target controlled by these bounds is thus $p_\tau$, not $p_0$. Likewise, \Cref{cor:chained_no_mcmc,cor:chained_mcmc} assume $q_{t_K} = p_{t_K}$, whereas the sampler starts from $\nu_T = \mathcal{N}(0,\sigma_T^2\Idd)$: decomposing along $\nu_T P_K\cdots P_1\psi$ adds $2\norm{\psi}_\infty\TV(\nu_T,p_{t_K})$ to the bias bound and to the square root of the MSE bound, with $\eta_k$ then the marginals of the exact backward chain started at $\nu_T$. Under the VE schedule $\alpha_T = 1$, so $p_{t_K} = p_0 * \mathcal{N}(0,\sigma_T^2\Idd)$ and joint convexity of the Kullback--Leibler divergence together with Pinsker's inequality give
    \begin{align*}
        \TV(\nu_T,p_{t_K}) \leq \frac{\sqrt{\mathbb{E}_{p_0}\norm{X_0}^2}}{2\sigma_T}.
    \end{align*}
    Neither term depends on $M$ or $L$, so consistency in $M$ holds for the $(\tau,T)$-truncated target only. Finally, bounding the bias uniformly over $\norm{\psi}_\infty \leq 1$ in \eqref{eq:chained_bias} controls $\TV$ between the law of the sampler output and $p_\tau$. The MSE bounds concern instead the across-run variability of expectation estimates.
\end{remark}

\subsection{Multiple Importance Sampling}\label{app:mis}

This appendix details the \ourmethod{}-MIS sampler introduced in \Cref{subsec:method:mis}, which builds on the \emph{multiple importance sampling} (MIS) literature \citep{veach1995optimally, owen2000safe, he2014optimalmixtureweightsmultiple, bugallo2017adaptive, elvira2019generalized}. Recall that \ourmethod{}-IS shares candidates only at initialization; the KL identity \eqref{eq:posterior_kl_gaussian} is what lets us extend this sharing further along the chain, since it shows that posteriors conditioned on nearby noisy states are themselves close. We specify the mixture proposal, component selection, and candidate sharing; derive the self-normalized importance-sampling correction for the reverse kernel; and give the full algorithm together with its computational cost.

\paragraph{Mixture proposal.} At noise level $t$, let $\{x_t^i\}_{i=1}^N$ denote the current trajectories. For each trajectory $i$, we build a mixture proposal by combining $C$ selected learned posteriors,
\begin{equation}\label{eq:mixture_proposal}
    m^{\theta}_{0|t}(x_0\mid x_t^i)
    = \sum_{c \in\mathcal{C}(i)}\omega_{i,c}\,q^\theta_{0|t}\left(x_0\mid x_t^c\right),
    \qquad \sum_{c\in\mathcal{C}(i)}\omega_{i,c}=1,
\end{equation}
while keeping the target unchanged at $p_{0|t}(\cdot| x_t^i)$. The component set $\mathcal{C}(i)$ is a multiset of trajectory indices that always contains $i$ itself, so trajectory $i$'s own posterior is always a summand; the other $C-1$ slots are drawn from elsewhere in the ensemble. \Cref{fig:mis_cartoon} works through a small example of this construction and of the candidate sharing described below.

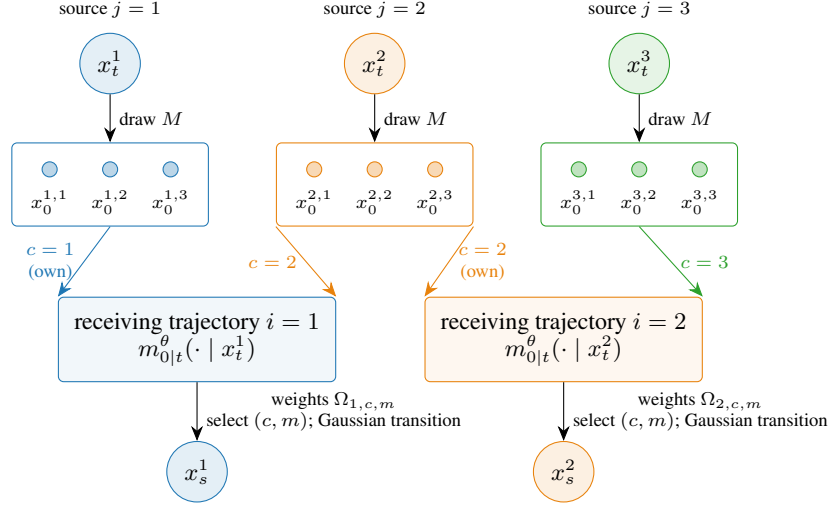
\begin{figure}[t]
\centering
\begin{tikzpicture}[
  >={Stealth[length=2mm,width=1.6mm]},
  font=\small,
  state/.style={circle, draw, minimum size=8mm, inner sep=2pt},
  batch/.style={rounded corners=2pt, draw, minimum width=2.6cm, minimum height=1.1cm},
  mix/.style={rounded corners=2pt, draw, align=center, inner sep=6pt},
  lab/.style={font=\scriptsize, align=center},
]
\foreach \j/\pos/\col in {1/0/parA,2/3.5/parB,3/7/parC}{
  \node[state,draw=\col,fill=\col!12] (s\j) at (\pos,0) {$x_t^{\j}$};
  \node[lab,above=2pt of s\j] {source $j=\j$};
  \node[batch,draw=\col] (b\j) at (\pos,-1.6) {};
  \draw[->] (s\j) -- node[lab,right] {draw $M$} (b\j);
  \foreach \m/\dx in {1/-0.8,2/0,3/0.8}{
    \node[circle,draw=\col,fill=\col!30,inner sep=2pt] at (\pos+\dx,-1.4) {};
    \node[lab] at (\pos+\dx,-1.83) {$x_0^{\j,\m}$};
  }
}
\node[mix,draw=parA,fill=parA!6] (m1) at (1.15,-3.65)
  {receiving trajectory $i=1$\\$m^\theta_{0|t}(\cdot\mid x_t^1)$};
\node[mix,draw=parB,fill=parB!6] (m2) at (6.0,-3.65)
  {receiving trajectory $i=2$\\$m^\theta_{0|t}(\cdot\mid x_t^2)$};
\draw[->,parA] (b1.south) -- node[lab,left] {$c=1$\\(own)} (m1.north west);
\draw[->,parB] (b2.south west) -- node[lab,left] {$c=2$} (m1.north east);
\draw[->,parB] (b2.south east) -- node[lab,right] {$c=2$\\(own)} (m2.north west);
\draw[->,parC] (b3.south) -- node[lab,right] {$c=3$} (m2.north east);
\node[state,draw=parA,fill=parA!12] (o1) at (1.15,-5.4) {$x_s^1$};
\node[state,draw=parB,fill=parB!12] (o2) at (6.0,-5.4) {$x_s^2$};
\draw[->] (m1) -- node[lab,right] {weights $\Omega_{1,c,m}$\\select $(c,m)$; Gaussian transition} (o1);
\draw[->] (m2) -- node[lab,right] {weights $\Omega_{2,c,m}$\\select $(c,m)$; Gaussian transition} (o2);
\end{tikzpicture}
\caption{\textbf{Candidate sharing and indexing in \ourmethod{}-MIS.} The example uses $N=3$ sources, $M=3$ candidates per source, and $C=2$ components per mixture; two receiving trajectories are shown. Candidate $x_0^{j,m}$ is indexed by its source $j$ and position $m$ within that source's batch. For receiving trajectory $i$, a component $c\in\mathcal{C}(i)$ is itself a source index: $c=i$ always appears (the trajectory's ``own'' component), and the remaining slots are drawn by posterior proximity. The middle batch, drawn once from source $j=2$, is reused by both trajectories: it serves as component $c=2$ of trajectory $1$'s mixture and as the own component $c=2$ of trajectory $2$'s mixture, receiving weight $\Omega_{1,2,m}$ and $\Omega_{2,2,m}$ respectively — the same $M$ flow draws inform two different reverse steps. Each trajectory samples $(c,m)$ from its own categorical weights and applies the endpoint-conditioned Gaussian transition. When a source is selected more than once within the same mixture, the repeat instead draws a fresh, independent batch, indexed by the occurrence counter $r(i,c)$ introduced below.}
\label{fig:mis_cartoon}
\end{figure}

\paragraph{Choosing components by posterior proximity.} By \eqref{eq:posterior_kl_gaussian}, for a Gaussian target $\mathcal{N}(\mu,\gamma^2\Idd)$ the denoising posteriors satisfy
\begin{align*}
\operatorname{KL}\left(p_{0|t}(\cdot\mid x)\,\|\,p_{0|t}(\cdot\mid y)\right) = \beta_t\norm{x-y}^2, \qquad \beta_t = \frac{\alpha_t^2\gamma^2}{2\sigma_t^2(\alpha_t^2\gamma^2+\sigma_t^2)},
\end{align*}
so state-space distance provides a schedule-dependent measure of posterior similarity in this Gaussian case; for molecular targets we use it as a heuristic with $\gamma^2=1$ in normalized coordinates. This is also what makes sharing most valuable at high noise levels: as $t$ grows, $\beta_t\to 0$ and the bound above shows the posteriors across \emph{all} trajectories converge toward the common Boltzmann target, so nearly any component is an informative proposal for any trajectory. Given $\beta_t$, define for every pair $(i,j)$ the proximity score
\begin{equation}\label{eq:mis_scores}
    \lambda_j(i) = \exp\left\{-\beta_t\norm{x_t^i - x_t^j}^2\right\}.
\end{equation}
The component set $\mathcal{C}(i)$ always retains the slot $i$ itself, and the remaining $C-1$ slots are drawn independently, with replacement, from the multinomial over $\{1,\ldots,N\}$ with weights $\lambda_1(i),\ldots,\lambda_N(i)$. The mixture weights are then set proportionally to these same scores, restricted to the selected slots,
\begin{equation}\label{eq:mis_mixture_weights}
    \omega_{i,c} = \frac{\lambda_c(i)}{\sum_{l\in\mathcal{C}(i)}\lambda_l(i)},
\end{equation}
so that conditioning states closer to $x_t^i$ contribute more strongly to its own proposal.

\paragraph{Rao-Blackwellized mixture sampling.} We now explain how expectations under $m^{\theta}_{0|t}(\cdot | x_t^i)$ are computed. Just as in \ourmethod{}-IS, we draw one \emph{base batch} of $M$ i.i.d.\ candidates $\xi_{j,1,1},\ldots,\xi_{j,1,M}\sim q^\theta_{0|t}(\cdot | x_t^j)$ for every trajectory $j=1,\ldots,N$ (this is the batches drawn under each source $j$ in \Cref{fig:mis_cartoon}). To assemble trajectory $i$'s $CM$ candidates, each slot $c\in\mathcal{C}(i)$ draws on the batch of its conditioning trajectory: the first time a trajectory is used as a component it reuses the shared base batch, as source $j=2$ does for both receiving trajectories in the figure. Because $\mathcal{C}(i)$ is a multiset, however, the same trajectory can occupy more than one slot of the \emph{same} mixture; reusing one batch twice within a single mixture would break the independence the estimator below relies on, so each repeated occurrence instead draws a fresh, independent batch. Writing $r(i,c)$ for the occurrence number of slot $c$ within $\mathcal{C}(i)$ (its first, second, \ldots\ appearance), we denote the resulting candidates
\begin{align*}
x_0^{i,c,m} = \xi_{c,\,r(i,c),\,m}, \qquad \xi_{j,r,m}\sim q^\theta_{0|t}(\cdot\mid x_t^j), \quad m=1,\ldots,M,
\end{align*}
where batches are shared across every trajectory whose component set calls for the same pair $(j,r)$, and drawn independently otherwise. This lets us reuse our expensive flow samples as much as possible while still having independent candidates within each trajectory's own mixture. Expectations under $m^\theta_{0|t}(\cdot | x_t^i)$ can then be computed in a Rao-Blackwellized fashion, by conditioning on the mixture component rather than sampling it:
\begin{equation}\label{eq:mis_rb_expectation}
    \mathbb{E}_{m^{\theta}_{0|t}(\cdot\mid x_t^i)}\left[\phi(X^i_0)\right] = \sum_{c \in \mathcal{C}(i)}  \omega_{i,c}\, \mathbb{E}_{q^{\theta}_{0|t}(\cdot\mid x_t^c)}\left[\phi(X_0^{i,c})\right] \approx \frac{1}{M}\sum_{c \in \mathcal{C}(i)} \sum_{m=1}^M \omega_{i,c}\, \phi\!\left(x_0^{i,c,m}\right).
\end{equation}

\paragraph{Self-normalized importance sampling for the reverse kernel.} We now return to the denoising kernel approximation. Writing the exact transition as an expectation under the mixture proposal,
\begin{align*}
    p_{s|t}(x^i_s\mid x^i_t) &= \int p_{s|t,0}(x^i_s\mid x^i_t,x^i_0)\, p_{0|t}(x^i_0\mid x^i_t)\,\rmd x^i_0 \\
    &= \int p_{s|t,0}(x^i_s\mid x^i_t,x^i_0)\, \frac{p_{0|t}(x^i_0\mid x^i_t)}{m^{\theta}_{0|t}(x^i_0\mid x^i_t)}\, m^{\theta}_{0|t}(x^i_0\mid x^i_t)\,\rmd x^i_0 \\
    &= \mathbb{E}_{m^\theta_{0|t}(\cdot\mid x_t^i)}\!\left[\frac{p_{0|t}(X_0^i\mid x_t^i)}{m^\theta_{0|t}(X_0^i\mid x_t^i)}\, p_{s|t,0}(x^i_s\mid x^i_t, X_0^i)\right] \\
    &\approx \sum_{c \in \mathcal{C}(i)} \sum_{m=1}^M \omega_{i,c}\, \Lambda_{i,c,m}\, p_{s|t,0}(x^i_s\mid x^i_t,x_0^{i,c,m}),
\end{align*}
where the last line applies the Rao-Blackwellized estimator \eqref{eq:mis_rb_expectation}, replacing the intractable ratio $p_{0|t}/m^\theta_{0|t}$ by a self-normalized importance weight $\Lambda_{i,c,m}$. Concretely, the unnormalized weights are
\begin{align*}
    \tilde{\Lambda}_{i,c,m} = \frac{p_{0|t}(x_0^{i,c,m}\mid x^i_t)}{m^{\theta}_{0|t}(x_0^{i,c,m}\mid x^i_t)} \propto \frac{p_{t|0}(x^i_t \mid x_0^{i,c,m})\, \exp\{-\mathcal{E}(x_0^{i,c,m})\}}{\sum_{c' \in\mathcal{C}(i)}\omega_{i,c'}\,q^\theta_{0|t}\left(x_0^{i,c,m}\mid x_t^{c'}\right)}.
\end{align*}
Because these weights only enter the final mixture in combination with $\omega_{i,c}$, we normalize the two together in one step,
\begin{equation}\label{eq:mis_omega}
    \Omega_{i,c,m} = \frac{\omega_{i,c}\,\tilde{\Lambda}_{i,c,m}}{\sum_{c' \in \mathcal{C}(i)} \sum_{m'=1}^M \omega_{i,c'}\,\tilde{\Lambda}_{i,c',m'}}.
\end{equation}
This gives a new mixture approximation of $p_{s|t}$,
\begin{equation}\label{eq:mis_kernel}
    \hat p^{\mathcal{C},M,\theta}_{s|t}(x^i_s\mid x^i_t) = \sum_{c \in \mathcal{C}(i)} \sum_{m=1}^M \Omega_{i,c,m}\; p_{s|t,0}(x^i_s\mid x^i_t,x_0^{i,c,m}),
\end{equation}
composed at the cost of only $M$ flow draws per component, yet at least as expressive than the $M$-candidate proposal of \ourmethod{}-IS.

\paragraph{The computational catch and the autoregressive blessing.} This added expressiveness is not free: computing $\Omega_{i,c,m}$ requires the cross-likelihoods
\begin{align*}
   q^\theta_{0|t}\left(x_0^{i,c,m}\mid x_t^{c'}\right), \qquad \text{for all } c, c' \in \mathcal{C}(i).
\end{align*}
 The $C$ "diagonal" terms $q^\theta_{0|t}(x_0^{i,\bar{c},m} | x_t^{\bar{c}})$ are returned for free with each flow sample and, like the samples themselves, are reused whenever a slot shares its source (\Cref{fig:mis_cartoon}). Only the remaining $C(C-1)$ cross terms per trajectory need a new evaluation, giving at most $NMC(C-1)$ additional evaluations per level (zero when $C=1$, where the mixture collapses to \ourmethod{}-IS) and less in practice when component sets overlap across trajectories. This overhead stays comparatively harmless for an autoregressive flow, where sampling is sequential and linear in token count but density evaluation is a single parallel pass: since sampling cost dominates, the batch-sharing scheme of \Cref{fig:mis_cartoon} is what drives wall-clock cost, while the extra cross-likelihoods remain comparatively negligible.

The full sampler is summarized in \Cref{alg:sequential_mis}.

\begin{algorithm}[t]
    \SetAlgoNlRelativeSize{-1}
    \setcounter{AlgoLine}{0}
    \renewcommand{\thealgocfproc}{mis}
    \SetAlgoLined
    \KwIn{Grid $0 = t_0 < \cdots < t_K = T$; posteriors $q^\theta_{0|t_k}$; target $p(x_0)\propto\exp\{-\mathcal{E}(x_0)\}$; trajectories $N$; batch size $M$; component slots $C$.}
    \KwOut{Ensemble $\{x_{t_0}^i\}_{i=1}^N$ approximately distributed as $p_0$.}
    Sample shared $x_{t_K}\sim\mathcal{N}(0,\sigma_T^2I)$ and draw $NM$ candidates from $q^\theta_{0|t_K}(\cdot\mid x_{t_K})$\;

    Form joint-pool SNIS weights as in \eqref{eq:denoising_snis_weights} with $NM$ candidates; select $N$ candidates with replacement according to these weights and propagate them through their DDIM kernels to obtain $x_{t_{K-1}}^i$\;

    \For{$k = K-1, \ldots, 1$}{
        For every trajectory $i$, compute the proximity scores $\lambda_j(i)=\exp\{-\beta_{t_k}\|x_{t_k}^i-x_{t_k}^j\|^2\}$, draw the component multiset $\mathcal{C}(i)$ by proximity-weighted sampling with slot $i$ fixed, and set $\omega_{i,\cdot}$ via \eqref{eq:mis_mixture_weights}\;

        Let $r(i,c)$ count the occurrences of $c$ among the slots of $\mathcal{C}(i)$ drawn so far. For each distinct pair $(j,r)$ arising across all trajectories, draw one independent batch of $M$ candidates $\xi_{j,r,m}\sim q^\theta_{0|t_k}(\cdot\mid x_{t_k}^j)$\;

        Assemble trajectory $i$'s $CM$ candidates as $x_0^{i,c,m}=\xi_{c,\,r(i,c),\,m}$, evaluate the required cross-likelihoods $q^\theta_{0|t_k}(x_0^{i,c,m}\mid x_{t_k}^{c'})$ for $c,c'\in\mathcal{C}(i)$, and form the selection weights
        $$
            \Omega_{i,c,m} \propto \omega_{i,c}\, \tilde{\Lambda}_{i,c,m}, \qquad \tilde{\Lambda}_{i,c,m} = \frac{p_{t_k|0}(x_{t_k}^i \mid x_0^{i,c,m})\,\exp\{-\mathcal{E}(x_0^{i,c,m})\}}{\sum_{c'\in\mathcal{C}(i)}\omega_{i,c'}\,q^\theta_{0|t_k}(x_0^{i,c,m} \mid x_{t_k}^{c'})}
        $$

        For every trajectory $i$, sample $(c^\star, m^\star) \sim \mathrm{Categorical}(\Omega_i)$ and set
        $$
            x_{t_{k-1}}^i \sim p_{t_{k-1}|t_k, 0}(\cdot \mid x_{t_k}^i,\, x_0^{i,c^\star,m^\star})
        $$
    }
    \Return $\{x_{t_0}^i\}_{i=1}^N$
    \caption{\ourmethod{}-MIS sampling.}
    \label{alg:sequential_mis}
\end{algorithm}

\paragraph{Amortized cost.} Let $U$ denote the number of distinct pairs $(j,r)$ drawn across the ensemble at a given level. Candidate generation then costs $UM$ flow samples, with $N\leq U\leq NC$, against $NCM$ without sharing. 

The lower bound $U=N$ holds when no source repeats \emph{within} any single trajectory's $\mathcal{C}(i)$: every external pick then falls at its first occurrence, already drawn by its source for itself, so cross-trajectory sharing is automatic. Such repeats become rarer as the proximity weights $\lambda_j(i)$ spread across more indices, the high-noise regime of \eqref{eq:posterior_kl_gaussian} ($\beta_t\to0$). The upper bound $U=NC$ is reached e.g.\ if every trajectory fills all $C$ slots with itself: no batch is ever shared, and the mixture collapses to a single density, paying $C\times$ the sampling cost of \ourmethod{}-IS for none of its added expressiveness.

\paragraph{Empirical ablation}

\begin{figure}
    \centering
    \includegraphics[width=\linewidth]{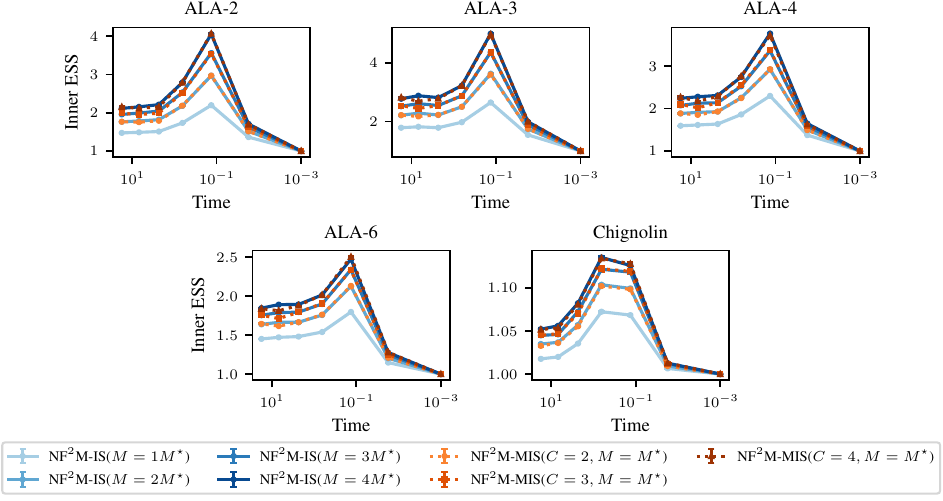}
    \caption{Average effective sample size of $q^{\theta}_{0|t}(\cdot | x_t)$ against $p_{0|t}(\cdot | x_t)$ as a function of $t$, for \ourmethod{}-IS and \ourmethod{}-MIS across different numbers of importance candidates or mixture components. $M^{\star}$ denotes the budget reported in \Cref{tab:hparam_nfm_is_budget}.}
    \label{fig:inner_ess_seqis_seqmis}
\end{figure}

\begin{figure}
    \centering
    \includegraphics[width=\linewidth]{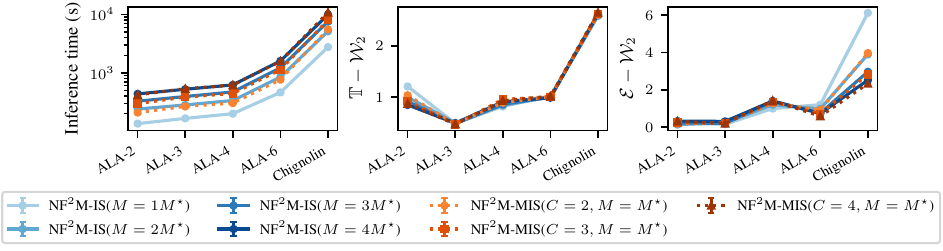}
    \caption{Inference time and distributional distances for all systems, comparing \ourmethod{}-IS and \ourmethod{}-MIS across different numbers of importance candidates or mixture components. $M^{\star}$ denotes the budget reported in \Cref{tab:hparam_nfm_is_budget}.}
    \label{fig:metrics_ablation_seqis_seq_mis}
\end{figure}

We compare the performance of \ourmethod{}-MIS against \ourmethod{}-IS along two axes: candidate-pool quality and downstream performance. \Cref{fig:inner_ess_seqis_seqmis} reports the average effective sample size at each denoising transition, and shows that \ourmethod{}-MIS with $C$ components and $M$ candidates per component matches the pool quality of \ourmethod{}-IS run with $CM$ independent candidates, confirming that sharing candidates across trajectories does not degrade the proposal. However, both methods exhibit low inner ESS, particularly for chignolin and at the lowest noise levels, indicating substantial importance-weight concentration. This is surprising, as noise levels near $t=0$ should make the target easiest to approximate. We hypothesize instead that the near-deterministic posterior forces the flow's affine-coupling scale network into a high-variance regime that is difficult to learn. Preconditioning the flow with techniques akin to \citet{karras2022elucidating} often degraded performance rather than fixing this. Fortunately, this is precisely the regime where the MAExL steps (\Cref{app:maexl}) are most accurate, compensating for the weak proposal. 

\Cref{fig:metrics_ablation_seqis_seq_mis} shows that this parity carries over to downstream distributional distances, with the two methods performing on par at matched budgets. The left panel of \Cref{fig:metrics_ablation_seqis_seq_mis} shows that inference time is likewise comparable between the two methods. This can be explained at our current scale, since for molecules this small the autoregressive flow's sampling and log-likelihood costs are already comparable on a modern GPU, leaving little headroom for the sharing scheme in \ourmethod{}-MIS to reduce wall-clock time. We expect this gap to widen in \ourmethod{}-MIS's favor as system size grows and sampling cost increasingly dominates evaluation cost. Finally, \Cref{fig:metrics_ablation_seqis_seq_mis} confirms the expected trend that both methods' metrics improve monotonically as the candidate budget increases.

\subsection{Metropolis-Adjusted Exponential Langevin} \label{app:maexl}

A recurring subroutine in our pipeline is sampling from a posterior whose likelihood is an isotropic Gaussian in the latent variable. This subsection describes the kernel we use for that subroutine: a Metropolis-adjusted exponential integrator of the corresponding Langevin diffusion, designed to remain stable across the full range of likelihood scales encountered during inference.

\paragraph{Target distribution.}
We need to sample from the conditional distribution of the clean sample $x_0$ given a noisy observation $x_t$ produced by the forward process at time $t$. With schedule $(\alpha_t, \sigma_t)$, this is
\begin{equation}
  p_{0 | t}(x_0 | x_t) \propto \mathcal{N}\left(x_t; \alpha_t x_0, \sigma_t^2 I\right) p(x_0),
  \label{eq:maexl-target}
\end{equation}
where $p$ is the data distribution, accessed through its log-density and score $\nabla \log p$. Throughout this subsection $(\alpha_t, \sigma_t)$ are fixed scalars.

\paragraph{Associated Langevin SDE.}
The score of \eqref{eq:maexl-target} decomposes as
\begin{equation}
  \nabla_{x_0} \log p_{0|t}(x_0 | x_t) = \nabla \log p(x_0) - A x_0 + b,
  \qquad A := \frac{\alpha_t^2}{\sigma_t^2}, \quad b := \frac{\alpha_t}{\sigma_t^2} x_t,
\end{equation}
and the overdamped Langevin diffusion targeting \eqref{eq:maexl-target} reads
\begin{equation}
  dX_s = \left(-A X_s + b + \nabla \log p(X_s)\right) ds + \sqrt{2} dW_s.
  \label{eq:maexl-sde}
\end{equation}
The drift has an \emph{affine} part ($-A X_s + b$) coming from the Gaussian likelihood and a nonlinear part coming from the prior score.

\paragraph{Exponential integration.}
We integrate \eqref{eq:maexl-sde} by treating the affine part exactly and freezing the prior score over each step. Variation of constants on $Y_s = \exp(A_s) X_s$ gives, on a step of size $h$,
\begin{equation}
  X_{s+h} = e^{-Ah} X_s + \int_0^h e^{-A(h-u)}\left(b + \nabla \log p(X_{s+u})\right) du + \sqrt{2} \int_0^h e^{-A(h-u)} dW_u.
\end{equation}
Freezing the score at $X_s$ produces the (exponential Euler) one-step kernel
\begin{equation}
  X_{s+h} \approx e^{-Ah} X_s + \frac{1 - e^{-Ah}}{A}\left(b + \nabla \log p(X_s)\right) + \sqrt{\tfrac{1 - e^{-2Ah}}{A}}\xi, \qquad \xi \sim \mathcal{N}(0, I).
  \label{eq:maexl-expeuler}
\end{equation}
Compared to plain Euler--Maruyama, this scheme (i) preserves the Ornstein--Uhlenbeck invariant law of the linear sub-problem at \emph{every} step size, so its noise variance is correct rather than $\mathcal{O}(h)$-biased, and (ii) is unconditionally stable in the affine part: there is no $hA \lesssim 1$ restriction. This is critical when the likelihood is informative (small $\sigma_t$), where Euler--Maruyama is forced to take vanishingly small steps.

\paragraph{Metropolis correction.}
Treating \eqref{eq:maexl-expeuler} as a Markov kernel proposal and accepting with the standard Metropolis--Hastings probability gives an exact sampler of \eqref{eq:maexl-target}. Because the proposal covariance is state-independent (scalar isotropic), the log-determinant terms cancel and the log Metropolis--Hastings ratio reduces to
\begin{equation}
  \log \alpha(x, x') = \log p_{0|t}(x' | x_t) - \log p_{0|t}(x | x_t) + \frac{1}{4\tau}\left(\norm{x' - \mu(x)}^2 - \norm{x - \mu(x')}^2\right),
\end{equation}
where $\mu(x) = e^{-Ah} x + A^{-1} (1 - e^{-Ah})(b + \nabla \log p(x))$ and $\tau$ is defined below. The proposed point is accepted with probability $\min(1, e^{\log \alpha})$. We refer to this kernel as \emph{Metropolis-Adjusted Exponential Langevin} (MAExL).

\paragraph{Step-size reparametrization.}
Rather than tuning $h$ directly, we reparametrize via the noise-to-signal ratio $\tau := (1 - e^{-2Ah})/(2A) \in (0, \tau_{\max})$, with $\tau_{\max} := 1/(2A) = \sigma_t^2/(2\alpha_t^2)$, so that the proposal is expressed entirely through
\begin{equation}
  c(\tau) := \sqrt{1 - 2A\tau}, \qquad \varphi(\tau) := \frac{1 - c(\tau)}{A} = \frac{2\tau}{1 + c(\tau)},
  \label{eq:maexl-coeffs}
\end{equation}
giving $\mu(x) = c(\tau) x + \varphi(\tau)(b + \nabla \log p(x))$ and noise variance $2\tau$. The right-hand expression for $\varphi(\tau)$ is algebraically equivalent but numerically well-conditioned for small $\tau$, avoiding cancellation in $1 - \sqrt{1 - 2A\tau}$.  This reparametrization has three advantages. First, it maps the unbounded, saturating $h \in (0,\infty)$ onto the bounded interval $(0,\tau_{\max})$. Second, $2\tau$ is exactly the proposal variance, and as $\tau \to 0$ the proposal reduces to MALA with step size $\tau$. Third, as $\tau \to \tau_{\max}$, $c(\tau) \to 0$ and the proposal forgets the current state except through the prior score, jumping at the likelihood scale $\sigma_t/\alpha_t$. A single bounded scalar thus moves the kernel from local MALA moves to global, likelihood-scale moves.

$\tau_{\max}$ itself still shrinks by several orders of magnitude as $t \to 0$, so $\tau$ is not the quantity we carry across noise levels. We instead track the induced proposal noise standard deviation $\omega := \sqrt{2\tau} \in (0, \omega_{\max})$, with $\omega_{\max} := \sqrt{2\tau_{\max}} = \sigma_t/\alpha_t$: this is the intended proposal noise scale \emph{in clean $x_0$-coordinates}, and it is $\omega$ that we adapt per chain and report per system,  as a practical parameterization for transferring step sizes across noise levels; the next paragraph clarifies the limits of this interpretation. Because $\omega_{\max}$ depends on $t$, an adapted $\omega$ can legitimately exceed it. Rather than clip $\omega$ itself, we cap only the noise scale actually used by the proposal,
\begin{equation}
  \epsilon := \min\left(\omega, \, \rho_{\max}\,\omega_{\max}\right), \qquad \tau = \frac{\epsilon^2}{2}, \qquad \rho_{\max} = 0.99,
\end{equation}
substituted for $\tau$ in \eqref{eq:maexl-coeffs}, leaving $\omega$ itself unclamped above the cap so that the previously adapted value is reused unchanged once $\omega_{\max}$ grows again at a later (larger) $t$.

\paragraph{ Why $\omega$ transfers across noise levels.}
 The exponential integrator treats the Gaussian-likelihood drift exactly, so the only approximation is freezing the prior score. Concretely,  suppose the prior score is constant , $\nabla \log p \equiv g_0$ for some constant $g_0$ (so $\log p(x) = g_0 \cdot x + \mathrm{const}$). Then $\mu(x) = cx + \varphi(b+g_0)$, and the identities $1 - c^2 = A\epsilon^2$ and $\varphi(1+c) = \epsilon^2$ give
\begin{equation}
  \frac{1}{2\epsilon^2}\left(\norm{x' - \mu(x)}^2 - \norm{x - \mu(x')}^2\right) = \frac{A}{2}\left(\norm{x'}^2 - \norm{x}^2\right) - (b + g_0)\cdot(x' - x),
\end{equation}
which cancels exactly against $\log p_{0|t}(x') - \log p_{0|t}(x) = g_0\cdot(x'-x) - \frac{A}{2}(\norm{x'}^2-\norm{x}^2) + b\cdot(x'-x)$. Hence $\log \alpha(x, x') \equiv 0$ for \emph{every} $A > 0$, every $\tau \in (0, \tau_{\max})$, and every pair $(x, x')$:  whenever the prior score is constant, every proposal is accepted regardless of the likelihood precision.

 Rejection therefore stems from the variation of $\nabla \log p$ over a step, i.e.\ from the curvature of the prior at the scale $\omega$ of the move, rather than from $A_t$ directly. Acceptance is not exactly $t$-independent, since $A_t$ still enters through $c$, $\varphi$ and the visited states, but this makes $\omega$ the natural quantity to carry across noise levels.

\paragraph{Connection to score-matching trade-offs.}
The two ends of the $\tau$-range, global likelihood-scale moves and local MALA moves on the posterior score, mirror the tension between using the likelihood (Tweedie's identity, denoising-score-matching style) and the posterior score (target-score-matching style) as the primary signal; a growing line of work argues that neither extreme is universally best \cite{debortoli2024targetscorematching, kahouli2025controlvariatescorematching, young2026diffusionpathsamplerssequential, blessing2026bridgematchingsamplerscalable}. Our step-size adaptation lets the sampler choose its own balance, per chain and per noise level.

\paragraph{Per-chain step-size adaptation.}
As above, adaptation runs on $\log \omega$ rather than $\log \tau$: step sizes are adapted independently for each chain in a batch by an asymmetric geometric update, so as to favor the high-acceptance regime:
\begin{equation}
  \omega_{m+1} = \max\left(\omega_m \exp\left(\eta_m (\hat\alpha_m - \alpha^\star)\right), \, 10^{-30}\right),
\end{equation}
\begin{equation}
  \eta_m = \begin{cases}
    0 & \text{if } \epsilon_m = \rho_{\max}\omega_{\max} \text{ and } \hat\alpha_m > \alpha^\star \\
    \eta_{\uparrow} & \text{if } \hat\alpha_m > \alpha^\star \\
    \eta_{\downarrow} & \text{if } \hat\alpha_m \le \alpha^\star,
  \end{cases}
\end{equation}
with target acceptance $\alpha^\star = 0.574$ \citep{roberts1998optimal} and gains $\eta_{\uparrow} = 0.1, \eta_{\downarrow} = 0.5$. Here $\hat\alpha_m$ is the per-chain Metropolis acceptance probability $\min(1, e^{\log r})$ at iteration $m$, and $\epsilon_m = \min(\omega_m, \rho_{\max}\omega_{\max})$ is the noise scale actually used by the $m$-th proposal. The asymmetry ($\eta_{\downarrow} > \eta_{\uparrow}$) is deliberate: a stalled chain in the low-acceptance regime is far more harmful for short runs than a slightly conservative one, so we react aggressively to low acceptance and gently to high acceptance. Working in $\log \omega$ keeps $\omega$ positive automatically; rather than clip $\omega$ against the $t$-dependent cap $\rho_{\max}\omega_{\max}$, we freeze upward moves while that cap is already active, since acceptance is then insensitive to $\omega$ and the frozen value is picked back up once $\omega_{\max}$ grows again at a later (larger) $t$. A separate, diminishing-gain Robbins--Monro schedule on $\log \omega$ ($\eta_m = \eta_0/(m+t_0)^\kappa$ with $\eta_0 = 0.05$, $t_0 = 10$, $\kappa = 0.75$) is used only offline, to calibrate the fixed per-system presets of Table~\ref{tab:hparam_nfm_stepsizes}; the geometric schedule above is what runs in the loop.  The analysis of \Cref{app:sec:proofs} applies to this kernel with the step sizes frozen.

\paragraph{Calibrated step sizes.}
Table~\ref{tab:hparam_nfm_stepsizes} reports, per system, the in-loop MAExL step size $\omega$, together with the step size $\tau$ used by a separate post-hoc MALA pass that polishes a finished sample set against the clean prior $p(x_0)$ directly, with no likelihood term. Both step sizes were obtained from very long calibration runs (128 chains, a short geometric burn-in followed by several thousand steps of the Robbins--Monro schedule above, at target acceptance $\alpha^\star = 0.574$) and are then held fixed as the shared starting point across all subsequent runs on that system.

\begin{table}[t]
    \centering
    \caption{Per-system calibrated step sizes (Robbins--Monro, target acceptance 0.574).}
    \label{tab:hparam_nfm_stepsizes}
    \small
    \begin{tabular}{lccccc}
        \toprule
         & ALA-2 & ALA-3 & ALA-4 & ALA-6 & Chignolin \\
        \midrule
        MAExL $\omega$ (in-loop)      & $1.02{\times}10^{-2}$ & $7.48{\times}10^{-3}$ & $6.55{\times}10^{-3}$ & $4.68{\times}10^{-3}$ & $3.46{\times}10^{-3}$ \\
        MALA $\tau$ (post-hoc)        & $5.27{\times}10^{-5}$ & $2.75{\times}10^{-5}$ & $2.14{\times}10^{-5}$ & $1.10{\times}10^{-5}$ & $6.01{\times}10^{-6}$ \\
        \bottomrule
    \end{tabular}
\end{table}

\paragraph{Connection with EI-MALA \citep{Durmus2015quantitative}.}
The MAExL kernel is closely related to the exponential-integrator version of MALA analyzed in \cite{Durmus2015quantitative}. In that work, the authors study targets that are Gaussian reference measures perturbed by a nonlinear potential, and construct a Metropolis-adjusted proposal by integrating the Gaussian linear drift exactly while treating the nonlinear part explicitly. This is precisely the structure exploited here: the Gaussian likelihood contribution to $p_{0|t}(\cdot | x_t)$ induces the affine drift $-Ax+b$, which is integrated exactly, while the prior score $\nabla\log p$ is frozen over one step and corrected by a Metropolis--Hastings accept-reject step. In the notation of \cite{Durmus2015quantitative}, our scalar precision $A I$ plays the role of the Gaussian precision matrix $Q$, and the reparametrization $\tau=(1-e^{-2Ah})/(2A)$ yields the same proposal covariance as their exponential-integrator MALA. Their Wasserstein convergence analysis is therefore directly aligned with the role of MAExL in our method: it provides quantitative control of the contraction and mixing of the $p_{0|t}$-invariant Markov kernel, which can be combined with the bridge sensitivity factor $\gamma_{t,s}$ to sharpen the propagated SNIS error bounds.

\begin{corollary}[Quantitative Wasserstein control of MAExL]\label{cor:maexl_durmus_wasserstein}
    Fix $t\in(0,T]$ and $x_t\in\mathbb{R}^d$, and write
    \begin{align*}
        \pi_t(\rmd x_0) = p_{0|t}(x_0 | x_t)\rmd x_0.
    \end{align*}
    Let $R_t$ denote the MAExL Markov kernel targeting $\pi_t$.
    Set
    \begin{align*}
        A_t= \frac{\alpha_t^2}{\sigma_t^2}, \qquad  \bar m_t(x_t)=A_t^{-1}b_t=\frac{x_t}{\alpha_t}, \qquad Q_t=A_t I_d.
    \end{align*}
    In the centered variable $z = x_0 - \bar m_t(x_t)$, assume that the target can be written as
    \begin{align*}
        \pi_t(\rmd z) \propto \exp\left(-\frac{1}{2} z^\top Q_t z -\Gamma_t(z)\right)\rmd z,
    \end{align*}
    where $\Gamma_t(z)=-\log p(z+\bar m_t(x_t))$, or more generally where the prior contribution is split as in the EI-MALA framework of \cite{Durmus2015quantitative}. Assume that the corresponding MAExL/EI-MALA proposal satisfies assumptions M1--M2 of \cite{Durmus2015quantitative} for the step size used by $R_t$. Then there exist constants
    \begin{align*}
        \eta_t>0,\qquad \varepsilon_t>0,\qquad C_t<\infty,
        \qquad r_t\in(0,1),
    \end{align*}
    and a bounded distance $d_t=d_{\eta_t,\varepsilon_t}$, topologically equivalent to the $Q_t$-norm, such that, for all $L\geq 0$ and all $x_0,y_0\in\mathbb{R}^d$,
    \begin{align}
        W_{d_t}\left(R_t^L(x_0,\cdot), R_t^L(y_0,\cdot)\right) &\leq C_t r_t^L \left\{V_t(x_0)+V_t(y_0)\right\}, \label{eq:maexl_pairwise_wasserstein_contraction} \\
        W_{d_t}\left(R_t^L(x_0,\cdot), \pi_t\right) &\leq C_t r_t^L V_t(x_0), \label{eq:maexl_wasserstein_mixing}
    \end{align}
    where
    \begin{align*}
        V_t(x_0) = 1 \vee \norm{x_0-\bar m_t(x_t)}_{Q_t} = 1 \vee \sqrt{A_t}\norm{x_0-\bar m_t(x_t)}.
    \end{align*}
    Consequently, for any measurable set $B\subset\mathbb{R}^d$, the local Wasserstein diameter of the $L$-step MAExL kernel satisfies
    \begin{align*}
        \Delta^{(d_t)}_{t,L}(B) := \sup_{x_0,y_0\in B} W_{d_t}\left(R_t^L(x_0,\cdot), R_t^L(y_0,\cdot)\right) \leq 2C_t r_t^L \sup_{x_0\in B}V_t(x_0).
    \end{align*}
    In particular, for
    \begin{align*}
        B_u = \mathcal{B}\left(m_t(x_t), \kappa_t(x_t)(\sqrt d+u)\right),
    \end{align*}
    one has
    \begin{align*}
        \Delta^{(d_t)}_{t,L}(B_u) \leq 2C_t r_t^L \left[1 \vee \sqrt{A_t}\left(\norm{m_t(x_t)-\bar m_t(x_t)} + \kappa_t(x_t)(\sqrt d+u)\right)\right].
    \end{align*}
\end{corollary}

\begin{proof}
    In the centered variable $z=x_0-\bar m_t(x_t)$, the MAExL proposal coincides with the exponential-integrator MALA proposal applied to a Gaussian reference measure with precision $Q_t=A_tI_d$, perturbed by the prior potential. Under assumptions M1--M2 of \cite{Durmus2015quantitative}, their Theorem~2 gives geometric convergence of the EI-MALA kernel in a Wasserstein distance $W_{d_t}$ associated with a bounded distance $d_t=d_{\eta_t,\varepsilon_t}$. This yields \eqref{eq:maexl_wasserstein_mixing}. Their Proposition~3 gives the corresponding pairwise contraction estimate, which yields \eqref{eq:maexl_pairwise_wasserstein_contraction}. Taking the supremum over $x_0,y_0\in B$ gives the bound on $\Delta^{(d_t)}_{t,L}(B)$. The final display follows by bounding $V_t$ on $B_u$.
\end{proof}

 \paragraph{Impact of MAExL.}

\begin{table*}[t]
    \centering
    \caption{MAExL ablation on ALA-$N$ systems, with $N=2,3,4,6$, and on chignolin. Each column is the \ourmethod{} configuration of \Cref{tab:hparam_nfm_selected}, rerun with the number of in-loop MAExL steps per level given in the header; the 512-step columns are those of \Cref{tab:wasserstein}. Results are obtained from three random seeds, reported as \texttt{mean} $\pm$ \texttt{std}. Energy and torus Wasserstein distances are computed on the test split against $10^4$ reference samples, from a proposal budget of $10^4$ samples. Inference time covers sampling and the post-hoc MALA ladder.}
    \scriptsize
    \label{tab:maexl_ablation}

    \renewcommand{\arraystretch}{1.5}
    \setlength{\tabcolsep}{1.5pt}

    \resizebox{\linewidth}{!}{%
    \begin{tabular}{c l c c c c c c}
        \cmidrule[0.8pt]{1-8}
        & & \multicolumn{2}{c}{{\footnotesize 1024 MAExL steps}}
          & \multicolumn{2}{c}{{\footnotesize 512 MAExL steps}}
          & \multicolumn{2}{c}{{\footnotesize 0 MAExL steps}} \\
        \cmidrule(lr){3-4} \cmidrule(lr){5-6} \cmidrule(lr){7-8}
        & & {\tiny \ourmethod{}-IS} & {\tiny \ourmethod{}-MIS}
          & {\tiny \ourmethod{}-IS} & {\tiny \ourmethod{}-MIS}
          & {\tiny \ourmethod{}-IS} & {\tiny \ourmethod{}-MIS} \\
        \cmidrule[0.5pt]{1-8}

        \multirow{3}{*}{\begin{tabular}[c]{@{}c@{}}\textbf{ALA-2} \\ ($d=66$)\end{tabular}}
        & $\mathcal{E}$-$\mathcal{W}_2$ $\downarrow$ & $0.127 \err{0.019}$ & $0.186 \err{0.042}$ & $0.132 \err{0.013}$ & $0.242 \err{0.080}$ & $0.256 \err{0.054}$ & $0.367 \err{0.035}$ \\
        & $\mathbb{T}$-$\mathcal{W}_2$ $\downarrow$ & $1.044 \err{0.019}$ & $1.241 \err{0.010}$ & $1.050 \err{0.011}$ & $1.239 \err{0.006}$ & $1.055 \err{0.009}$ & $1.257 \err{0.023}$ \\
        & Inference (s) $\downarrow$ & $129.2 \err{5.8}$ & $118.7 \err{0.2}$ & $137.2 \err{14.5}$ & $125.8 \err{14.4}$ & $122.8 \err{14.4}$ & $116.6 \err{14.6}$ \\
        \cmidrule{1-8}

        \multirow{3}{*}{\begin{tabular}[c]{@{}c@{}}\textbf{ALA-3} \\ ($d=99$)\end{tabular}}
        & $\mathcal{E}$-$\mathcal{W}_2$ $\downarrow$ & $0.167 \err{0.016}$ & $0.151 \err{0.030}$ & $0.164 \err{0.011}$ & $0.157 \err{0.009}$ & $0.215 \err{0.037}$ & $0.212 \err{0.033}$ \\
        & $\mathbb{T}$-$\mathcal{W}_2$ $\downarrow$ & $0.497 \err{0.011}$ & $0.471 \err{0.020}$ & $0.490 \err{0.016}$ & $0.465 \err{0.007}$ & $0.505 \err{0.022}$ & $0.462 \err{0.020}$ \\
        & Inference (s) $\downarrow$ & $162.8 \err{9.8}$ & $163.2 \err{10.7}$ & $164.1 \err{0.4}$ & $147.2 \err{13.1}$ & $137.2 \err{12.0}$ & $149.7 \err{0.8}$ \\
        \cmidrule{1-8}

        \multirow{3}{*}{\begin{tabular}[c]{@{}c@{}}\textbf{ALA-4} \\ ($d=126$)\end{tabular}}
        & $\mathcal{E}$-$\mathcal{W}_2$ $\downarrow$ & $1.022 \err{0.073}$ & $0.970 \err{0.100}$ & $0.988 \err{0.027}$ & $1.107 \err{0.013}$ & $7.125 \err{11.571}$ & $1.745 \err{2.142}$ \\
        & $\mathbb{T}$-$\mathcal{W}_2$ $\downarrow$ & $0.827 \err{0.011}$ & $0.850 \err{0.063}$ & $0.824 \err{0.041}$ & $0.882 \err{0.051}$ & $0.840 \err{0.055}$ & $0.882 \err{0.024}$ \\
        & Inference (s) $\downarrow$ & $193.7 \err{12.3}$ & $174.8 \err{12.6}$ & $184.3 \err{13.5}$ & $162.3 \err{13.6}$ & $159.4 \err{14.0}$ & $148.0 \err{12.5}$ \\
        \cmidrule{1-8}

        \multirow{3}{*}{\begin{tabular}[c]{@{}c@{}}\textbf{ALA-6} \\ ($d=189$)\end{tabular}}
        & $\mathcal{E}$-$\mathcal{W}_2$ $\downarrow$ & $1.070 \err{0.137}$ & $1.084 \err{0.123}$ & $1.018 \err{0.063}$ & $1.301 \err{0.092}$ & $2.077 \err{0.196}$ & $2.649 \err{0.109}$ \\
        & $\mathbb{T}$-$\mathcal{W}_2$ $\downarrow$ & $1.013 \err{0.023}$ & $1.019 \err{0.016}$ & $0.990 \err{0.026}$ & $1.027 \err{0.021}$ & $0.994 \err{0.018}$ & $1.016 \err{0.004}$ \\
        & Inference (s) $\downarrow$ & $455.3 \err{1.5}$ & $440.3 \err{13.4}$ & $452.8 \err{13.9}$ & $405.6 \err{14.2}$ & $417.3 \err{5.2}$ & $372.8 \err{9.3}$ \\
        \cmidrule{1-8}

        \multirow{3}{*}{\begin{tabular}[c]{@{}c@{}}\textbf{Chignolin} \\ ($d=414$)\end{tabular}}
        & $\mathcal{E}$-$\mathcal{W}_2$ $\downarrow$ & $4.966 \err{0.369}$ & $5.734 \err{0.091}$ & $4.331 \err{0.262}$ & $6.112 \err{0.129}$ & $10.155 \err{0.502}$ & $9.129 \err{0.077}$ \\
        & $\mathbb{T}$-$\mathcal{W}_2$ $\downarrow$ & $2.588 \err{0.018}$ & $2.579 \err{0.023}$ & $2.597 \err{0.054}$ & $2.581 \err{0.018}$ & $2.594 \err{0.045}$ & $2.573 \err{0.017}$ \\
        & Inference (s) $\downarrow$ & $2945.3 \err{13.1}$ & $2356.3 \err{16.1}$ & $2830.6 \err{17.4}$ & $2264.5 \err{1.1}$ & $2727.2 \err{9.7}$ & $2181.8 \err{14.4}$ \\
        \cmidrule[0.8pt]{1-8}
    \end{tabular}%
    }
\end{table*}

\Cref{tab:maexl_ablation} isolates the contribution of the in-loop MAExL steps by rerunning the selected \ourmethod{} configurations with 1024, 512 and 0 steps per noise level. The effect is concentrated almost entirely in the energy metric. Disabling MAExL degrades $\mathcal{E}$-$\mathcal{W}_2$ on every system and for both weighting schemes, in most cases by a factor of two or more, and on ALA-4 it is what separates a usable estimator from an unreliable one: without MAExL the energy discrepancy grows by an order of magnitude and its across-seed standard deviation exceeds its mean, i.e. individual seeds fail outright, while the same configuration with MAExL steps is both accurate and tightly concentrated. MAExL therefore buys run-to-run reproducibility as much as average accuracy. By contrast, $\mathbb{T}$-$\mathcal{W}_2$ is essentially flat across the three settings on all five systems, moving by at most a few percent and generally within one seed standard deviation. This is what one expects from a local, $p_{0|t}$-invariant kernel: MAExL relaxes energetic strain in the proposed $x_0$ configurations, the clashes and distorted bonds that dominate the energy distribution, but its moves are too local to transfer probability mass between torsional metastable states, so the dihedral marginals remain governed by the coverage of the learned model. The price is modest and shrinks in relative terms with system size, since on the larger systems the flow evaluations and the post-hoc MALA ladder dominate the inference budget. Returns also saturate well before the largest setting: 512 and 1024 steps are statistically indistinguishable for \ourmethod{}-IS and nominally favour the shorter chains on the three largest systems, with only \ourmethod{}-MIS showing a consistent preference for 1024 steps.

\subsection{Centre-of-mass adjustment for the conditional posterior}
\label{app:com-adjustment-conditional}

We follow the Gaussian centre-of-mass (CoM) augmentation of SBG \citep{tan2025scalable} and extend it to our diffusion posterior. Unlike SBG, we do not use a radial correction: every weight below is built from the Cartesian density of the centroid. We first derive valid importance weights and then, separately, characterise when matching the centroid proposal improves the effective sample size (ESS).

\subsubsection{Setup and notation}

For $x=(x_1,\ldots,x_N)\in\mathbb{R}^{Nd}$, let $c=N^{-1}\sum_i x_i$ be the unweighted centroid, which is what we mean by the CoM throughout, and let $\hat{x}=x-\mathbf{1}_N\otimes c$ be the centred configuration. Then
\begin{align*}
    x=\hat{x}+\mathbf{1}_N\otimes c,
    \qquad
    \hat{x}\in\mathcal{H}:=\Bigl\{s\in\mathbb{R}^{Nd}:\textstyle\sum_i s_i=0\Bigr\}.
\end{align*}
The two components are orthogonal, and $\dim\mathcal{H}=(N-1)d$. All densities below are written in the coordinates $(\hat{x},c)$, with respect to Lebesgue measure on $\mathcal{H}\times\mathbb{R}^d$. The change of variables to Cartesian coordinates has a constant Jacobian, which we include whenever the flow density is expressed in these coordinates.

Let $p$ be the normalised Boltzmann density on $\mathcal{H}$. Restricting to $\mathcal{H}$ removes the translational degree of freedom, which would otherwise make the density unnormalisable. As in SBG, we lift $p$ to the ambient space by attaching an independent Gaussian centroid:
\begin{equation}
\label{eq:lift}
    \tilde{p}(\hat{x},c)=p(\hat{x})\,g_\sigma(c),
    \qquad
    g_\sigma(c)=\mathcal{N}(c;0,\sigma^2 I_d),
    \qquad
    \sigma^2=1/N.
\end{equation}
This variance is that of the centroid of $N$ independent standard Gaussian vectors. Since $g_\sigma$ integrates to one, the centred marginal of $\tilde{p}$ is exactly $p$.

\subsubsection{Relation to SBG's unconditional adjustment}

Like SBG, we introduce an auxiliary Gaussian centroid so that the model can be trained in the full ambient space; we differ in the density used for reweighting. For $d=3$, the radius $r=\|c\|$ has density
\begin{align*}
    f_\sigma(r)=\frac{1}{\sigma}\chi_3(r/\sigma)=4\pi r^2 g_\sigma(c),
    \qquad \|c\|=r,
\end{align*}
where the factor $4\pi r^2$ arises from integrating the spherical volume element over the angles. Thus $f_\sigma$ is a density with respect to $dr$, whereas $g_\sigma$ is a density with respect to $dc$.

SBG's Eq.~(5) subtracts $\log f_\sigma(\|c\|)$ from the flow log-density. For the Gaussian lift~\eqref{eq:lift}, however, the correct unconditional importance weight is
\begin{align*}
    w(x)\propto\frac{p(\hat{x})\,g_\sigma(c)}{q^\theta(x)}.
\end{align*}
Rewriting the denominator as $q^\theta/g_\sigma$ amounts to adding $\|c\|^2/(2\sigma^2)$ to the flow log-density, up to an additive constant. We omit SBG's additional $-2\log\|c\|$ term, since it stems from the radial Jacobian rather than from the Cartesian Gaussian density.

\subsubsection{The conditional Gaussian lift}

We consider variance-exploding noising with $\alpha_t=1$ and $\sigma_t>0$:
\begin{align*}
    X_t\mid X_0\sim\mathcal{N}(X_0,\sigma_t^2 I_{Nd}).
\end{align*}
The conditional flow $q^\theta_{0|t}$ is trained on pairs $(x_0,x_t)$ generated by first augmenting with a Gaussian centroid and then applying this noising. The posterior over centred configurations is
\begin{equation}
\label{eq:centred-posterior}
    p^\perp_{0|t}(\hat{x}_0\mid\hat{x}_t)
    \propto p(\hat{x}_0)\exp\!\left(-\frac{\|\hat{x}_t-\hat{x}_0\|^2}{2\sigma_t^2}\right),
\end{equation}
whereas the lifted model has posterior
\begin{equation}
\label{eq:lifted-posterior}
    \tilde{p}_{0|t}(x_0\mid x_t)
    \propto \mathcal{N}(x_t;x_0,\sigma_t^2 I_{Nd})\,p(\hat{x}_0)\,g_\sigma(c_0).
\end{equation}
Here and below, evaluating a density at $x_0$ is shorthand for evaluating it at $(\hat{x}_0,c_0)$ in the coordinates fixed above.

The squared distance splits orthogonally as
\begin{align*}
    \|x_t-x_0\|^2=\|\hat{x}_t-\hat{x}_0\|^2+N\|c_t-c_0\|^2,
\end{align*}
so the centroid is observed through the Gaussian likelihood
\begin{align*}
    c_t\mid c_0\sim\mathcal{N}(c_0,\sigma_t^2 I_d/N).
\end{align*}
Combining this likelihood with the Gaussian prior $g_\sigma$ yields the factorisation
\begin{equation}
\label{eq:lifted-factorisation}
    \tilde{p}_{0|t}(x_0\mid x_t)
    =p^\perp_{0|t}(\hat{x}_0\mid\hat{x}_t)\,g_t(c_0\mid c_t),
    \qquad
    g_t(c_0\mid c_t)=\mathcal{N}(c_0;\eta_t c_t,\beta_t^2 I_d),
\end{equation}
where
\begin{equation}
\label{eq:alpha-beta}
\begin{aligned}
    \eta_t
    &=\frac{\sigma^2}{\sigma^2+\sigma_t^2/N}
      =\frac{1}{1+\sigma_t^2},\\
    \beta_t^2
    &=\frac{\sigma^2(\sigma_t^2/N)}{\sigma^2+\sigma_t^2/N}
      =\frac{\sigma_t^2}{N(1+\sigma_t^2)},
\end{aligned}
\end{equation}
and the second equalities use $\sigma^2=1/N$. In particular, lifting before conditioning leaves the centred posterior~\eqref{eq:centred-posterior} unchanged.

\paragraph{Rotations.}
If $p$ is rotationally invariant, the lifted posterior is jointly rotation invariant: for every $R\in\mathrm{SO}(d)$, acting identically on each particle,
\begin{equation}
\label{eq:joint-rot-inv}
    \tilde{p}_{0|t}(Rx_0\mid Rx_t)=\tilde{p}_{0|t}(x_0\mid x_t).
\end{equation}
This is invariance under rotating $x_0$ and $x_t$ together, not under rotating $x_0$ alone at fixed $x_t$. Neither the factorisation above nor the weights below require rotational invariance; both follow solely from the independent Gaussian lift and the orthogonal split of the likelihood.

\subsubsection{Importance weights and effective log-likelihood}

Fix $t$ and $x_t$, and draw independent samples from the learned flow $q^\theta_{0|t}(\cdot\,|\,x_t)$, whose support we assume covers that of the lifted posterior. The importance weights are
\begin{equation}
\label{eq:conditional-com-weight}
    w_t(x_0)
    =\frac{p^\perp_{0|t}(\hat{x}_0\mid\hat{x}_t)\,g_t(c_0\mid c_t)}
           {q^\theta_{0|t}(x_0\mid x_t)},
\end{equation}
where normalising constants independent of $x_0$ may be dropped in practice. These weights do not require the learned flow to factorise. Indeed, since $g_t(\cdot\,|\,c_t)$ integrates to one, every integrable centred observable $F$ satisfies
\begin{align*}
    \mathbb{E}_{q^\theta_{0|t}}\bigl[w_t(X_0)\,F(\hat{X}_0)\bigr]
    =\int_{\mathcal{H}}F(s)\,p^\perp_{0|t}(s\mid\hat{x}_t)\,ds.
\end{align*}
Taking $F\equiv1$ gives $\mathbb{E}_{q^\theta_{0|t}}[w_t]=1$, so self-normalised importance sampling is consistent for such observables.

To compare with SBG's log-likelihood adjustment, we define the \emph{effective} log-likelihood
\begin{equation}
\label{eq:adj-flow-loglik}
\begin{aligned}
    \log q^{\theta,\mathrm{eff}}_{0|t}(x_0\mid x_t)
    &:=\log q^\theta_{0|t}(x_0\mid x_t)-\log g_t(c_0\mid c_t)\\
    &=\log q^\theta_{0|t}(x_0\mid x_t)
      +\frac{\|\delta\|^2}{2\beta_t^2}
      +\frac{d}{2}\log(2\pi\beta_t^2),
    \qquad \delta=c_0-\eta_t c_t,
\end{aligned}
\end{equation}
so that $w_t=p^\perp_{0|t}/q^{\theta,\mathrm{eff}}_{0|t}$. This merely rewrites~\eqref{eq:conditional-com-weight}; it is not a new sampling proposal. In general, $q^{\theta,\mathrm{eff}}_{0|t}$ is neither normalised on the ambient space nor equal to the centred marginal of the flow.

Note that the first factor in the numerator of~\eqref{eq:conditional-com-weight} is the \emph{centred} posterior~\eqref{eq:centred-posterior}, not the lifted posterior~\eqref{eq:lifted-posterior}. Equivalently, the entire numerator may be replaced by the lifted posterior, that is, the full likelihood times $p(\hat{x}_0)\,g_\sigma(c_0)$, which already contains the centroid factor. Either way, the Gaussian centroid factor must not be included twice.

\subsubsection{When centroid matching improves ESS}

For $K$ independent proposal samples with weights $w_1,\ldots,w_K$, the effective sample size is
\begin{align*}
    \mathrm{ESS}_K=\frac{\bigl(\sum_{i=1}^K w_i\bigr)^2}{\sum_{i=1}^K w_i^2}.
\end{align*}
For a normalised target $\pi$ and a proposal $q$ whose support covers that of $\pi$, if $w=\pi/q$ has a finite second moment, then
\begin{align*}
    \frac{\mathrm{ESS}_K}{K}
    \xrightarrow[K\to\infty]{\mathrm{a.s.}}
    \rho(q,\pi):=\frac{(\mathbb{E}_q[w])^2}{\mathbb{E}_q[w^2]}.
\end{align*}
Our comparison concerns this asymptotic fractional ESS, not the ESS of any particular finite sample.

\begin{proposition}[Matching the centroid proposal]
\label{prop:cond-ess}
Fix $t$ and $x_t$, and write $P(s)=p^\perp_{0|t}(s\,|\,\hat{x}_t)$ and $g(c)=g_t(c\,|\,c_t)$. Consider the lifted target $\pi(s,c)=P(s)\,g(c)$ and the two proposals
\begin{align*}
    q(s,c)=Q(s)\,k(c),
    \qquad
    q_{\mathrm{match}}(s,c)=Q(s)\,g(c).
\end{align*}
Assume that all densities are normalised, that $Q>0$ wherever $P>0$ and $k>0$ wherever $g>0$ (up to null sets), and that
\begin{align*}
    \mathbb{E}_Q\bigl[(P/Q)^2\bigr]<\infty,
    \qquad
    \mathbb{E}_k\bigl[(g/k)^2\bigr]<\infty.
\end{align*}
Then
\begin{equation}
\label{eq:ess-ineq}
    \rho(q,\pi)\leq\rho(q_{\mathrm{match}},\pi),
\end{equation}
with equality if and only if $k=g$ almost everywhere.
\end{proposition}

\begin{proof}
Under $q$, the weight factorises as $w=A(S)\,B(C)$ with $A=P/Q$ and $B=g/k$. These factors are independent and each has expectation one, so
\begin{align*}
    \rho(q,\pi)
    =\frac{1}{\mathbb{E}_Q[A^2]\,\mathbb{E}_k[B^2]}
    \leq\frac{1}{\mathbb{E}_Q[A^2]}
    =\rho(q_{\mathrm{match}},\pi),
\end{align*}
where the inequality follows from $\mathbb{E}_k[B^2]\geq(\mathbb{E}_k B)^2=1$. Equality holds exactly when $B=1$ $k$-almost surely, that is, when $k=g$ almost everywhere.
\end{proof}

The proof relies on the same independent-weight factorisation as SBG \citep[Appendix~B.1]{tan2025scalable}. It differs from their Proposition~1 in that it compares two normalised proposals against a single fixed, normalised target, under explicit moment assumptions. It is therefore a statement about proposal matching and provides no additional ESS guarantee for the effective log-likelihood.

Several consequences follow. If the flow already factorises as $q^\theta_{0|t}=Q\,g_t$, its ordinary lifted-target weights reduce to $P/Q$, because the centroid factor cancels; if moreover $Q=P$, all weights are constant. For a mismatched factorised proposal, realising the improvement requires actually sampling from $Q\,g_t$ and evaluating its density; modifying the likelihood in the denominator alone is not enough. Finally, for a generic non-factorised flow, the weights~\eqref{eq:conditional-com-weight} remain valid, but the ESS comparison no longer applies.

\paragraph{Limits.}
For fixed $x_t$, as $\sigma_t\to\infty$ we have $\eta_t\to0$ and $\beta_t^2\to\sigma^2$, which recovers the unconditional \emph{Gaussian} adjustment above rather than SBG's radial formula. As $\sigma_t\to0$, we have $\eta_t\to1$ and $\beta_t^2\to0$, so the centroid posterior concentrates at $c_t$. At zero noise it becomes singular, which is why the density formulas are stated only for $\sigma_t>0$. Under $g_t$, writing $\delta=\beta_t Z$ with $Z\sim\mathcal{N}(0,I_d)$ gives
\begin{align*}
    -\log g_t(c_0\mid c_t)=\frac{\|Z\|^2}{2}+\frac{d}{2}\log(2\pi\beta_t^2).
\end{align*}
The second term is constant in $x_0$ and cancels under self-normalisation, but the sample-dependent first term does not vanish at low noise. It does, however, cancel exactly in the importance ratio at every positive noise level whenever the proposal centroid matches $g_t$.

\subsection{Conditional TarFlow and StarFlow}
\label{app:molecular-tarflow}

\Cref{subsec:method:training} introduces our conditional molecular TarFlow and StarFlow backbones and the two mechanisms used to inject $x_t$. This appendix gives the detail that does not fit there: a brief recap of the two base architectures (\Cref{app:tarflow-starflow-background}, pointing to their own papers for anything not needed below), exactly how time and space conditioning are added on top of them (\Cref{app:conditional-mechanism}), the cache that makes conditioning nearly free at sampling time (\Cref{app:conditioning-cache}), and the resulting time and space complexity of every operation our sampler uses (\Cref{app:tarflow-complexity}).

\subsubsection{TarFlow and StarFlow}
\label{app:tarflow-starflow-background}

\paragraph{TarFlow.} TarFlow \citep{zhai2025normalizingflowscapablegenerative} is an autoregressive normalising flow built from causal-Transformer coupling layers. Given tokens $x=(x_1,\ldots,x_N)$, $x_i\in\mathbb{R}^D$, one layer computes
\begin{equation}
    \label{eq:tarflow-coupling}
    z_i = \bigl(x_i-\mu_i(x_{<i})\bigr)\odot\exp\bigl(-\alpha_i(x_{<i})\bigr), \qquad i=1,\ldots,N,
\end{equation}
with $z_1\equiv x_1$ and $(\mu_i,\alpha_i)$ produced by a single causal-Transformer pass over the prefix $x_{<i}$; the Jacobian is lower-triangular, with $\log|\det\partial z/\partial x|=-\sum_i\mathbf{1}^\top\alpha_i(x_{<i})$. Training computes every $(\mu_i,\alpha_i)$ in one parallel pass under a causal mask; sampling reruns the same Transformer autoregressively, with a standard KV cache. A deep flow stacks $L$ such layers, permuting the sequence between layers so every coordinate sits both early and late in the autoregressive order across the stack. We refer to \citet{zhai2025normalizingflowscapablegenerative} for the full architecture, including the image-patch tokenisation, which does not affect anything below.

\paragraph{StarFlow.} StarFlow \citep{gu2026starflow} scales the same recipe with two changes we adopt. First, \emph{deep--shallow depth allocation}: rather than $L$ equal layers, one \emph{deep} block of $\ell$ layers generates a first sample from noise and $L-1$ \emph{shallow} blocks of two layers each refine it, written $\ell(L)$-$d$ for width $d$. Second, a \emph{guarded affine head}: TarFlow's $\exp(\alpha_i)$ overflows once $|\alpha_i|\approx88$, so StarFlow soft-clips the pre-activation and passes it through a softplus instead,
\begin{equation}
    \label{eq:starflow-affine}
    \tilde\alpha_i = s\tanh(\alpha_i/s), \qquad \sigma_i=\mathrm{softplus}\bigl(\tilde\alpha_i+\mathrm{softplus}^{-1}(1)\bigr), \qquad z_i=(x_i-\mu_i)/\sigma_i,
\end{equation}
with $s=4$, which we use throughout, including for conditional TarFlow. StarFlow's remaining two contributions (a pretrained-autoencoder latent space and a classifier-free guidance rule) are specific to text-to-image generation: an autoencoder would break the exactness of $\log p_\theta$ that our importance sampler needs, and there is no class or text label to guide on. We do not adopt either; see \citet{gu2026starflow} for both.

\paragraph{Molecular adaptation.} Following \citet{tan2025scalable}, we tokenise an $N$-atom configuration $x\in\mathbb{R}^{N\times3}$ as one token per atom, so StarFlow's 2-D patch grid becomes a degenerate $1\times N$ image, and handle the target's $\mathrm{SE}(3)$ symmetry outside the architecture through centering, a lifted centroid, and random rotation augmentation rather than through an equivariant backbone, \Cref{subsec:method:training} gives the details we share with the main text. Our molecular blocks use learned per-block positional embeddings rather than rotary ones, a choice \Cref{app:conditioning-cache} returns to, full multi-head attention, GELU feed-forwards, and the guarded affine head of \Cref{eq:starflow-affine}. Per-system depth and width are given in \Cref{sec:hyperparameters}.

\subsubsection{Making the flow conditional on $(x_t,t)$}
\label{app:conditional-mechanism}

Both backbones are made \emph{diffusion-ready} by turning the unconditional density $p_\theta(x_0)$ into the posterior $p^{\theta}_{0|t}(x_0|x_t)$ of \Cref{subsec:method:training}: a small continuous $t$ and an entire noisy configuration $x_t\in\mathbb{R}^{N\times3}$ are injected into every block without disturbing the autoregressive structure or the KV-cached sampling cost. \Cref{fig:conditional-pipeline} gives the resulting picture before the two paragraphs below detail the conditioning pathways themselves, and \Cref{fig:conditioning-mechanism} zooms into a single block.

\begin{figure}[t]
    \centering
    \includegraphics[width=\linewidth]{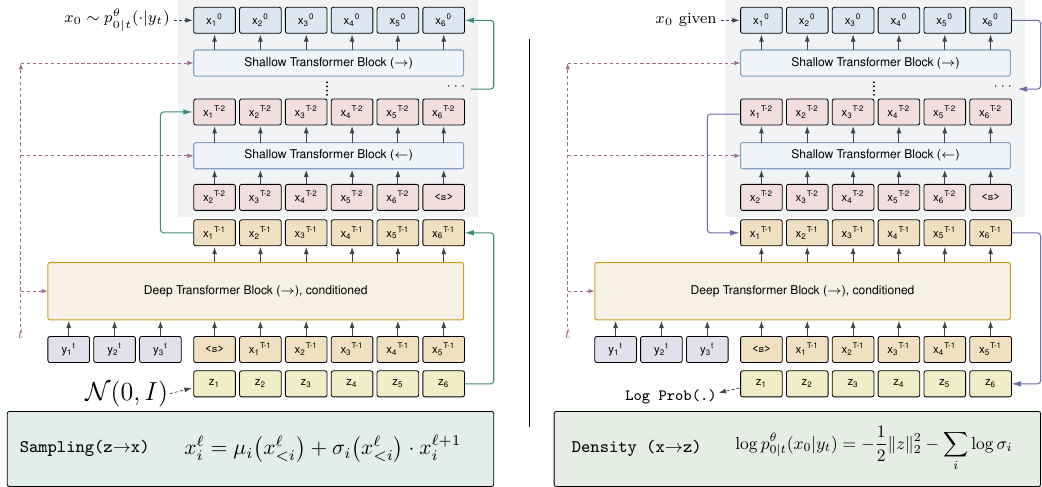}
    \caption{\textbf{Sampling and density evaluation share one stack.} The same $L$ blocks (a deep block conditioned on $y_t$, shallow blocks refining) run in two directions. \emph{(a)} Autoregressively, from $z\sim\mathcal{N}(0,\mathrm{I})$ up to a sample $x_0\sim p^{\theta}_{0|t}(\cdot|y_t)$: $N$ sequential per-atom steps per block, under a KV cache. \emph{(b)} In parallel, from a given $x_0$ down to $\log p^{\theta}_{0|t}(x_0|y_t)$: one causal pass per block. Here $y_t$ denotes the noisy configuration $x_t$ of \Cref{subsec:method:training}, relabelled so the conditioning is never mistaken for another stage of the flow's own $x^0,\ldots,x^{T-1}$ sequence. The green/purple arrows are the inverse/forward instance of the same per-atom update between every pair of adjacent levels, $x^{\ell}_i=\mu_i(x^{\ell}_{<i})+\sigma_i(x^{\ell}_{<i})\cdot x^{\ell+1}_i$ (bottom left),  $x^{\ell+1}_i$ is literally $z_i$ only at the very first level. $y_t$ is shown entering as a causal prefix at the deep block only. \Cref{fig:conditioning-mechanism} shows the cross-attention alternative and the every-block scope TarFlow uses instead (\Cref{app:conditional-mechanism}), plus the AdaLN-Zero time conditioning we add on top of both shown here only as the dashed $t$ arrows reaching every block, deep and shallow alike (the modulation itself is not drawn). Both directions read the same conditioning cache, built once from $(y_t,t)$ and shared across every particle drawn at that step (\Cref{app:conditioning-cache}). Layout follows \citet{gu2026starflow}'s own figure.}
    \label{fig:conditional-pipeline}
\end{figure}

\paragraph{Time, via AdaLN-Zero.} We follow DiT \citep{peebles2023scalable}: $t$ is embedded by a sinusoidal map and a two-layer MLP into a single vector $c\in\mathbb{R}^C$, which drives an AdaLN-Zero modulation in every attention layer of every block. A $\mathrm{SiLU}\to\mathrm{Linear}(C,9C)$ produces a (shift, scale, gate) triple for each of the self-attention, cross-attention (when present) and MLP sub-layers, and every one of these modulation linears is zero-initialised. The unconditional flow already starts at the identity because its affine head is zero-initialised, while attention and the MLP still run live. The conditional flow starts at a strictly stronger identity, since every gate is zero too, so attention, cross-attention and the MLP contribute nothing at all, on top of the zero affine head.

\paragraph{Space, via cross-attention or a causal prefix.} $x_t$ enters through one of two mechanisms, both shown in \Cref{fig:conditioning-mechanism}. \emph{Cross-attention} projects $x_t$ once per block through a shared $\mathbb{R}^3\to\mathbb{R}^C$ map, adds the same positional code used for the running $x_0$ representation and lets each layer, every layer or a strided subset every $m$-th one, project its own keys and values from that shared embedding and attend to it from a query built out of $x_0$. \emph{Prefix conditioning} instead concatenates the projected, positionally-coded $x_t$ ahead of $x_0$ along the sequence axis and lets the block's existing causal self-attention do the work: writing $\pi$ for the block's permutation and $W_{\mathrm{in}},W_{\mathrm{cond}}$ for the two input projections,
\begin{equation}
    \label{eq:prefix-conditioning}
    z = \bigl[\,W_{\mathrm{cond}}x_t^\pi+p^\pi+c\mathbf{1}^\top\,;\,W_{\mathrm{in}}x_0^\pi+p^\pi\,\bigr]\in\mathbb{R}^{2N\times C},
\end{equation}
reading $(\mu_i,\tilde\alpha_i)$ off the last $N$ positions only. Because attention is causal and the prefix leads, the first $N$ positions never depend on $x_0$, so $x_0\mapsto z$ is still autoregressive with the same lower-triangular Jacobian. Every conditioning atom is simply visible to every query atom through the same attention operator that carries the autoregressive chain, at the cost of attending over $2N$ rather than $N$ tokens and no parameters beyond $W_{\mathrm{cond}}$.

Both mechanisms are available to either backbone, but their scope differs. Conditional TarFlow conditions every block while conditional StarFlow restricts conditioning to the last (deep) block only, while every block still receives $t$ through its own AdaLN-Zero. In practice, cross-attention is what nearly every reported model uses: every TarFlow model, and StarFlow on every system but chignolin. Chignolin,the largest system with the deepest deep block, is the sole case where the prefix variant is selected instead, over its cross-attention twin at matched depth (\Cref{tab:hparam_nfm_arch_sweep,tab:hparam_nfm_selected}).

\begin{figure}[t]
    \centering
    \includegraphics[width=\linewidth]{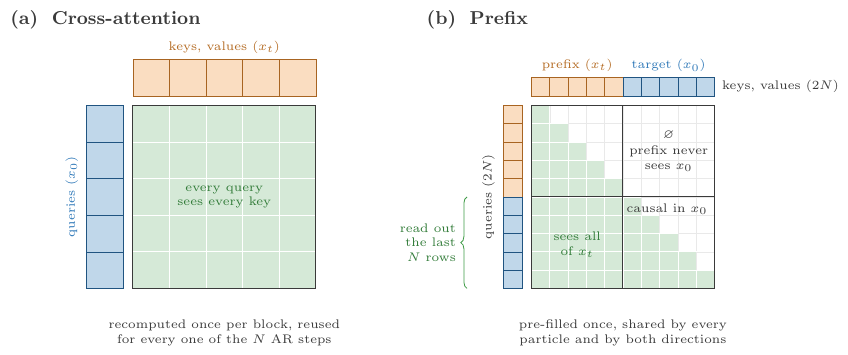}
    \caption{\textbf{The two conditioning mechanisms, inside one block.} \emph{(a) Cross-attention:} $x_t$ is embedded once per block, and each of the $K$ layers still projects its own keys and values from that shared embedding and attends to it from a query built out of $x_0$; every sub-layer also carries an AdaLN-Zero modulation driven by the time embedding $c$, with every modulation linear zero-initialised, so the block is the identity map at init. \emph{(b) Prefix:} $x_t$ instead joins the sequence as a causal prefix (\Cref{eq:prefix-conditioning}). A plain lower-triangular mask over the concatenated $2N$ tokens already gives every target atom access to every conditioning atom while keeping the prefix invisible to $x_0$, so the added cost is a longer attention span and $W_{\mathrm{cond}}$, not a new sub-layer. In both variants the conditioning state is fixed across every particle sharing $x_t$, hence cacheable (\Cref{app:conditioning-cache}).}
    \label{fig:conditioning-mechanism}
\end{figure}

\subsubsection{The conditioning cache}
\label{app:conditioning-cache}

$x_t$ is fixed across every particle drawn from the same posterior: one denoising step of \ourmethod{}-IS/MIS scores and draws $M$ (or $CM$) particles against a single $(x_t,t)$. Recomputing the conditioning pathway from scratch for each of them (the time embedding, the projected $x_t$ tokens, and every conditioned layer's cross-attention or prefix keys and values) would repeat identical work $M$ times over.

Instead, a dedicated preparation step builds this state once per distinct $(x_t,t)$, and every subsequent call to the flow takes it as an argument together with an index selecting each particle's row, so only the row lookup runs per particle. We have verified this reduction bit-for-bit against recomputing the conditioning independently for every particle. For StarFlow under prefix conditioning this goes one step further: the object the preparation step builds is not a separate cache that later gets copied in, but literally the same key/value buffer the autoregressive sampling loop appends to, so pre-filling and decoding share one piece of memory rather than two.

\subsubsection{Computational complexity}
\label{app:tarflow-complexity}
We summarise the per-evaluation cost of the molecular TarFlow and StarFlow variants. Let $N$ be the number of atoms, $C$ the Transformer channel width and $L$ the number of blocks, with block $\ell$ carrying $K_\ell$ attention layers, so that the full flow contains $K_{\mathrm{tot}} = \sum_{\ell=1}^{L} K_\ell$ layers. The TarFlow models set $K_\ell = K$ for every block, giving $K_{\mathrm{tot}} = LK$ and the StarFlow models set $K_\ell = 2$ for the shallow blocks and $K_\ell = K_{\mathrm{deep}}$ for the deep ones. We write $K_{\mathrm{cond}}$ for the number of layers that carry the $x_t$ conditioning and $m$ for the cross-attention period, so that $m = 1$ places cross-attention on every layer. \Cref{tab:tarflow-complexity} gives the cost of a single call into the flow, \Cref{tab:conditioning-amortisation} the cost once the conditioning cache of \Cref{app:conditioning-cache} is in place.

\begin{table}[t]
\centering
\caption{Cost of one call into the flow, dropping constants, with nothing precomputed: every
call rebuilds the conditioning state from $x_t$. Notation as in the text. The \emph{flow
backbone} column is what every variant pays for the flow itself; the other two are what
conditioning adds on top of it.}
\label{tab:tarflow-complexity}
\begin{tabular}{@{}cccc@{}}
\toprule
 & & \multicolumn{2}{c}{Added by conditioning} \\
\cmidrule(lr){3-4}
 & Flow backbone & cross-attention & prefix \\
\midrule
$\log p$                      & $K_{\mathrm{tot}}(N^2 C + N C^2)$              & $\tfrac{1}{m} K_{\mathrm{cond}}(N^2 C + N C^2)$ & $K_{\mathrm{cond}}(3 N^2 C + N C^2)$ \\
\addlinespace
$x\sim p$                                & $K_{\mathrm{tot}}(\tfrac{1}{2} N^2 C + N C^2)$ & $\tfrac{1}{m} K_{\mathrm{cond}}(N^2 C + N C^2)$ & $K_{\mathrm{cond}}(2 N^2 C + N C^2)$ \\
\addlinespace
additional params.                        & ---                                            & $\tfrac{1}{m} K_{\mathrm{cond}} \cdot 4 C^2$    & $0$ \\
\bottomrule
\end{tabular}
\end{table}
One caveat comes first, because it changes which term matters. Every Transformer layer costs $\mathcal{O}(N^2 C + N C^2)$: the attention matrix, plus the $\mathrm{Q}/\mathrm{K}/\mathrm{V}$ and feed-forward linears. At the sizes we run, $N$ between $22$ (alanine dipeptide) and $138$ (chignolin), $C$ between $256$ and $512$, the ratio of the two terms is $C/N$, so the linear term dominates by an order of magnitude on the smallest system and still by roughly $4\times$ on the largest. The quadratic-in-$N$ attention cost is therefore \emph{not} the binding constraint at this scale, and a forward pass is close to linear in the number of atoms.

\paragraph{Forward pass (density evaluation).} All variants compute the entire collection $\{(\mu_i,\alpha_i)\}_{i=1}^{N}$ in a single parallel call with a causal attention mask, at a total cost of $\mathcal{O}(K_{\mathrm{tot}} (N^2 C + N C^2))$. Cross-attention conditioning adds one further attention against $x_t$, also $\mathcal{O}(N^2 C + N C^2)$, on every $m$-th layer of the conditioned blocks, placing it on every layer ($m = 1$) roughly doubles the per-layer constant, and larger $m$ scales that addition by $1/m$. Prefix conditioning instead runs the conditioned block over a sequence of length $2N$, costing $\mathcal{O}(K_{\mathrm{cond}}(4 N^2 C + 2 N C^2))$; once the prefix keys and values have been pre-filled, only the $N$ target positions need to be projected and attended, which reduces this to $\mathcal{O}(K_{\mathrm{cond}}(2 N^2 C + N C^2))$. Neither form of conditioning changes the asymptotics, and prefix conditioning adds no parameters beyond the shared $\mathbb{R}^3 \to \mathbb{R}^C$ projection.

\paragraph{Reverse pass (sampling).} Sampling is inherently sequential along the autoregressive axis: producing token $i+1$ requires the predictions $(\mu_{i+1}, \alpha_{i+1})$, which depend on the prefix $z_{1:i+1}$. With KV caching, each step extends the cached keys and values by one token and computes a query of length one against a prefix of length $i+1$, costing $\mathcal{O}((i{+}1) C + C^2)$ per attention layer. Summing over $i = 0, \ldots, N{-}1$ gives the same total $\mathcal{O}(K_{\mathrm{tot}} (N^2 C + N C^2))$ as the forward pass, but with a substantially larger constant due to the $N$ sequential kernel launches.

Conditioning does not change that order, and neither mode lengthens the autoregressive chain: the conditioning is either attended in parallel with the decode step or pre-filled in a single parallel pass, so a draw is $N$ sequential steps whether or not the flow is conditioned. For cross-attention, the projection of $x_t$ and its keys and values are computed once per block, outside the autoregressive loop, and reused for all $N$ steps; each step then costs $\mathcal{O}(N C + C^2)$, contributing $N^2 C$ per conditioned layer against the $N^2 C / 2$ of the causal loop. For prefix conditioning, a single pre-fill of $\mathcal{O}(K_{\mathrm{cond}}(N^2 C + N C^2))$ populates the cache, after which step $i$ attends over $N + i$ keys rather than $i$; the sum is $3 N^2 / 2$ against $N^2 / 2$, i.e. three times the attention term of the bare causal loop, with the $N C^2$ term unchanged. The real price is memory: the KV cache doubles in length, and its prefix half must be retained for as long as its conditioning row is alive. For the deep block of our chignolin model ($20$ layers, $138$ conditioning tokens, $8$ key/value heads of dimension $64$, in bfloat16) that is roughly $5.5$ MiB per conditioning row, which is why the forward and reverse consumers read one shared store rather than each building its own (\Cref{app:conditioning-cache}).

The two conditioning modes therefore cost the same per decode step, $K_{\mathrm{cond}} N^2 C$ of extra attention, and differ only in what surrounds it: cross-attention pays for it in parameters, prefix conditioning pays for it in a pre-fill and a doubled cache, and the pre-fill is amortised over every particle that shares the conditioning row.

Finally, the deep--shallow allocation reshapes where the sequential cost sits. Each of the $L-1$ shallow blocks contributes only two layers, so the reverse pass is dominated by $K_{\mathrm{deep}}$; at a fixed $K_{\mathrm{tot}}$ the StarFlow allocation concentrates the sequential work in the same block that carries the conditioning, which is what makes the ``condition the deep block only'' choice of \Cref{app:conditional-mechanism} nearly free.

\begin{table}[t]
\centering
\caption{Amortising the conditioning. $x_t$ is fixed across every $x_0$ drawn from the same
posterior, so the conditioning state is built once per distinct $x_t$ and read by every
particle that conditions on it, in both directions of the flow. Entries are the cost
\emph{added} to the flow backbone; notation as in \Cref{tab:tarflow-complexity}. Amortising
deletes the projection terms from the per-particle cost, leaving attention against the cached
state; what remains is the state itself, which scales with the number of conditioning rows
held live rather than with the number of particles.}
\label{tab:conditioning-amortisation}
\small
\setlength{\tabcolsep}{4pt}
\resizebox{\linewidth}{!}{%
\begin{tabular}{@{}lcccc@{}}
\toprule
 & \multicolumn{2}{c}{cross-attention} & \multicolumn{2}{c}{prefix} \\
\cmidrule(lr){2-3} \cmidrule(lr){4-5}
 & uncached & amortised & uncached & amortised \\
\midrule
\multicolumn{5}{@{}l}{\emph{Per particle}} \\
\quad density evaluation & $\tfrac{1}{m} K_{\mathrm{cond}}(N^2 C + N C^2)$ & $\tfrac{1}{m} K_{\mathrm{cond}} N^2 C$ & $K_{\mathrm{cond}}(3 N^2 C + N C^2)$ & $K_{\mathrm{cond}} N^2 C$ \\
\quad sampling           & $\tfrac{1}{m} K_{\mathrm{cond}}(N^2 C + N C^2)$ & $\tfrac{1}{m} K_{\mathrm{cond}} N^2 C$ & $K_{\mathrm{cond}}(2 N^2 C + N C^2)$ & $K_{\mathrm{cond}} N^2 C$ \\
\addlinespace
\multicolumn{5}{@{}l}{\emph{Once per distinct $x_t$}} \\
\quad build              & ---                                             & $\tfrac{1}{m} K_{\mathrm{cond}} N C^2$ & ---                                  & $K_{\mathrm{cond}}(N^2 C + N C^2)$ \\
\quad state retained     & ---                                             & $2 K_{\mathrm{cond}} N C$              & ---                                  & $2 K_{\mathrm{cond}} N C$ \\
\bottomrule
\end{tabular}%
}
\end{table}

\section{Experimental details}
\subsection{Datasets and Metrics}
\label{app:datasets}

\paragraph{Datasets.}
We reuse the molecular dynamics datasets of \cite{tan2025scalable}; we briefly summarise them here for completeness and refer the reader to their Appendix D for full details. All datasets consist of all-atom Cartesian configurations sampled from MD trajectories generated with classical force fields in implicit solvent using OpenMM \citep{eastman2017openmm}. For each system, the first $1\mu s$ of trajectory is used for training, the next $0.2\mu s$ for validation, and the remainder for test, except for alanine dipeptide which follows the split of \cite{klein2024transferable}. Some metastable states may be missing from the training set, which is the realistic regime in which we evaluate our models.

The systems considered, in increasing order of size, are:
\begin{itemize}[leftmargin=*]
    \item \textbf{Alanine dipeptide} (Amber ff99SBildn, $300\mathrm{K}$): a single alanine residue capped by an acetyl group and an N-methyl group. We use the dataset and split of \cite{klein2024transferable}, in which the lower-probability $\varphi$ state is intentionally over-represented in training to provide a tractable but biased setting.
    \item \textbf{Trialanine} (Amber 14, $310\mathrm{K}$): three alanine residues without capping groups, generated as in \cite{klein2023timewarp}.
    \item \textbf{Alanine tetrapeptide} (Amber ff99SBildn, $300\mathrm{K}$): three alanines with acetyl and N-methyl caps. We use the system setup of \cite{dibak2021temperature} but treat all bonds as flexible (the original setup constrained hydrogen bonds for use with internal-coordinate flows).
    \item \textbf{Hexa-alanine} (Amber 14, $310\mathrm{K}$): six alanine residues without capping groups.
    \item \textbf{Chignolin} (Amber 14, $310\mathrm{K}$): the 10-residue protein \texttt{GYDPETGTWG}, $138$ atoms, simulated for $4.5\mu s$. Notable for forming a $\beta$-hairpin in solvent \citep{honda200410}.
\end{itemize}
All MD simulations use a $1\mathrm{fs}$ integration time step. We use the same train/val/test trajectory splits, force fields, and evaluation subsets ($10^4$ uniformly strided samples from the held-out portion) as \cite{tan2025scalable} to ensure direct comparability with their reported numbers.

\paragraph{Target energy evaluation.} While the datasets themselves are generated by OpenMM MD, we do not call OpenMM to evaluate the target energy and its gradient at inference time: an OpenMM \texttt{Context} evaluates one configuration at a time with no batched interface, making it the dominant cost of importance-weighted and MCMC-refined evaluation. We instead reimplement the same potential (bonded, Lennard-Jones/Coulomb, and Generalized-Born terms, OBC1 or OBC2 depending on the system) as a batched PyTorch module whose parameters are read directly from the \texttt{openmm.System}, so OpenMM remains the sole authority on the force field, atom typing, and topology; only the summation, evaluated as vectorized tensor operations with exact autograd gradients in float32, is reimplemented. Validated against OpenMM's double-precision Reference platform on $3.4\times10^{6}$ configurations spanning all splits plus perturbed and clashing geometries, the worst-case deviation below $10^{2}\,k_BT$ is $1.2\times10^{-5}\,k_BT$ in float64 and $7.7\times10^{-4}\,k_BT$ in float32 (more accurate than OpenMM's own CPU platform), whose internal single-precision nonbonded and GB evaluation deviates from the Reference platform by up to $3.3\times10^{-4}\,k_BT$ on typical frames and $126\,k_BT$ near atomic clashes. This batched evaluation is $30$--$400\times$ faster per configuration than OpenMM's serial loop and, unlike OpenMM's hard failure on non-finite energies, returns a large but finite, differentiable energy at clashing configurations, avoiding sampler crashes during exploration.
  
\paragraph{Metrics.}
For computational efficiency we subsample $10^4$ reference configurations from the held-out trajectory to serve as ground truth, and likewise restrict generated samples to a random subset of size $10^4$ when more are produced. We quantify distributional similarity using empirical Wasserstein-2 distances between generated and reference samples. Given two empirical distributions $p = \tfrac{1}{n}\sum_i \delta_{x_i}$ and $q = \tfrac{1}{m}\sum_j \delta_{y_j}$,
\begin{align*}
    W_2(p, q) = \min_{\pi \in \Pi(p, q)} \sqrt{\sum_{i,j} \pi_{ij}  c(x_i, y_j)^2},
\end{align*}
where the optimal coupling is computed with the POT library \citep{flamary2021pot}. We use two cost functions:
\begin{itemize}[leftmargin=*]
    \item \textbf{Energy distance ($\mathcal{E}$-$\mathcal{W}_2$):} the squared cost is $c_{\mathcal{E}}(x, y)^2 = |E(x) - E(y)|^2$ where $E$ is the force-field potential energy. This metric is highly sensitive to fine-grained structural details (bond lengths, angles, clashes).
    \item \textbf{Dihedral torus distance ($\mathbb{T}$-$\mathcal{W}_2$):} for a peptide with $L$ residues we extract the backbone dihedral angles $\mathrm{Dihedrals}(x) = (\varphi_1, \psi_1, \ldots, \varphi_{L-1}, \psi_{L-1})$, and use the wrapped squared cost on the torus
    \begin{align*}
    c_{\mathbb{T}}(x, y)^2 = \sum_{i=1}^{2(L-1)} \big[(\mathrm{Dihedrals}(x)_i - \mathrm{Dihedrals}(y)_i + \pi) \bmod 2\pi - \pi\big]^2.
    \end{align*}
    This captures macrostructural information such as metastable-state occupancy and is robust to fine-grained noise.
\end{itemize}

\subsection{Hyperparameters}
\label{sec:hyperparameters}

\subsubsection{Training}

The unconditional flows (TarFlow/StarFlow) and \ourmethod{} share one training recipe: AdamW ($\beta_1=0.9$, $\beta_2=0.95$) under bf16 mixed precision, gradient-norm clipping at $1.0$, and an EMA of the weights with decay $0.999$. Each model trains for 1000 epochs (195k steps) at batch size 512 over the $10^5$ training frames of the corresponding single-peptide system, on a single GPU, or with standard data-parallel DDP when more than one is used.

\paragraph{Unconditional TarFlow/StarFlow.} Following \cite{tan2025scalable}, the flows are trained by maximum likelihood with learning rate $10^{-4}$ under a one-cycle schedule (5\% warmup, division factor 500, final division factor 1) and weight decay $4\times10^{-4}$. \Cref{tab:hparam_tarflow_arch} gives the per-system TarFlow architecture, and \Cref{tab:hparam_starflow_arch} the alternative StarFlow architecture.

\begin{table}[t]
    \centering
    \caption{Unconditional TarFlow architecture.}
    \label{tab:hparam_tarflow_arch}
    \small
    \begin{tabular}{lccccc}
        \toprule
         & ALA-2 & ALA-3 & ALA-4 & ALA-6 & Chignolin \\
        \midrule
        Channels          & 256 & 256 & 256 & 384 & 512 \\
        Blocks            & 4   & 6   & 6   & 6   & 8   \\
        Layers / block    & 4   & 6   & 6   & 6   & 8   \\
        \bottomrule
    \end{tabular}
\end{table}

\begin{table}[t]
    \centering
    \caption{Unconditional StarFlow architecture. Every variant (here and in \Cref{tab:hparam_nfm_arch_sweep}) follows the pattern layers/block $=[2,2,2,2,n,n]$: the first four blocks are fixed at 2 layers, and only the last two (``top-block depth'' $n$) vary.}
    \label{tab:hparam_starflow_arch}
    \small
    \begin{tabular}{lccccc}
        \toprule
         & ALA-2 & ALA-3 & ALA-4 & ALA-6 & Chignolin \\
        \midrule
        Channels                & 256 & 256 & 256 & 384 & 512 \\
        Blocks                  & 6   & 6   & 6   & 6   & 6   \\
        Head dim                & 64  & 64  & 64  & 64  & 64  \\
        Top-block depth $n$     & 4   & 14  & 14  & 14  & 28  \\
        \bottomrule
    \end{tabular}
\end{table}

\paragraph{\ourmethod{}.} \ourmethod{} trains with the loss of \Cref{subsec:method:training}, under an EDM noise schedule ($\sigma_{\min}=10^{-3}$, $\sigma_{\max}=40$) \citep{karras2022elucidating} and uniform weighting of the loss across noise levels. Every run uses the same optimizer recipe as the unconditional flows: learning rate $10^{-4}$, one-cycle schedule, weight decay $4\times10^{-4}$, 1000 epochs / 195k steps at batch 512.

The conditional backbone is one of six variants, each defined relative to the corresponding unconditional preset (Tables~\ref{tab:hparam_tarflow_arch} and~\ref{tab:hparam_starflow_arch}):
\begin{itemize}
    \item \textbf{T1} -- a shallow TarFlow backbone, conditioned by cross-attention injected every 4th block;
    \item \textbf{K1} -- the same shallow backbone as T1, but with cross-attention injected at every block;
    \item \textbf{W1} -- the same shallow backbone as T1, widened rather than left narrow (channels $+64$);
    \item \textbf{S1} -- a shallow-top-block StarFlow, conditioned through a single prefix token rather than cross-attention;
    \item \textbf{S1k} -- the same shallow-top-block StarFlow as S1, but conditioned by cross-attention rather than a prefix token;
    \item \textbf{S2} -- StarFlow with no depth cut at all, conditioned by cross-attention.
\end{itemize}
Cross-attention stride is $\min(4, l)$, where $l$ is the variant's (top-block) depth, so it falls back from every 4th block to every block on the shallowest variants. The three TarFlow variants (T1, K1, W1) fail to train on chignolin and are not reported there; only S1, S1k, and S2 are available for that system. \Cref{tab:hparam_nfm_arch_sweep} gives the resulting shape of every variant on every system.

For each system, we choose the architecture variant and number of steps ($K=5$ or $K=8$) according to the validation performance of \ourmethod{}-IS. For \ourmethod{}-MIS, selection instead uses \ourmethod{}-IS performance, restricted to $K=8$ since the mixture is only efficient at that setting. \Cref{tab:hparam_nfm_selected} reports both selected configurations, together with the wall-clock training time of the corresponding checkpoint.

\begin{table*}[t]
    \centering
    \caption{\ourmethod{} architecture sweep. TarFlow variants report channels / blocks $\times$ layers-per-block / conditioning stride; StarFlow variants report channels / top-block depth $n$ (layers/block $=[2,2,2,2,n,n]$) / conditioning stride. ``--'' marks variants that failed to train.}
    \label{tab:hparam_nfm_arch_sweep}
    \small
    \begin{tabular}{llccccc}
        \toprule
        Variant & Conditioning & ALA-2 & ALA-3 & ALA-4 & ALA-6 & Chignolin \\
        \midrule
        T1  & cross-attn (stride 3--4) & 256/4$\times$3/3 & 256/6$\times$4/4 & 256/6$\times$4/4 & 384/6$\times$4/4 & -- \\
        K1  & cross-attn (stride 1)    & 256/4$\times$3/1 & 256/6$\times$4/1 & 256/6$\times$4/1 & 384/6$\times$4/1 & -- \\
        W1  & cross-attn (stride 3--4) & 320/4$\times$3/3 & 320/6$\times$4/4 & 320/6$\times$4/4 & 448/6$\times$4/4 & -- \\
        S1  & prefix                   & 256/3            & 256/10           & 256/10           & 384/10           & 512/20 \\
        S1k & cross-attn (stride 3--4) & 256/3/3          & 256/10/4         & 256/10/4         & 384/10/4         & 512/20/4 \\
        S2  & cross-attn (stride 4)    & 256/4/4          & 256/14/4         & 256/14/4         & 384/14/4         & 512/28/4 \\
        \bottomrule
    \end{tabular}
\end{table*}

\begin{table*}[t]
    \centering
    \caption{Selected NFM checkpoint per system, with its architecture (same notation as \Cref{tab:hparam_nfm_arch_sweep}: TarFlow variants as channels / blocks $\times$ layers-per-block / conditioning stride, StarFlow variants as channels / top-block depth / conditioning stride or ``prefix''). $K$ is the number of inference-time levels; training time is the wall-clock of the 1000-epoch run of the checkpoint actually loaded.}
    \label{tab:hparam_nfm_selected}
    \small
    \begin{tabular}{lccccc}
        \toprule
         & ALA-2 & ALA-3 & ALA-4 & ALA-6 & Chignolin \\
        \midrule
        \multicolumn{6}{l}{\ourmethod{}-IS} \\
        Variant             & W1               & W1               & T1               & S2               & S1              \\
        Architecture    & 320/4$\times$3/3 & 320/6$\times$4/4 & 256/6$\times$4/4 & 384/14/4         & 512/20 (prefix) \\
        $K$             & 5   & 8   & 8   & 5    & 5    \\
        Training (h)    & 4.3 & 6.0 & 5.3 & 12.8 & 21.7 \\
        \midrule
        \multicolumn{6}{l}{\ourmethod{}-MIS} \\
        Variant             & W1               & W1               & T1               & S1k              & S1              \\
        Architecture    & 320/4$\times$3/3 & 320/6$\times$4/4 & 256/6$\times$4/4 & 384/10/4         & 512/20 (prefix) \\
        $K$             & 8  & 8  & 8  & 8   & 8  \\
        Training (h)    & 4.3 & 6.0 & 5.3 & 10.2 & 21.7 \\
        \bottomrule
    \end{tabular}
\end{table*}

\paragraph{Velocity teacher.} Before turning to the flow-map baselines (FALCON, FALCON-A, F2D2, SCALLOP), we describe the velocity field they all warm-start from and which we also evaluate as a standalone flow-matching baseline in the style of ECNF++ \citep{tan2025scalable}. A velocity field $v_\phi$ is trained by flow matching on the linear interpolant $x_t = (1-t)\,x_0 + t\,\varepsilon$, with $t \in [10^{-4}, 1-10^{-4}]$: $t=0$ is data and $t=1$ is the prior. The network is the DiT backbone \citep{ying2021graphormer, peebles2023scalable} of \Cref{tab:hparam_dit_arch}. We use the antithetic variance reduction trick for the loss \citep{albergo2025interpolant}. One teacher is trained per system, shared by all four flow-map methods and both of their pretraining objectives (below), and is always read at its EMA weights. Unlike every phase that reads it, the teacher trains against an \emph{ambient} $\mathcal{N}(0, I)$ prior with centre-of-mass augmentation on the data. Its own network removes the mean from input and output, so what it supplies downstream is the mean-free component of the velocity. Its optimizer settings, step count, and per-step cost are given alongside the flow-map phases that follow it, in \Cref{tab:hparam_flowmap_optim} (``Velocity, phase 1''), \Cref{tab:hparam_flowmap_steps}, and \Cref{tab:hparam_flowmap_stepcost} respectively, since its own recorded wall-clock is what the flow-map training budget below is computed from.

\begin{table}[t]
    \centering
    \caption{DiT backbone for the phase-1 velocity teacher.}
    \label{tab:hparam_dit_arch}
    \small
    \begin{tabular}{lccccc}
        \toprule
         & ALA-2 & ALA-3 & ALA-4 & ALA-6 & Chignolin \\
        \midrule
        $d_{\text{model}}$        & 192 & 192 & 192 & 192 & 384 \\
        Layers                    & 6   & 6   & 6   & 6   & 12  \\
        Heads                     & 6   & 6   & 6   & 6   & 12  \\
        MLP ratio                 & 4.0 & 4.0 & 4.0 & 4.0 & 4.0 \\
        Time-embedding dim        & 64  & 64  & 64  & 64  & 128 \\
        \bottomrule
    \end{tabular}
\end{table}

\paragraph{FALCON, FALCON-A, F2D2 and SCALLOP.} We follow the original papers for the architecture and optimizer hyperparameters of these four flow-map baselines (FALCON/FALCON-A: \citet{rehman2025falconfewstepaccuratelikelihoods}; F2D2: \citet{ai2026jointdistillationfastlikelihood}; SCALLOP: \citet{ouyang2026fewstep}). All four continue from the velocity teacher described above, through a two-phase pipeline : flow-map pretraining (Phase 1.5) and the distillation objective itself (Phase 2). Each phase frozen before the next one reads it. The original FALCON paper does not use a pretrain-then-distill procedure; we adopt it uniformly across all four baselines for consistency with F2D2 and SCALLOP, and since it starts from an independently pretrained velocity field, it can only match or improve on FALCON's original recipe.

The flow map itself extends the teacher's time convention: $X(x_t,t,s) = x_t + (s-t)\,u_\theta(x_t,t,s)$ with $s < t$. All three papers run the opposite direction; each identity below is the reversal of the printed one, and each is invariant under it.

\textit{Phase 1.5 --- flow-map pretraining.} One self-consistency identity for the average-velocity map $u_\theta(x_t,t,s)$, trained against the frozen teacher on the same single-head backbone Phase~2 uses. Two objectives, neither of them new mathematics:
\begin{itemize}
    \item \textbf{ESD} (Eulerian), the MeanFlow identity $u_\theta = v + (s-t)\,(v \cdot \nabla_x u_\theta + \partial_t u_\theta)$, which is FALCON's average-velocity term standing alone;
    \item \textbf{LSD} (Lagrangian), $u_\theta(x_t,t,s) + (s-t)\,\partial_s u_\theta(x_t,t,s) = \mathrm{sg}[u_\theta](x_s,s,s)$ at the mapped point $x_s$, which is F2D2's and SCALLOP's shared self-distillation term with no divergence head attached.
\end{itemize}
We encourage to look at \cite{boffi2025buildconsistencymodellearning} for more details. ESD warm-starts FALCON and FALCON-A, LSD warm-starts F2D2 and SCALLOP, so there are two pretraining runs per system rather than four. The student loads the pretraining checkpoint's EMA weights only: fresh optimizer, fresh schedule, step counter back to zero. For F2D2/SCALLOP the single-head parameters are remapped onto the dual-head layout (blocks $\to$ trunk, head $\to$ velocity head) and the divergence head is left at its zero initialization.

\textit{Phase 2 --- distillation.} The four methods proper, on the single-head backbone (FALCON, FALCON-A) or the dual-head backbone with its separate divergence-readout head (F2D2, SCALLOP); \Cref{tab:hparam_flowmap_arch} gives the sizes. \Cref{tab:hparam_flowmap_optim} gives the optimizer settings of all three phases and Tables~\ref{tab:hparam_flowmap_loss} and~\ref{tab:hparam_flowmap_loss_lsd} the per-objective loss hyperparameters. The teacher enters phase 2 in three places: it is the regression target of the instantaneous/diagonal term, the tangent of the average-velocity JVP (FALCON, ESD), and the field whose divergence F2D2 estimates.

\textit{Common to every phase.} The whole problem is posed in the mean-free subspace so the map transports the centroid by the identity. Both times reach the network through a Fourier embedding of \emph{raw} $t$. F2D2, SCALLOP and LSD condition on the anchor and the step $s-t$; FALCON, FALCON-A and ESD on the two times independently, as MeanFlow \citep{geng2025mean} does. The choice has to agree across a warm start, which is why the pretraining rows match the students they initialize. Training runs under bf16 mixed precision but the log-density accumulator is in fp32. Each phase keeps an EMA of decay $0.999$ and skips the optimizer step on a non-finite loss or gradient norm. Chignolin trains 4-way DDP at the same global batch of 1024, every other system on a single GPU.

\paragraph{Training budget.} \citet[Appendix~C.1]{ouyang2026fewstep} reports a budget of $4{\times}10^6$ update steps at batch 1000, but this figure is inconsistent with the paper's own wall-clock numbers: at our measured per-step time on our GPU class, the reported step count would take an order of magnitude longer than the $\sim$7--8\,h the paper reports for its two phases combined. We therefore size the pipeline by wall-clock instead: the whole of it (all three phases) is given the recorded training time of the \ourmethod{} actually selected for that system (\Cref{tab:hparam_nfm_selected}). \Cref{tab:hparam_flowmap_stepcost} gives all three quantities and \Cref{tab:hparam_flowmap_steps} the resulting step counts. Matching time is not matching capacity or data: these backbones are 4.1M (single-head) and 6.1M (dual-head) parameters on the ALA systems against 18--87M for the \ourmethod{} flows, so a flow-map step is much cheaper and a time-matched pipeline takes many more of them. Summed over the three phases, the baselines see 2.7--9.0$\times$ the number of passes over the training set that \ourmethod{}'s own 1000 epochs give.

\paragraph{Departures from a literal reading of the papers.} Of the three, only \citet{ai2026jointdistillationfastlikelihood} released code, and the molecular code the F2D2/SCALLOP comparison is reproduced from isn't available. Where code and paper disagree, we follow the code; where a value is left unstated, we flag our choice as unswept. The departures below fall into three kinds: the released code overriding a paper's written equation, a stability or correctness fix we made and verified against a diagnostic (synthetic or molecular), and a value neither paper states at all.

\textit{Code overrides the papers' prose.}
\begin{itemize}
    \item \textbf{Self-distillation stop-gradient (F2D2, SCALLOP).} Both papers write the correction term with the whole derivative detached. The released code detaches only the divergence-head correction and lets the velocity term's derivative backpropagate. The literal, fully-detached version is unstable in our hands: on our synthetic diagnostic, the weight norm drifts unboundedly until a batch produces $|x_s|\!\sim\!10^{10}$ and training diverges.
    \item \textbf{Time-pair sampling.} SCALLOP's prose restricts $(s,t)$ to the upper triangle, but the F2D2 code samples both independently and uniformly over the full square and its own inverse sampler at evaluation time relies on having seen the reversed ordering during training. We sample the full square for all six objectives, so every pretraining checkpoint covers the same measure as the student that reads it.
    \item \textbf{Divergence head and target scale (F2D2).} SCALLOP's stated contribution is a vectorized divergence head (F2D2's released code has a scalar one). Vectorizing F2D2 looks at first like a free upgrade so we did it. Targets are then scaled by the reference's own head-shape rule ($\div 1$ vectorized, $\div d$ scalar) rather than the factor of 10 transcribed from the papers' 2D toy, which has no molecular support: their other published anchor is ${\sim}10^4$ at $d=12288$, four orders away, with the peptides ($d=66$--$414$) in between.
\end{itemize}

\textit{Stability and correctness fixes.} They were all obtained by doing ablations on synthetic or molecular targets.
\begin{itemize}
    \item \textbf{Unsigned average-velocity head (FALCON).} The paper's $u_\theta = \mathrm{sign}(t-s)\,h_\theta$ forces $h_\theta$ to jump by $2v$ at the diagonal, so only the branch the tie-forcing sits on ever learns the correct side; we leave the head unsigned.
    \item \textbf{Stopped average-velocity derivative.} Left live, as in the reference, the residual becomes second-order in the network for both FALCON's $L_{\mathrm{avg}}$ and the ESD identity, since prediction and target come out of the same JVP call. We stop it.
    \item \textbf{Floored, projected SCALLOP divergence target.} $1/\sigma_t$ with $\sigma_t=t$ reaches $10^4$ near the data end, so a single sample would dominate the batch; we floor it at $10^{-3}$, at the cost of a small bias over the innermost decade of times, and state the target in the mean-free subspace rather than the ambient one.
    \item \textbf{FALCON-A's sub-jump weights.} The paper prints the two shares swapped relative to the identity's own derivative; they only agree at the paper's fixed $\lambda=0.5$, which is why the error never surfaces in its results.
    \item \textbf{SCALLOP's loss-uncertainty weighting.} Implemented but left off: as specified, $\exp(-\psi)L+\psi$ is unbounded below wherever the (self-distillation) loss it weights can reach zero, stationary at $\psi=\log L$ with value $1+\log L$. Clamping $\psi$ bounds the weight but not the outcome; we found no stated fix and record this as an open parity gap rather than paper over it.
\end{itemize}

\textit{Left unstated by both papers.}
\begin{itemize}
    \item \textbf{Diagonal batch share.} $0.75$ for FALCON/ESD (MeanFlow's \citep{geng2025mean} own ratio, the only published value for an average-velocity objective) and $0.5$ for the Lagrangian family (the complement of F2D2's stated self-distillation split). Never swept.
    \item \textbf{Chignolin's flow-map backbone and NFE.} Neither paper reports chignolin; both are extrapolated with the same channel/block-doubling rule already used for the DiT presets (\Cref{tab:hparam_dit_arch}), giving the backbone in \Cref{tab:hparam_flowmap_arch} and the step count in \Cref{tab:hparam_flowmap_nfe}.
\end{itemize}

\begin{table}[t]
    \centering
    \caption{Flow-map backbone architecture. The single-head network serves the ESD pretraining objective, FALCON and FALCON-A; the dual-head network serves LSD's students, F2D2 and SCALLOP, and counts its divergence-readout blocks only (the velocity reads straight off the shared trunk). Both carry a learned atom-type embedding, a bond-distance attention bias and centre-of-mass removal.}
    \label{tab:hparam_flowmap_arch}
    \small
    \begin{tabular}{lccccc}
        \toprule
         & ALA-2 & ALA-3 & ALA-4 & ALA-6 & Chignolin \\
        \midrule
        \multicolumn{6}{l}{\textit{Single-head (ESD, FALCON, FALCON-A)}} \\
        $d_{\text{model}}$     & 192 & 192 & 192 & 192 & 384 \\
        Blocks                 & 6   & 6   & 6   & 6   & 12  \\
        Heads                  & 6   & 6   & 6   & 6   & 12  \\
        Parameters             & \multicolumn{4}{c}{4.1M} & 32M \\
        \midrule
        \multicolumn{6}{l}{\textit{Dual-head (LSD, F2D2, SCALLOP)}} \\
        $d_{\text{model}}$     & 192 & 192 & 192 & 192 & 384 \\
        Shared blocks          & 6   & 6   & 6   & 6   & 12  \\
        Divergence-head blocks & 3   & 3   & 3   & 3   & 6   \\
        Heads                  & 6   & 6   & 6   & 6   & 12  \\
        Parameters             & \multicolumn{4}{c}{6.1M} & 48M \\
        \midrule
        MLP ratio              & 4.0 & 4.0 & 4.0 & 4.0 & 4.0 \\
        Time-embedding dim     & 64  & 64  & 64  & 64  & 128 \\
        \bottomrule
    \end{tabular}
\end{table}

\begin{table}[t]
    \centering
    \caption{Flow-map optimizer hyperparameters, per phase. Each pretraining objective takes the optimizer settings of the phase it initializes, so a warm start lands its student in a basin that student's own optimizer is already tuned for. All three columns use $\beta = (0.9, 0.999)$, $\varepsilon = 10^{-8}$, an EMA of decay $0.999$ and batch size 1024. ``Flat-then-sqrt'' holds the rate at its maximum for an absolute number of steps and then decays as $1/\sqrt{\text{step}}$; the onset is 10,000 steps on ALA-2 and ALA-4 and 20,000 elsewhere, read off the reference's molecular configurations. RAdam's variance rectification is why that schedule carries no warmup.}
    \label{tab:hparam_flowmap_optim}
    \small
    \begin{tabular}{lccc}
        \toprule
         & FM & ESD / FALCON / FALCON-A & LSD / F2D2 / SCALLOP \\
        \midrule
        Optimizer         & AdamW & AdamW & RAdam \\
        Learning rate     & $10^{-4}$ & $5{\times}10^{-4}$ & $10^{-4}$ \\
        Weight decay      & $10^{-4}$ & $10^{-4}$ & 0 \\
        Schedule          & sqrt-anneal & cosine & flat-then-sqrt \\
        Warmup            & 2.5\% & 5\% & none \\
        Grad-norm clip    & 10.0 & 10.0 & 10.0 \\
        \bottomrule
    \end{tabular}
\end{table}

\begin{table}[t]
    \centering
    \caption{ESD / FALCON / FALCON-A loss hyperparameters, the single-head family. ``Diagonal share'' is the fraction of each batch drawn at $s=t$ and given to the instantaneous term, its complement carrying the self-consistency term; it is an exact-count split of every batch rather than a Bernoulli mask. FALCON-A is the exception at 1.0: forcing $s=t$ degenerates its additivity target to $0=0$ rather than to flow matching, so it keeps an explicit instantaneous term and evaluates both on every sample. Values marked $\dagger$ are not stated numerically by any of the three papers and were never swept. ``--'' marks a term the objective does not have.}
    \label{tab:hparam_flowmap_loss}
    \small
    \begin{tabular}{lccc}
        \toprule
         & ESD & FALCON & FALCON-A \\
        \midrule
        Self-consistency weight $\lambda_{\text{avg}}$ & 1.0$^\dagger$ & 1.0$^\dagger$ & 1.0$^\dagger$ \\
        Cycle-consistency weight $\lambda_r$           & --             & 10.0           & 10.0           \\
        Split point $\lambda$                          & --             & --             & 0.5            \\
        Diagonal share                                 & 0.75$^\dagger$ & 0.75$^\dagger$ & 1.0            \\
        Time pair                                      & \multicolumn{3}{c}{unordered square} \\
        Time conditioning                               & pair & pair & pair \\
        Average-velocity derivative                    & stopped & stopped & -- \\
        \bottomrule
    \end{tabular}
\end{table}

\begin{table}[t]
    \centering
    \caption{LSD / F2D2 / SCALLOP loss hyperparameters, the dual-head family, which alone carries a divergence head. Same conventions as \Cref{tab:hparam_flowmap_loss}: $\dagger$ marks values neither paper states numerically, ``--'' marks a term the objective does not have.}
    \label{tab:hparam_flowmap_loss_lsd}
    \small
    \begin{tabular}{lccc}
        \toprule
         & LSD & F2D2 & SCALLOP \\
        \midrule
        Self-consistency weight $\lambda_{\text{avg}}$ & 1.0$^\dagger$ & 1.0 & 1.0 \\
        Diagonal share                                 & 0.5           & 0.5 & 0.5 \\
        Time pair                                      & \multicolumn{3}{c}{unordered square} \\
        Time conditioning                               & anchor + step & anchor + step & anchor + step \\
        Average-velocity derivative                    & live & live & live \\
        Divergence derivative                          & --   & stopped & stopped \\
        Divergence head                                & --   & vectorized & vectorized \\
        Divergence target scale                        & --   & 1.0 & 1.0 \\
        Divergence estimator                           & --   & Hutchinson, 1 probe$^\dagger$ & closed-form \\
        CDM-v time weighting                           & --   & --  & linear \\
        CDM-v denominator floor                        & --   & --  & $10^{-3}$ \\
        \bottomrule
    \end{tabular}
\end{table}

\begin{table}[t]
    \centering
    \caption{Training steps per system and phase, at batch size 1024. Phase 1 is one velocity teacher per system, shared by all four methods; phase 1.5 is one pretraining run per objective, shared by the two students that read it (ESD $\to$ FALCON, FALCON-A; LSD $\to$ F2D2, SCALLOP). Each entry is what the reported checkpoint was actually trained for.}
    \label{tab:hparam_flowmap_steps}
    \small
    \begin{tabular}{lccccc}
        \toprule
         & ALA-2 & ALA-3 & ALA-4 & ALA-6 & Chignolin \\
        \midrule
        \multicolumn{6}{l}{\textit{Phase 1 --- velocity teacher}} \\
        Velocity           & 177{,}510 & 162{,}378 & 139{,}389 & 206{,}125 & 117{,}661 \\
        \midrule
        \multicolumn{6}{l}{\textit{Phase 1.5 --- flow-map pretraining}} \\
        ESD                & 130{,}768 & 226{,}158 & 208{,}892 & 362{,}830 & 293{,}385 \\
        LSD                &  68{,}937 &  70{,}578 &  55{,}023 &  77{,}465 & 270{,}018 \\
        \midrule
        \multicolumn{6}{l}{\textit{Phase 2 --- distillation}} \\
        FALCON             & 111{,}321 & 179{,}349 & 163{,}935 & 304{,}518 & 258{,}576 \\
        FALCON-A           & 166{,}981 & 171{,}378 & 147{,}863 & 215{,}861 & 161{,}439 \\
        F2D2               &  76{,}066 &  77{,}869 &  66{,}004 &  97{,}224 & 119{,}655 \\
        SCALLOP            & 109{,}927 & 104{,}984 &  91{,}525 & 137{,}399 & 189{,}516 \\
        \bottomrule
    \end{tabular}
\end{table}

\begin{table}[t]
    \centering
    \caption{Measured per-step cost of every flow-map objective, and the two wall-clock quantities the step budget is solved against. Costs are seconds per optimizer step at batch size 1024 on one H100, each measured at the compilation mode its production run uses; chignolin is measured on its own 4-rank configuration at 256 rows per rank, so its column is directly comparable to the others per step.}
    \label{tab:hparam_flowmap_stepcost}
    \small
    \begin{tabular}{lccccc}
        \toprule
         & ALA-2 & ALA-3 & ALA-4 & ALA-6 & Chignolin \\
        \midrule
        \multicolumn{6}{l}{\textit{Seconds per step at batch 1024 (one H100)}} \\
        FM  & 0.0383 & 0.0380 & 0.0416 & 0.0581 & 0.1453 \\
        ESD     & 0.0332 & 0.0341 & 0.0361 & 0.0470 & 0.1040 \\
        LSD     & 0.0284 & 0.0320 & 0.0339 & 0.0472 & 0.1130 \\
        FALCON              & 0.0390 & 0.0430 & 0.0460 & 0.0560 & 0.1180 \\
        FALCON-A            & 0.0260 & 0.0450 & 0.0510 & 0.0790 & 0.1890 \\
        F2D2                & 0.0970 & 0.1020 & 0.1060 & 0.1230 & 0.2550 \\
        SCALLOP             & 0.0410 & 0.0470 & 0.0500 & 0.0690 & 0.1610 \\
        \bottomrule
    \end{tabular}
\end{table}

\subsubsection{Inference}
\label{subsec:hparam_inference}

\paragraph{\ourmethod{}.} Both \ourmethod{}-IS and \ourmethod{}-MIS sample a checkpoint from \Cref{tab:hparam_nfm_selected}, using the time grid of \citep{karras2022elucidating} with $\rho=7$, run in bf16 precision with $N=10{,}000$ particles.

\textit{Shared initialization of multiple trajectories.} At terminal noise, $X_T$ and $X_0$ are nearly independent, so $p_{0|T}\approx p_0$: the first denoising kernel, from $t_K$ to $t_{K-1}$, is essentially independent of the value of $x_T$. We exploit this when drawing $N$ trajectories in parallel: rather than sampling $N$ independent terminal states $x_T^n\sim\mathcal{N}(0,\sigma_T^2I)$ and running $N$ separate SNIS corrections \eqref{eq:denoising_snis} with $M$ particles each, we draw a single $x_T$ and build one SNIS approximation with $NM$ candidates, which we then reuse by sampling from it $N$ times. Since $p_{0|T}\approx p_0$ means the denoising posterior at this step is as hard to approximate as the target itself, pooling all $NM$ candidates concentrates statistical power precisely where the conditional flow $q^\theta_{0|t}$ faces its hardest task. This sharing introduces a slight correlation between the $N$ trajectories, which were otherwise independent; the correlation stays weak, however, precisely because the denoising kernel depends only weakly on $x_T$ at this step.

\textit{Candidate budget.} For \ourmethod{}-IS, the candidate count $M$ is solved per (system, variant, $K$) for the largest value whose measured wall-clock matches the corresponding unconditional baseline's own wall-clock; the selected checkpoints use the values in \Cref{tab:hparam_nfm_is_budget}. For \ourmethod{}-MIS, the architecture is fixed to the single variant with the best validation \ourmethod{}-IS score at $K=8$ per system (\Cref{tab:hparam_nfm_selected}, MIS row); the values of $M$ and the mixture size are given in \Cref{tab:hparam_nfm_mis_budget}.

\textit{Refinement.} At every level, each of the $N$ particles is refined in-loop by 512 steps of MAExL at a per-system calibrated initial step size. After the last level, the returned samples are polished by a post-hoc MALA chain targeting the clean Boltzmann density $p(x_0) \propto \exp(-\mathcal{E}(x_0))$ (geometric step-size schedule, per-system calibrated initial step size $\tau$) for 128 steps. \Cref{tab:hparam_nfm_stepsizes} gives the initial step sizes used for MAExL and MALA.

\begin{table}[t]
    \centering
    \caption{\ourmethod{}-IS candidate budget $M$ for the selected (variant, $K$) of \Cref{tab:hparam_nfm_selected}.}
    \label{tab:hparam_nfm_is_budget}
    \small
    \begin{tabular}{lccccc}
        \toprule
         & ALA-2 & ALA-3 & ALA-4 & ALA-6 & Chignolin \\
        \midrule
        $M$ & 14 & 12 & 14 & 20 & 20 \\
        \bottomrule
    \end{tabular}
\end{table}

\begin{table}[t]
    \centering
    \caption{\ourmethod{}-MIS selected mixture configuration. $C$ and $M$ are the validation-selected mixture size and per-particle candidate count in every system.}
    \label{tab:hparam_nfm_mis_budget}
    \small
    \begin{tabular}{lccccc}
        \toprule
         & ALA-2 & ALA-3 & ALA-4 & ALA-6 & Chignolin \\
        \midrule
        $C$ & 4 & 4  & 2  & 3  & 2  \\
        $M$                  & 2 & 3  & 7  & 5  & 6  \\
        \bottomrule
    \end{tabular}
\end{table}

\paragraph{Velocity teacher.}
\label{subsec:hparam_inference_velocity}
The phase-1 velocity field is also evaluated on its own, as a flow-matching baseline in the style of ECNF++ \citep{tan2025scalable}, on the same protocol as everything else so its row sits beside theirs. Samples are the probability-flow ODE's push-forward of an ambient $\mathcal{N}(0, I)$ prior, integrated from the prior end down to the data end by adaptive Dormand--Prince at $\texttt{rtol} = \texttt{atol} = 10^{-4}$ with the data end inset at $t_{\min} = 10^{-5}$, where the linear schedule is singular. That comes to roughly 130--165 function evaluations per solve. The log-density is the instantaneous change of variables accumulated along the same trajectory, with the \emph{exact} divergence: one JVP per input coordinate, $d = 66$--$414$. There is no probe count to tune. This run is fp32 rather than bf16: the log-determinant is an integral accumulated over ${\sim}160$ solver steps and is the whole point of the evaluation. 
We use $N = 10^4$ proposals per seed over 3 seeds then post-hoc MALA chain as every other method to match the energy evaluation budget (see \Cref{tab:posthoc_mala_steps}). $10^4$ rather than the $10^6$ the flow-map and TarFlow cells use is what the exact likelihood affords .

\paragraph{FALCON, FALCON-A, F2D2 and SCALLOP.} All four sample the phase-2 checkpoint at its EMA weights, at a fixed per-system number of flow-map steps (\Cref{tab:hparam_flowmap_nfe}), in bf16 with the prior, the log-determinant, and the log-density accumulator held in fp32. FALCON and FALCON-A step on FALCON's own EDM-shaped grid ($\sigma_{\min}=10^{-3}$, $\sigma_{\max}=1$, $\rho=7$, mapped affinely onto $[10^{-4}, 1-10^{-4}]$, which puts the fine steps at the data end); F2D2 and SCALLOP on the uniform grid in $t$ that their paper states.

\textit{Density evaluation.} The two families read their density differently, and the difference is what the evaluation costs. FALCON and FALCON-A have no divergence head, so $\log q$ is a sum of per-step Jacobian log-determinants,  a dense $d \times d$ Jacobian per step, evaluated in fp32 whatever precision the network itself ran in, since a determinant amplifies rounding error with matrix size. F2D2 and SCALLOP read their divergence head and accumulate $\sum_i (s_i - t_i) \sum_j D_\theta(x_{t_i}, t_i, s_i)_j$.

\textit{Evaluation protocol.} Evaluation is self-normalized importance sampling against the Boltzmann target, with $10^6$ proposals per seed over 3 seeds, scored on the test split. The one exception is FALCON and FALCON-A on chignolin, where the exact per-step Jacobian costs ${\sim}500$ GPU-hours per seed at $10^6$ proposals, and the pool is $10^5$ instead.  Importance weights are computed on the original proposals using their proposal densities and Boltzmann energies. We then resample $10^4$ configurations according to the normalized weights and apply post-hoc MALA steps to each selected configuration (see \Cref{tab:posthoc_mala_steps} for the number of steps). The refined configurations retain equal weights, with no post-MALA importance-weight computation, and are compared against $10^4$ reference configurations. A MALA kernel that leaves the Boltzmann target invariant can be applied before resampling while retaining the pre-move importance weights, or after resampling while retaining equal weights. Both constructions preserve the target in the importance-sampling limit, although their finite-sample dependence and variance can differ. The unconditional SNIS baseline follows the same order. No centre-of-mass correction is applied, unlike the TarFlow/StarFlow baselines: these models are mean-free end to end and no centre-of-mass augmentation enters their training, so their proposal density carries no centroid factor to correct.

\begin{table}[t]
    \centering
    \caption{Flow-map inference steps (NFEs), shared by all four methods. The four ALA values are the ones both papers report; chignolin appears in neither and reuses the largest.}
    \label{tab:hparam_flowmap_nfe}
    \small
    \begin{tabular}{lccccc}
        \toprule
         & ALA-2 & ALA-3 & ALA-4 & ALA-6 & Chignolin \\
        \midrule
        NFE & 4 & 8 & 8 & 16 & 16 \\
        \bottomrule
    \end{tabular}
\end{table}

\paragraph{Batched evaluation and chunk sizes.} Every sampler above draws and scores its pool in chunks rather than at once, and the chunk size is a throughput and memory setting that no reported number depends on. \Cref{tab:hparam_chunking} collects all of them.

\begin{table*}[t]
    \centering
    \caption{Evaluation chunk sizes, in rows, on an A100-80GB. The SNIS and \ourmethod{} architectures are selected per system based on validation metrics. \ourmethod{} variant labels (T1, W1, S1, S1k, S2) are defined in \Cref{tab:hparam_nfm_arch_sweep}.}
    \label{tab:hparam_chunking}
    \small
    \begin{tabular}{lccccc}
        \toprule
         & ALA-2 & ALA-3 & ALA-4 & ALA-6 & Chignolin \\
        \midrule
        \multicolumn{6}{l}{\textit{SNIS: Unconditional TarFlow/StarFlow}} \\
        Selected flow            & TarFlow & TarFlow & StarFlow & StarFlow & TarFlow \\
        Sample                   & 131072 & 131072 & 65536 & 32768 & 16384 \\
        Log-prob                 & 131072 & 131072 & 32768 & 32768 & 32768 \\
        \midrule
        \multicolumn{6}{l}{\textit{\ourmethod{} (measured knee, bf16)}} \\
        Selected variant (IS / MIS)  & W1 & W1 & T1 & S2 / S1k & S1 \\
        Sample              & 32768 & 65536 & 65536 & 16384 & 4096 \\
        Log-prob            & 4096 & 8192 & 8192 & 4096 & 256 \\
        \ourmethod{}-IS chunk         & 32768 & 65536 & 65536 & 16384 & 4096 \\
        \ourmethod{}-MIS chunk        & 32768 & 65536 & 65536 & 16384 & 1024 \\
        \midrule
        \multicolumn{6}{l}{\textit{Flow maps (sample with log-prob, bf16)}} \\
        FALCON / FALCON-A        & 4096 & 1024 & 1024 & 256 & 16 \\
        F2D2 / SCALLOP           & 8192 & 8192 & 8192 & 8192 & 8192 \\
        \midrule
        \multicolumn{6}{l}{\textit{Velocity teacher (exact divergence, fp32)}} \\
        Sample with log-prob     & 64 & 48 & 32 & 32 & -- \\
        \bottomrule
    \end{tabular}
\end{table*}

\paragraph{Comparing the distilled to their teacher.}

\begin{table*}[t]
    \centering
    \caption{Flow-map samplers against the velocity field they are distilled from, on ALA-$N$ systems with $N=2,3,4,6$. Results are obtained from three random seeds, reported as \texttt{mean} $\pm$ \texttt{std}. Energy and torus Wasserstein distances are computed on the test split against $10^4$ reference samples.}
    \scriptsize
    \label{tab:flow_map_vs_velocity}

    \renewcommand{\arraystretch}{1.5}
    \setlength{\tabcolsep}{1.5pt}

    \begin{tabular}{c l c c c c c}
        \cmidrule[0.8pt]{1-7}
        \textbf{} & & {\footnotesize \textbf{Velocity}} & {\footnotesize FALCON} & {\footnotesize FALCON-A} & {\footnotesize F2D2} & {\footnotesize SCALLOP} \\
        \cmidrule[0.5pt]{1-7}

        \multirow{4}{*}{\begin{tabular}[c]{@{}c@{}}\textbf{ALA-2} \\ ($d=66$)\end{tabular}}
        & $\mathcal{E}$-$\mathcal{W}_2$ $\downarrow$ & $0.765 \err{0.097}$ & $0.787 \err{0.091}$ & $0.966 \err{0.285}$ & $1.132 \err{0.018}$ & $1.110 \err{0.014}$ \\
        & $\mathbb{T}$-$\mathcal{W}_2$ $\downarrow$ & $0.274 \err{0.032}$ & $0.524 \err{0.212}$ & $0.892 \err{0.932}$ & $0.774 \err{0.059}$ & $0.391 \err{0.029}$ \\
        & Inference (s) $\downarrow$ & $1375.3 \err{10.4}$ & $1477.6 \err{2.0}$ & $1461.3 \err{11.9}$ & $70.6 \err{0.1}$ & $71.5 \err{0.6}$ \\
        & Training (h) $\downarrow$ & $1.9$ & $4.8$ & $5.0$ & $6.3$ & $5.6$ \\
        \cmidrule{1-7}

        \multirow{4}{*}{\begin{tabular}[c]{@{}c@{}}\textbf{ALA-3} \\ ($d=99$)\end{tabular}}
        & $\mathcal{E}$-$\mathcal{W}_2$ $\downarrow$ & $0.410 \err{0.219}$ & $0.250 \err{0.031}$ & $0.445 \err{0.410}$ & $0.648 \err{0.211}$ & $0.596 \err{0.448}$ \\
        & $\mathbb{T}$-$\mathcal{W}_2$ $\downarrow$ & $0.677 \err{0.168}$ & $0.468 \err{0.144}$ & $1.102 \err{0.540}$ & $1.072 \err{0.370}$ & $0.946 \err{0.117}$ \\
        & Inference (s) $\downarrow$ & $2956.8 \err{3.0}$ & $7832.2 \err{7.4}$ & $7813.1 \err{5.2}$ & $158.1 \err{0.1}$ & $158.9 \err{0.9}$ \\
        & Training (h) $\downarrow$ & $1.7$ & $7.1$ & $7.2$ & $6.5$ & $5.8$ \\
        \cmidrule{1-7}

        \multirow{4}{*}{\begin{tabular}[c]{@{}c@{}}\textbf{ALA-4} \\ ($d=126$)\end{tabular}}
        & $\mathcal{E}$-$\mathcal{W}_2$ $\downarrow$ & $1.410 \err{0.380}$ & $1.556 \err{0.164}$ & $2.375 \err{1.573}$ & $3.477 \err{2.242}$ & $3.306 \err{1.521}$ \\
        & $\mathbb{T}$-$\mathcal{W}_2$ $\downarrow$ & $2.133 \err{0.370}$ & $0.948 \err{0.125}$ & $1.532 \err{0.929}$ & $2.564 \err{0.060}$ & $2.490 \err{0.142}$ \\
        & Inference (s) $\downarrow$ & $4299.4 \err{7.3}$ & $12172.3 \err{0.2}$ & $12155.4 \err{11.9}$ & $191.4 \err{0.1}$ & $192.5 \err{0.7}$ \\
        & Training (h) $\downarrow$ & $1.6$ & $7.1$ & $7.1$ & $5.6$ & $5.1$ \\
        \cmidrule{1-7}

        \multirow{4}{*}{\begin{tabular}[c]{@{}c@{}}\textbf{ALA-6} \\ ($d=189$)\end{tabular}}
        & $\mathcal{E}$-$\mathcal{W}_2$ $\downarrow$ & $2.339 \err{1.352}$ & $0.652 \err{0.137}$ & $0.980 \err{0.488}$ & $2.749 \err{2.219}$ & $2.216 \err{0.248}$ \\
        & $\mathbb{T}$-$\mathcal{W}_2$ $\downarrow$ & $2.696 \err{0.557}$ & $1.829 \err{0.291}$ & $1.659 \err{0.008}$ & $3.111 \err{0.261}$ & $3.300 \err{0.410}$ \\
        & Inference (s) $\downarrow$ & $9489.6 \err{73.7}$ & $59606.9 \err{40.4}$ & $59610.4 \err{50.4}$ & $492.3 \err{0.6}$ & $484.4 \err{13.2}$ \\
        & Training (h) $\downarrow$ & $3.3$ & $15.1$ & $15.1$ & $10.3$ & $9.8$ \\
        \cmidrule[0.8pt]{1-7}
    \end{tabular}
\end{table*}

\Cref{tab:flow_map_vs_velocity} compares SNIS performed directly against the velocity teacher with SNIS performed against each of its four distilled students, on the four ALA systems. The teacher itself is a strong sampler in its own right (competitive with the global comparison of \Cref{tab:wasserstein}) and on ALA-2 it beats all four flow maps outright. On the larger systems the picture is mixed rather than uniformly closing: it is specifically the single-head family, FALCON and FALCON-A, that occasionally overtakes the teacher, while the dual-head family essentially never does. Since the teacher is already a competitive sampler, these occasional wins are not evidence that it was undertrained.

The inference-time column looks paradoxical: the flow maps draw $100\times$ more particles than the teacher ($10^6$ against $10^4$ as $10^6$ was unfeasible with the velocity teacher) while taking roughly $10\times$ fewer steps (4--16 NFE against $\sim$130--165, \Cref{tab:hparam_flowmap_nfe}). For FALCON and FALCON-A these two factors largely cancel: their dense $d\times d$ Jacobian log-determinant, paid at every step, erases the step-count saving, leaving wall-clock inference never faster than the teacher and up to $6.3\times$ slower on ALA-6. F2D2 and SCALLOP avoid that cost by reading a divergence head instead, so the step-count saving survives the extra particles: both are consistently $19$--$22\times$ faster than the teacher across all four systems.

\paragraph{Cost of the Langevin refinements.} Every sampler reported here spends part of its inference budget on Langevin dynamics, and at a matched target-energy budget the two kernels enter at the same scale. \ourmethod{} pays both: 512 MAExL steps at every level, over the full particle set, plus 128 steps of post-hoc MALA on its (comparatively small) returned sample. Every other method pays only the post-hoc chain, but at the step count that equalises the budget (\Cref{tab:posthoc_mala_steps}). \Cref{tab:hparam_langevin_cost} gives the measured wall-clock and its share of each sampler's inference time.

\begin{table*}[t]
    \centering
    \caption{Langevin refinement cost at a matched target-energy budget: 512 MAExL steps for \ourmethod{} and the budget-matched post-hoc MALA ladder for all other methods. Each entry gives the wall-clock time spent in the Langevin kernels and its share of the sampler's total inference time (means over 3 seeds, A100-80GB).}
    \label{tab:hparam_langevin_cost}
    \small
    \begin{tabular}{lccccc}
        \toprule
         & ALA-2 & ALA-3 & ALA-4 & ALA-6 & Chignolin \\
        \midrule
        \multicolumn{6}{l}{\textit{In-loop MAExL (512 steps) and post-hoc MALA (128 steps)}} \\
        NF$^2$M-IS           & 6 (4.1\%) & 14 (8.6\%) & 17 (9.4\%) & 23 (5.0\%) & 108 (3.8\%) \\
        NF$^2$M-MIS          & 9 (7.0\%) & 14 (9.5\%) & 17 (10.7\%) & 36 (8.8\%) & 112 (5.0\%) \\
        \midrule
        \multicolumn{6}{l}{\textit{Budget-matched post-hoc MALA only (\Cref{tab:posthoc_mala_steps})}} \\
        TarFlow/StarFlow     & 8 (7.2\%) & 14 (10.6\%) & 17 (10.1\%) & 35 (9.3\%) & 166 (6.8\%) \\
        Velocity             & 435 (24.2\%) & 322 (9.8\%) & 470 (9.9\%) & 499 (5.0\%) & -- \\
        FALCON               & 8 (0.6\%) & 14 (0.2\%) & 17 (0.1\%) & 35 (0.1\%) & 168 (0.1\%) \\
        FALCON-A             & 8 (0.6\%) & 13 (0.2\%) & 17 (0.1\%) & 35 (0.1\%) & 168 (0.1\%) \\
        F2D2                 & 8 (18.5\%) & 13 (11.5\%) & 17 (12.0\%) & 35 (8.6\%) & 167 (3.8\%) \\
        SCALLOP              & 8 (18.3\%) & 13 (11.5\%) & 17 (11.9\%) & 35 (8.6\%) & 167 (3.8\%) \\
        \bottomrule
    \end{tabular}
\end{table*}

\paragraph{Skipping post-hoc MALA steps.}

\begin{table*}[t]
    \centering
    \caption{Results on ALA-$N$ systems, with $N=2,3,4,6$, and on chignolin, with and without post-hoc MALA refinement, at a matched target-energy budget \Cref{tab:posthoc_mala_steps}. Results are obtained from three random seeds, reported as \texttt{mean} $\pm$ \texttt{std}. Energy and torus Wasserstein distances are computed on the test split against $10^4$ reference samples. The proposal budget differs by method: $10^4$ samples for \ourmethod{}, $10^5$ for FALCON and FALCON-A on chignolin, and $10^6$ for every other column. Inference time covers sampling, plus the post-hoc ladder for the \textsc{MALA} rows. Best in \textbf{bold}, second best \underline{underlined}, within each row.}
    \label{tab:post_hoc_mala}
    \tiny
    \renewcommand{\arraystretch}{1.3}
    \setlength{\tabcolsep}{1.3pt}

    \resizebox{\linewidth}{!}{%
    \begin{tabular}{c l c H H c c c c}
        \cmidrule[0.8pt]{1-9}
        \textbf{} & & {\footnotesize SNIS} & {\footnotesize \ourmethod{}-IS} & {\footnotesize \ourmethod{}-MIS} & {\footnotesize FALCON} & {\footnotesize FALCON-A} & {\footnotesize F2D2} & {\footnotesize SCALLOP} \\
        \cmidrule[0.5pt]{1-9}

        \multirow{6}{*}{\begin{tabular}[c]{@{}c@{}}\textbf{ALA-2} \\ ($d=66$)\end{tabular}}
        & $\mathcal{E}$-$\mathcal{W}_2$ MALA $\downarrow$ & $0.77 \err{0.03}$ & $\mathbf{0.13} \err{0.01}$ & $\underline{0.24} \err{0.08}$ & $0.76 \err{0.05}$ & $0.60 \err{0.33}$ & $0.89 \err{0.03}$ & $0.87 \err{0.03}$ \\
        & $\mathcal{E}$-$\mathcal{W}_2$ raw $\downarrow$  & $1.07 \err{0.15}$ & $\mathbf{0.20} \err{0.05}$ & $\underline{0.24} \err{0.07}$ & $0.64 \err{0.07}$ & $1.50 \err{1.05}$ & $2.84 \err{0.42}$ & $2.77 \err{0.36}$ \\
        & $\mathbb{T}$-$\mathcal{W}_2$ MALA $\downarrow$ & $\mathbf{0.24} \err{0.01}$ & $1.05 \err{0.01}$ & $1.24 \err{0.01}$ & $0.52 \err{0.25}$ & $0.86 \err{1.09}$ & $0.75 \err{0.07}$ & $\underline{0.37} \err{0.04}$ \\
        & $\mathbb{T}$-$\mathcal{W}_2$ raw $\downarrow$  & $\mathbf{0.25} \err{0.01}$ & $1.05 \err{0.01}$ & $1.24 \err{0.01}$ & $0.52 \err{0.20}$ & $0.92 \err{0.97}$ & $0.78 \err{0.06}$ & $\underline{0.40} \err{0.03}$ \\
        & Inference MALA (s) $\downarrow$ & $117 \err{0}$ & $137 \err{15}$ & $126 \err{14}$ & $1453 \err{2}$ & $1448 \err{5}$ & $46 \err{0}$ & $46 \err{0}$ \\
        & Inference raw (s) $\downarrow$  & $109 \err{0}$ & $137 \err{15}$ & $126 \err{14}$ & $1444 \err{2}$ & $1440 \err{4}$ & $37 \err{0}$ & $38 \err{0}$ \\
        \cmidrule{1-9}

        \multirow{6}{*}{\begin{tabular}[c]{@{}c@{}}\textbf{ALA-3} \\ ($d=99$)\end{tabular}}
        & $\mathcal{E}$-$\mathcal{W}_2$ MALA $\downarrow$ & $0.27 \err{0.07}$ & $\underline{0.16} \err{0.01}$ & $\mathbf{0.16} \err{0.01}$ & $0.18 \err{0.05}$ & $0.24 \err{0.14}$ & $0.55 \err{0.13}$ & $0.70 \err{0.19}$ \\
        & $\mathcal{E}$-$\mathcal{W}_2$ raw $\downarrow$  & $0.30 \err{0.07}$ & $\underline{0.17} \err{0.03}$ & $\mathbf{0.14} \err{0.02}$ & $0.65 \err{0.21}$ & $3.87 \err{2.64}$ & $2.30 \err{0.92}$ & $1.73 \err{0.53}$ \\
        & $\mathbb{T}$-$\mathcal{W}_2$ MALA $\downarrow$ & $\mathbf{0.36} \err{0.02}$ & $0.49 \err{0.02}$ & $0.47 \err{0.01}$ & $\underline{0.46} \err{0.17}$ & $1.10 \err{0.71}$ & $1.05 \err{0.48}$ & $0.85 \err{0.20}$ \\
        & $\mathbb{T}$-$\mathcal{W}_2$ raw $\downarrow$  & $\mathbf{0.38} \err{0.01}$ & $0.49 \err{0.02}$ & $\underline{0.47} \err{0.01}$ & $0.48 \err{0.15}$ & $1.13 \err{0.55}$ & $1.12 \err{0.38}$ & $1.00 \err{0.10}$ \\
        & Inference MALA (s) $\downarrow$ & $128 \err{0}$ & $164 \err{0}$ & $147 \err{13}$ & $7792 \err{9}$ & $7789 \err{6}$ & $117 \err{0}$ & $117 \err{0}$ \\
        & Inference raw (s) $\downarrow$  & $115 \err{0}$ & $164 \err{0}$ & $147 \err{13}$ & $7778 \err{7}$ & $7775 \err{5}$ & $104 \err{0}$ & $104 \err{0}$ \\
        \cmidrule{1-9}

        \multirow{6}{*}{\begin{tabular}[c]{@{}c@{}}\textbf{ALA-4} \\ ($d=126$)\end{tabular}}
        & $\mathcal{E}$-$\mathcal{W}_2$ MALA $\downarrow$ & $1.52 \err{0.35}$ & $\mathbf{0.99} \err{0.03}$ & $\underline{1.11} \err{0.01}$ & $1.76 \err{0.19}$ & $1.99 \err{0.74}$ & $2.75 \err{2.10}$ & $2.98 \err{0.97}$ \\
        & $\mathcal{E}$-$\mathcal{W}_2$ raw $\downarrow$  & $2.25 \err{1.26}$ & $\mathbf{0.95} \err{0.14}$ & $\underline{1.00} \err{0.06}$ & $2.09 \err{0.33}$ & $4.40 \err{4.38}$ & $6.37 \err{1.22}$ & $5.36 \err{1.95}$ \\
        & $\mathbb{T}$-$\mathcal{W}_2$ MALA $\downarrow$ & $\mathbf{0.78} \err{0.13}$ & $\underline{0.82} \err{0.04}$ & $0.88 \err{0.05}$ & $0.89 \err{0.14}$ & $1.46 \err{1.09}$ & $2.43 \err{0.12}$ & $2.37 \err{0.19}$ \\
        & $\mathbb{T}$-$\mathcal{W}_2$ raw $\downarrow$  & $\underline{0.86} \err{0.13}$ & $\mathbf{0.83} \err{0.04}$ & $0.88 \err{0.05}$ & $1.00 \err{0.13}$ & $1.59 \err{0.96}$ & $2.66 \err{0.04}$ & $2.59 \err{0.13}$ \\
        & Inference MALA (s) $\downarrow$ & $167 \err{5}$ & $184 \err{14}$ & $162 \err{14}$ & $12124 \err{1}$ & $12118 \err{5}$ & $143 \err{0}$ & $143 \err{0}$ \\
        & Inference raw (s) $\downarrow$  & $150 \err{5}$ & $184 \err{14}$ & $162 \err{14}$ & $12107 \err{0}$ & $12101 \err{4}$ & $126 \err{0}$ & $126 \err{0}$ \\
        \cmidrule{1-9}

        \multirow{6}{*}{\begin{tabular}[c]{@{}c@{}}\textbf{ALA-6} \\ ($d=189$)\end{tabular}}
        & $\mathcal{E}$-$\mathcal{W}_2$ MALA $\downarrow$ & $0.92 \err{1.02}$ & $1.02 \err{0.06}$ & $1.30 \err{0.09}$ & $\mathbf{0.61} \err{0.41}$ & $\underline{0.68} \err{0.49}$ & $2.16 \err{1.72}$ & $3.77 \err{1.72}$ \\
        & $\mathcal{E}$-$\mathcal{W}_2$ raw $\downarrow$  & $2.66 \err{1.97}$ & $\mathbf{1.05} \err{0.13}$ & $\underline{1.36} \err{0.09}$ & $2.43 \err{0.88}$ & $1.70 \err{0.23}$ & $13.33 \err{1.61}$ & $11.58 \err{1.46}$ \\
        & $\mathbb{T}$-$\mathcal{W}_2$ MALA $\downarrow$ & $1.77 \err{0.68}$ & $\mathbf{0.99} \err{0.03}$ & $\underline{1.03} \err{0.02}$ & $1.72 \err{0.33}$ & $1.56 \err{0.03}$ & $2.93 \err{0.29}$ & $3.13 \err{0.46}$ \\
        & $\mathbb{T}$-$\mathcal{W}_2$ raw $\downarrow$  & $1.95 \err{0.74}$ & $\mathbf{0.99} \err{0.03}$ & $\underline{1.03} \err{0.02}$ & $1.91 \err{0.29}$ & $1.74 \err{0.01}$ & $3.24 \err{0.26}$ & $3.44 \err{0.43}$ \\
        & Inference MALA (s) $\downarrow$ & $373 \err{0}$ & $453 \err{14}$ & $406 \err{14}$ & $59502 \err{50}$ & $59528 \err{58}$ & $405 \err{1}$ & $402 \err{5}$ \\
        & Inference raw (s) $\downarrow$  & $338 \err{0}$ & $452 \err{14}$ & $405 \err{14}$ & $59467 \err{40}$ & $59493 \err{47}$ & $370 \err{1}$ & $367 \err{4}$ \\
        \cmidrule{1-9}

        \multirow{6}{*}{\begin{tabular}[c]{@{}c@{}}\textbf{Chignolin} \\ ($d=414$)\end{tabular}}
        & $\mathcal{E}$-$\mathcal{W}_2$ MALA $\downarrow$ & $7.75 \err{5.29}$ & $\mathbf{4.33} \err{0.26}$ & $6.11 \err{0.13}$ & $4.81 \err{1.07}$ & $5.03 \err{3.62}$ & $4.79 \err{2.34}$ & $\underline{4.77} \err{1.08}$ \\
        & $\mathcal{E}$-$\mathcal{W}_2$ raw $\downarrow$  & $27.76 \err{19.91}$ & $\mathbf{4.05} \err{0.27}$ & $\underline{6.11} \err{0.23}$ & $8.69 \err{2.25}$ & $12.00 \err{3.43}$ & $26.12 \err{11.05}$ & $46.48 \err{16.13}$ \\
        & $\mathbb{T}$-$\mathcal{W}_2$ MALA $\downarrow$ & $3.86 \err{0.17}$ & $\underline{2.60} \err{0.05}$ & $\mathbf{2.58} \err{0.02}$ & $2.99 \err{0.09}$ & $3.45 \err{0.62}$ & $3.93 \err{0.45}$ & $3.57 \err{0.20}$ \\
        & $\mathbb{T}$-$\mathcal{W}_2$ raw $\downarrow$  & $4.22 \err{0.14}$ & $\underline{2.60} \err{0.05}$ & $\mathbf{2.58} \err{0.02}$ & $3.33 \err{0.08}$ & $3.76 \err{0.53}$ & $4.29 \err{0.37}$ & $3.94 \err{0.22}$ \\
        & Inference MALA (s) $\downarrow$ & $2443 \err{2}$ & $2831 \err{17}$ & $2265 \err{1}$ & $169239 \err{171}$ & $169131 \err{151}$ & $4380 \err{4}$ & $4384 \err{1}$ \\
        & Inference raw (s) $\downarrow$  & $2277 \err{2}$ & $2826 \err{17}$ & $2259 \err{1}$ & $169071 \err{140}$ & $168963 \err{124}$ & $4214 \err{3}$ & $4218 \err{1}$ \\
        \cmidrule[0.8pt]{1-9}
    \end{tabular}%
    }
\end{table*}

\Cref{tab:post_hoc_mala} reports every sampler at matched budget, both with its post-hoc MALA ladder ($128$ steps for \ourmethod{}, \Cref{tab:posthoc_mala_steps} for the others) and on its raw output. With refinement, \ourmethod{} is best in energy on every system except ALA-6, and best in torsion on ALA-6 and chignolin. Without it, \ourmethod{} is best in energy on every system and in torsion on the three largest. Its own numbers barely depend on the ladder: torsion is unchanged to within $0.001$ and energy moves by small amounts in both directions. The baselines behave in the opposite way. Most of their energy accuracy comes from the ladder, their raw distances are up to an order of magnitude worse on the larger systems, and for several of them the raw seed standard deviation is comparable to the mean. FALCON's refined energy win on ALA-6 in particular does not survive: its raw distance is more than twice that of \ourmethod{}-IS, even though FALCON runs over a hundred times longer. In torsion space, the ladder improves the baselines by at most about $15\%$ and leaves their order unchanged, so the shape of that comparison is a property of the proposals themselves.

\paragraph{Matching the target-energy budget.}

\begin{figure*}[t]
    \centering
    \includegraphics[width=\linewidth]{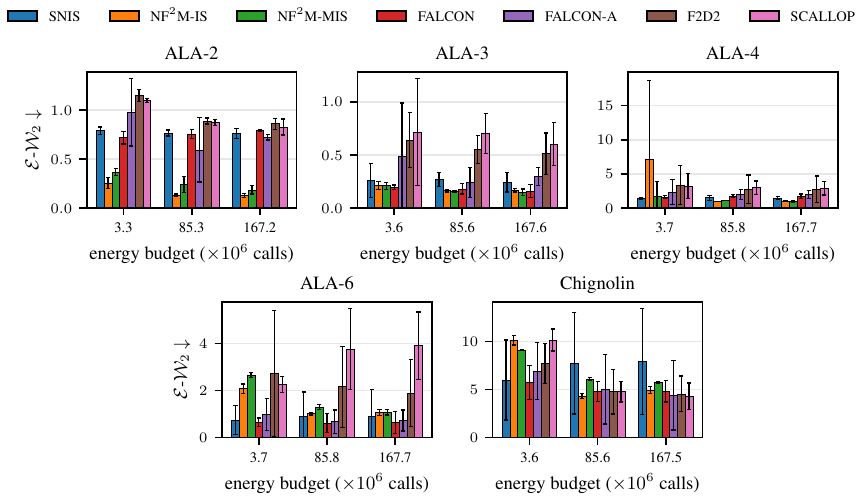}
    \caption{Energy Wasserstein distance against the matched target-energy budget, on the four ALA systems and chignolin. Every point on a panel spends the same number of target-energy calls. The three budgets are the 0-, 512- and 1024-MAExL-step settings of \Cref{tab:maexl_ablation} for \ourmethod{}, and the equalizing post-hoc MALA ladders of \Cref{tab:posthoc_mala_steps} for the SNIS-based samplers. Means over three seeds.}
    \label{fig:budget_match_energy}
\end{figure*}

\begin{figure*}[t]
    \centering
    \includegraphics[width=\linewidth]{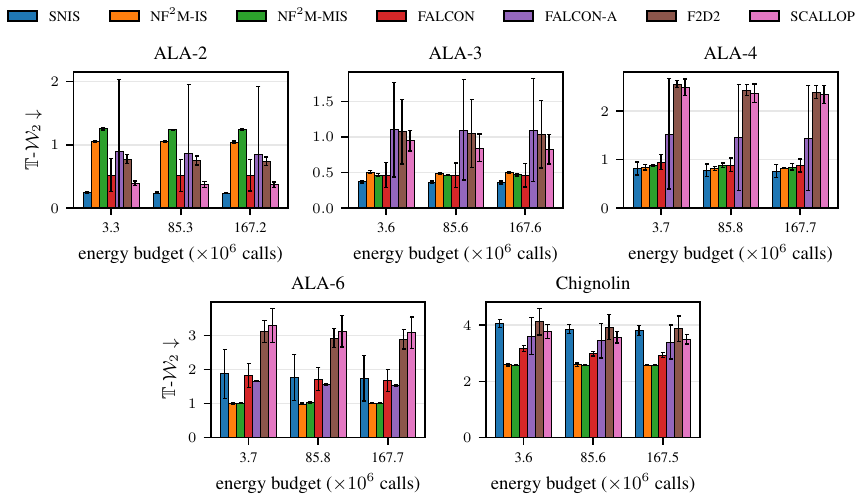}
    \caption{Torus Wasserstein distance against the matched target-energy budget, same protocol and seeds as \Cref{fig:budget_match_energy}.}
    \label{fig:budget_match_torsion}
\end{figure*}

\begin{figure*}[t]
    \centering
    \includegraphics[width=\linewidth]{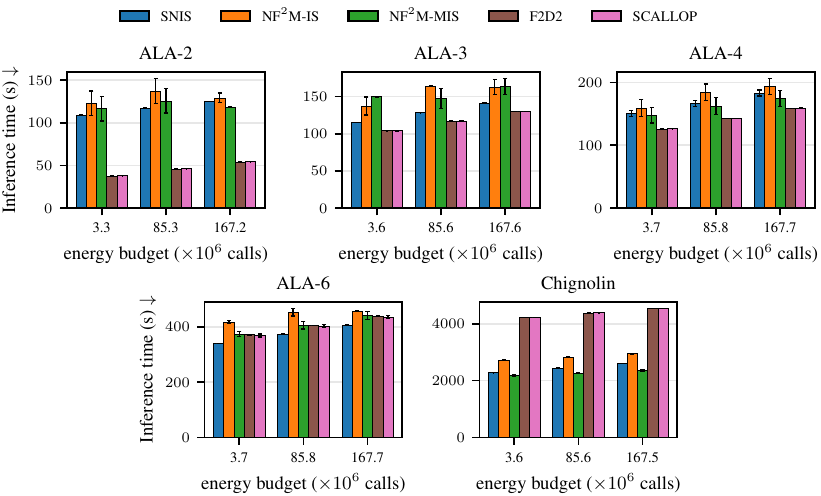}
    \caption{Inference wall-clock against the matched target-energy budget, same protocol and seeds as \Cref{fig:budget_match_energy}. FALCON and FALCON-A are omitted, their per-step Jacobian log-determinant putting them off the scale of every panel.}
    \label{fig:budget_match_time}
\end{figure*}

\begin{table}[t]
    \centering
    \caption{Number of post-hoc MALA steps run for the SNIS-based samplers at each matched budget. Each step count is chosen so that the SNIS-based total, $P$ energy calls for the importance weights plus $n\,(S+1)$ score calls on $n = 10^4$ chains with one score call counted as two energy calls, matches the budget of the corresponding \ourmethod{} setting. All SNIS-based samplers use $P = 10^6$ proposals, except FALCON and FALCON-A on chignolin, which use $P = 10^5$ and therefore take longer ladders for the same total.}
    \label{tab:posthoc_mala_steps}
    \small
    \renewcommand{\arraystretch}{1.2}
    \setlength{\tabcolsep}{6pt}
    \begin{tabular}{l l c c c}
        \cmidrule[0.8pt]{1-5}
        & & \multicolumn{3}{c}{MAExL steps} \\
        \cmidrule(lr){3-5}
        \textbf{System} & \textbf{Methods} & 0 & 512 & 1024 \\
        \cmidrule[0.5pt]{1-5}
        ALA-2 ($d=66$)  & \multirow{4}{*}{\begin{tabular}[c]{@{}l@{}}SNIS, FALCON, FALCON-A,\\ F2D2, SCALLOP\end{tabular}} & 114 & 4212 & 8308 \\
        ALA-3 ($d=99$)  & & 127 & 4231 & 8327 \\
        ALA-4 ($d=126$) & & 135 & 4239 & 8335 \\
        ALA-6 ($d=189$) & & 134 & 4238 & 8334 \\
        \cmidrule{1-5}
        \multirow{2}{*}{Chignolin ($d=414$)} & SNIS, F2D2, SCALLOP & 129 & 4228 & 8324 \\
                                             & FALCON, FALCON-A    & 174 & 4273 & 8369 \\
          \cmidrule{1-5}
          ALA-2 ($d=66$) & \multirow{4}{*}{Velocity teacher} & 163 & 4261 & 8357 \\
          ALA-3 ($d=99$) & & 176 & 4280 & 8376 \\
          ALA-4 ($d=126$) & & 184 & 4288 & 8384 \\
          ALA-6 ($d=189$) & & 183 & 4287 & 8383 \\
        \cmidrule[0.8pt]{1-5}
    \end{tabular}
\end{table}

Since MAExL and the post-hoc ladder both query the target, wall-clock alone does not settle whether \ourmethod{} simply spends more energy evaluations than the baselines. We therefore equalize the count exactly at three operating points, one per MAExL setting of \Cref{tab:maexl_ablation}, re-polishing every SNIS-based sampler from its unrefined samples with the ladder length that spends the same number of target-energy calls, counting a score as two. \Cref{tab:posthoc_mala_steps} gives those lengths, and \Cref{fig:budget_match_energy,fig:budget_match_torsion,fig:budget_match_time} trace every method across the three budgets. The leftmost point handicaps \ourmethod{}, since matching it means running with no MAExL at all, precisely the setting the ablation identifies as its weakest, while the baselines lose little from a shortened chain. From the middle budget onward \ourmethod{} recovers its \Cref{tab:wasserstein} numbers and with them the ranking of the main table, whereas the extra calls granted to the baselines buy energy accuracy only: their torsion curves are close to flat and stay well behind on the larger systems, since a local kernel cannot redistribute mass between torsional metastable states however long it runs. The one place the verdict shifts is chignolin energy at the largest budget, where several baselines become indistinguishable from \ourmethod{} in the mean, though their across-seed spread remains far wider. Matching energy calls also does not equalize wall-clock, and the gap widens with the budget, since a longer ladder is pure serial cost for the baselines while \ourmethod{} pays its share inside the sampling loop. The conclusion of the main table is therefore unchanged: the separation lies in the proposal rather than in the refinement budget.

\paragraph{Comparison with SBG.}

\begin{table*}[t]
    \centering
    \caption{Results on ALA-$N$ systems, with $N=2,3,4,6$, and on chignolin. Results are obtained from three random seeds, reported as \texttt{mean} $\pm$ \texttt{std}. Energy and torus Wasserstein distances are computed on the test split against $10^4$ reference samples. The proposal budget differs by method: $10^4$ samples for \ourmethod{} and SBG-SMC, and $10^6$ for SNIS. The SNIS column reports, per system, the better of the TarFlow and StarFlow checkpoints; the SBG-SMC columns are reported per architecture. Inference time covers sampling and the post-hoc MALA ladder; training time is the wall-clock of the corresponding model's training run. Best in \textbf{bold}, second best \underline{underlined}.}
    \scriptsize
    \label{tab:results_sbg}

    \renewcommand{\arraystretch}{1.5}
    \setlength{\tabcolsep}{1.5pt}

    \resizebox{\linewidth}{!}{%
    \begin{tabular}{c l c H H c c}
        \cmidrule[0.8pt]{1-7}
        \textbf{} & & {\small SNIS} & {\small \ourmethod{}-IS} & {\small \ourmethod{}-MIS} & {\small SBG-SMC {\tiny StarFlow}} & {\small SBG-SMC {\tiny TarFlow}} \\
        \cmidrule[0.5pt]{1-7}

        \multirow{4}{*}{\begin{tabular}[c]{@{}c@{}}\textbf{ALA-2} \\ ($d=66$)\end{tabular}}
        & $\mathcal{E}$-$\mathcal{W}_2$ $\downarrow$ & $0.752 \err{0.027}$ & $\mathbf{0.132} \err{0.013}$ & $\underline{0.242} \err{0.080}$ & $0.737 \err{0.142}$ & $0.686 \err{0.042}$ \\
        & $\mathbb{T}$-$\mathcal{W}_2$ $\downarrow$ & $\mathbf{0.245} \err{0.014}$ & $1.050 \err{0.011}$ & $1.239 \err{0.006}$ & $0.387 \err{0.075}$ & $\underline{0.385} \err{0.021}$ \\
        & Inference (s) $\downarrow$ & $128.8 \err{0.4}$ & $137.2 \err{14.5}$ & $125.8 \err{14.4}$ & $242.1 \err{9.7}$ & $263.2 \err{11.4}$ \\
        & Training (h) $\downarrow$ & $3.4$ & $4.3$ & $4.3$ & $3.6$ & $3.4$ \\
        \cmidrule{1-7}

        \multirow{4}{*}{\begin{tabular}[c]{@{}c@{}}\textbf{ALA-3} \\ ($d=99$)\end{tabular}}
        & $\mathcal{E}$-$\mathcal{W}_2$ $\downarrow$ & $0.225 \err{0.047}$ & $\underline{0.164} \err{0.011}$ & $\mathbf{0.157} \err{0.009}$ & $2.715 \err{3.857}$ & $0.524 \err{0.072}$ \\
        & $\mathbb{T}$-$\mathcal{W}_2$ $\downarrow$ & $\mathbf{0.371} \err{0.016}$ & $0.490 \err{0.016}$ & $\underline{0.465} \err{0.007}$ & $1.256 \err{1.256}$ & $0.470 \err{0.044}$ \\
        & Inference (s) $\downarrow$ & $163.1 \err{0.3}$ & $164.1 \err{0.4}$ & $147.2 \err{13.1}$ & $282.5 \err{9.6}$ & $288.4 \err{9.5}$ \\
        & Training (h) $\downarrow$ & $4.7$ & $6.0$ & $6.0$ & $5.3$ & $4.7$ \\
        \cmidrule{1-7}

        \multirow{4}{*}{\begin{tabular}[c]{@{}c@{}}\textbf{ALA-4} \\ ($d=126$)\end{tabular}}
        & $\mathcal{E}$-$\mathcal{W}_2$ $\downarrow$ & $1.313 \err{0.159}$ & $\mathbf{0.988} \err{0.027}$ & $\underline{1.107} \err{0.013}$ & $1.927 \err{0.096}$ & $1.734 \err{0.138}$ \\
        & $\mathbb{T}$-$\mathcal{W}_2$ $\downarrow$ & $\underline{0.819} \err{0.134}$ & $0.824 \err{0.041}$ & $0.882 \err{0.051}$ & $0.896 \err{0.035}$ & $\mathbf{0.789} \err{0.093}$ \\
        & Inference (s) $\downarrow$ & $206.0 \err{7.3}$ & $184.3 \err{13.5}$ & $162.3 \err{13.6}$ & $337.2 \err{0.3}$ & $357.7 \err{10.5}$ \\
        & Training (h) $\downarrow$ & $5.8$ & $5.3$ & $5.3$ & $5.8$ & $5.2$ \\
        \cmidrule{1-7}

        \multirow{4}{*}{\begin{tabular}[c]{@{}c@{}}\textbf{ALA-6} \\ ($d=189$)\end{tabular}}
        & $\mathcal{E}$-$\mathcal{W}_2$ $\downarrow$ & $0.714 \err{0.584}$ & $1.018 \err{0.063}$ & $1.301 \err{0.092}$ & $\mathbf{0.249} \err{0.116}$ & $\underline{0.395} \err{0.182}$ \\
        & $\mathbb{T}$-$\mathcal{W}_2$ $\downarrow$ & $1.874 \err{0.724}$ & $\mathbf{0.990} \err{0.026}$ & $\underline{1.027} \err{0.021}$ & $1.355 \err{0.061}$ & $1.797 \err{0.068}$ \\
        & Inference (s) $\downarrow$ & $460.9 \err{0.3}$ & $452.8 \err{13.9}$ & $405.6 \err{14.2}$ & $1014.0 \err{11.9}$ & $975.2 \err{11.2}$ \\
        & Training (h) $\downarrow$ & $11.7$ & $12.8$ & $10.2$ & $11.7$ & $10.7$ \\
        \cmidrule{1-7}

        \multirow{4}{*}{\begin{tabular}[c]{@{}c@{}}\textbf{Chignolin} \\ ($d=414$)\end{tabular}}
        & $\mathcal{E}$-$\mathcal{W}_2$ $\downarrow$ & $\underline{5.898} \err{3.986}$ & $\mathbf{4.331} \err{0.262}$ & $6.112 \err{0.129}$ & $8.635 \err{4.105}$ & $14.439 \err{6.417}$ \\
        & $\mathbb{T}$-$\mathcal{W}_2$ $\downarrow$ & $4.056 \err{0.134}$ & $\underline{2.597} \err{0.054}$ & $\mathbf{2.581} \err{0.018}$ & $3.933 \err{0.749}$ & $4.302 \err{0.392}$ \\
        & Inference (s) $\downarrow$ & $2792.2 \err{3.5}$ & $2830.6 \err{17.4}$ & $2264.5 \err{1.1}$ & $26337.6 \err{175.1}$ & $11875.9 \err{178.6}$ \\
        & Training (h) $\downarrow$ & $13.3$ & $21.7$ & $21.7$ & $17.4$ & $13.3$ \\
        \cmidrule[0.8pt]{1-7}
    \end{tabular}%
    }
\end{table*}

\Cref{tab:results_sbg} provides a full comparison against SBG-SMC on both proposal architectures. The two methods trade places depending on the regime. On the smallest peptides, the geometric annealing path is clearly beneficial in torsion space, where SBG-SMC is the strongest baseline, while \ourmethod{} already dominates in energy space. This ordering reverses as the dimension grows: on the largest systems \ourmethod{} improves on the best SBG-SMC variant on both metrics, and the gap is widest on chignolin, where the SMC ladder degrades sharply. Two further properties separate the methods. First, SBG-SMC is markedly less stable, both across seeds, with several runs whose standard deviation is of the order of their mean, whereas \ourmethod{} varies little across seeds and across its two importance-sampling variants. Second, at an equal proposal budget the annealing ladder roughly doubles inference time on the ALA systems and costs several times more on chignolin, for comparable training cost. Finally, since SNIS and SBG-SMC use the same unconditional checkpoints and differ only at inference, their comparison isolates the effect of the annealing path itself. Although the two use different proposal budgets for a comparable wall-clock, the picture matches the ablation of \citet{tan2025scalable}, the ladder helping on some systems and hurting on others rather than yielding a consistent gain.

\stopcontents[appendix]

\end{document}